\documentclass{article}

\usepackage[preprint]{neurips_2026}

\usepackage[utf8]{inputenc} 
\usepackage[T1]{fontenc}    
\usepackage{hyperref}       
\usepackage{url}            
\usepackage{booktabs}       
\usepackage{amsfonts}       
\usepackage{nicefrac}       
\usepackage{microtype}      
\usepackage{xcolor}         
\usepackage{graphicx} 

\setcitestyle{authoryear,round,citesep={;},aysep={,}}
\usepackage{subcaption}
\usepackage{amsmath}
\usepackage{amsthm}
\newtheorem{lemma}{Lemma}
\newtheorem{proposition}{Proposition}
\newtheorem{theorem}{theorem}
\usepackage{enumitem}
\usepackage[capitalize,noabbrev]{cleveref}
\usepackage{fontawesome5} 

\title{Do Latent-CoT Models Think Step-by-Step? \\ A Mechanistic Study on Sequential Reasoning Tasks}

\author{
\parbox{\textwidth}{\centering \normalfont
\textbf{Jia Liang}$^{2}$ \qquad
\textbf{Liangming Pan}$^{1}$$\thanks{Corresponding author.} $ \\[0.5ex]
$^{1}$MOE Key Laboratory of Computational Linguistics, \\ School of Computer Science, Peking University \\[0.5ex]
$^{2}$Institute for Computational and Mathematical Engineering, Stanford University  \\[0.5ex]
\texttt{jialiang2015@gmail.com} \qquad 
\texttt{liangmingpan@pku.edu.cn}  \\ 
    \href{https://github.com/jialiang19/latent-cot-thinking}{\textcolor{black}{\faGithub} \ \texttt{Latent-CoT-Thinking}}
}
}

\begin{document}

\maketitle

\begin{abstract}
Latent Chain-of-Thought aims to enable step-by-step computation without emitting long rationales, yet its internal mechanisms remain unclear. We study CODI, a continuous-thought teacher–student distillation model, on strictly sequential polynomial-iteration tasks with known intermediate states. Using logit-lens decoding, linear probes, attention analysis, and activation patching, we localize intermediate-state representations and trace how they are routed to the final readout. In short-horizon, low-hop tasks, CODI forms faithful bridge states across latent-thought positions, while the final input follows a separate near-direct route; predictions arise through late fusion at the answer readout. As task depth and difficulty increase, however, CODI does not reliably sustain a full latent rollout: it either compresses computation into a partial late-intermediate pathway or, in harder regimes, loses the latent reasoning signature altogether. To explain this transition, we show theoretically that the task’s algebraic structure controls its effective memory: compressible regimes support late-bottleneck reasoning, while incompressible regimes preserve full-history dependence and destabilize latent rollouts. Overall, our results characterize when CODI-style latent-CoT yields faithful iterative computation versus compressed or shortcut strategies.
\end{abstract}

\section{Introduction}

Recent \textit{Large Language Models} (LLMs) demonstrate substantial competence on multi-step reasoning tasks, producing correct outputs by integrating information across a sequence of intermediate deductions \citep{jaech2024openai, guo2025deepseek}. These capabilities are typically realized through two paradigms. In \emph{implicit reasoning}, the model carries out multi-hop inference internally within its hidden activations, without emitting intermediate steps. In \emph{explicit reasoning}, the model verbalizes intermediate computations as discrete tokens, most commonly in the form of \textit{Chain-of-Thought} (CoT) \citep{wei2022chain} traces. 

Recently, a third direction—\textit{latent CoT reasoning}~\citep{zhu2025survey, deng2024explicit, hao2024training, shen2025codi}—has begun to bridge these paradigms. Rather than relying solely on the fixed computational depth of a standard transformer (as in typical implicit reasoning), latent CoT architectures add extra internal compute through mechanisms such as continuous ``thought tokens'', iterative refinement, or recurrent state updates. These behaviors are often learned by distilling explicit CoT traces into hidden-state trajectories. In principle, latent CoT offers the benefits of explicit step-by-step computation—greater effective depth and expressivity—while avoiding the decoding overhead and brittleness of generating long natural-language rationales.

However, since latent CoT obscures the reasoning process within high-dimensional continuous vectors, this opacity creates a critical gap in our understanding: without the window of interpretability provided by text, it is difficult to verify whether the model is genuinely reasoning or merely employing sophisticated heuristics. Most mechanistic interpretability work has focused on either explicit traces or implicit reasoning \citep{cabannes2024iteration, zhang2025finite, biran2024hopping, wang2024grokked}, leaving the mechanisms and step-by-step faithfulness of latent chain-of-thought comparatively underexplored. Yet understanding these mechanisms is crucial: accuracy alone cannot distinguish multi-step computation from shortcuts, nor can it reveal when extra latent compute successfully increases compositional generalization. Many core mechanistic questions remain, such as when a latent-CoT model forms true intermediate states, where they are represented, and how they are stored, propagated, and used to produce the final answer.

\begin{figure*}[th]
  \centering
  \includegraphics[width=0.9\linewidth]{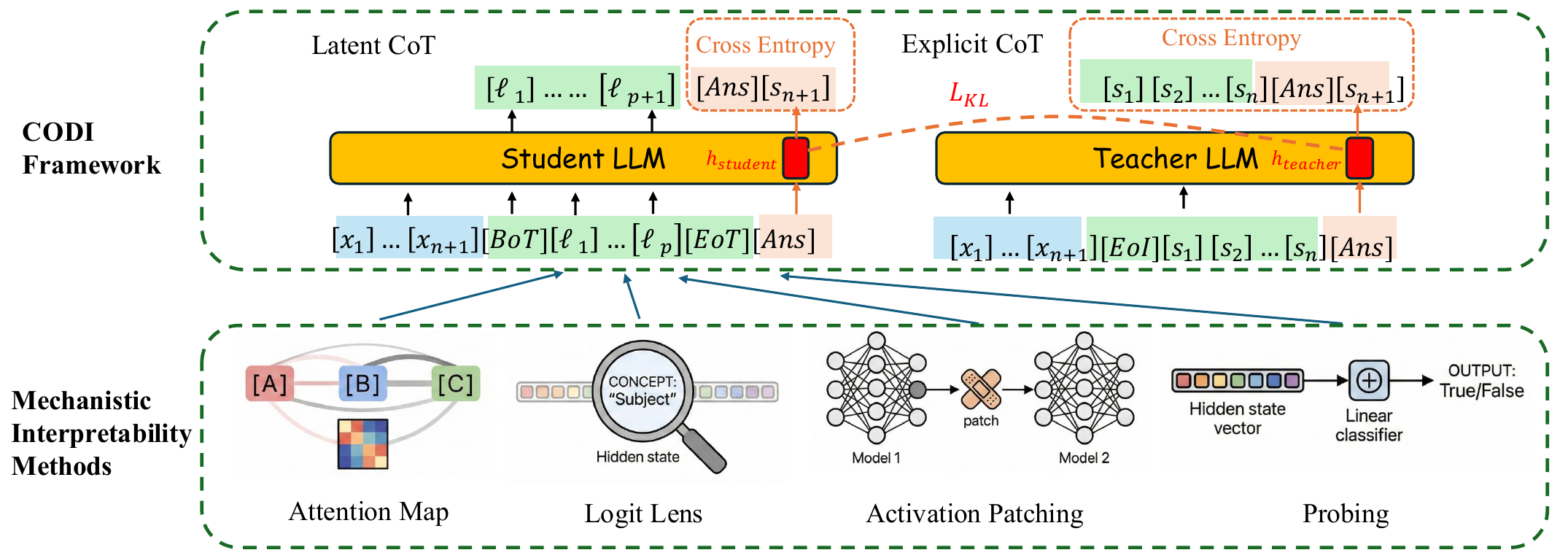}
  \caption{\textbf{Mechanistic study of CODI on sequential reasoning tasks.} \emph{Top:} the CODI training setup used for our polynomial iteration task. \emph{Bottom:} the four mechanistic interpretability methods we use to analyze the student model’s internal computations.}
  \label{fig:main_figure}
\end{figure*} 

In this paper, we give a mechanistic account of how latent CoT solves sequential multi-hop algorithmic tasks when computation is routed through a latent ``thought'' channel. Using logit-lens \citep{nostalgebraist2020logitlens} decoding, linear probes \citep{alain2016understanding}, attention analysis \citep{vaswani2017attention}, and targeted activation patching \citep{meng2022locating}, we test whether distilling explicit CoT traces into hidden states actually leads the model to internalize genuine step-by-step reasoning, or if it instead relies on shortcuts and late-stage fusion.

Specifically, we use the polynomial sequential tasks introduced by \citet{cabannes2024iteration} as a controlled testbed to probe latent CoT mechanisms in \textit{Continuous Chain-of-Thought via Self-Distillation (CODI)}~\citep{shen2025codi}, a typical latent CoT model that uses teacher–student distillation to learn latent reasoning (\cref{fig:main_figure}). Polynomial sequential tasks are an ideal testbed for studying CODI because they enforce a transparent, strictly sequential state update with ground-truth intermediates, letting us precisely probe whether and where CODI represents and propagates step-by-step computation within its latent thought channel. 
Empirically, our mechanistic analyses reveal a consistent two-route computation in CODI: In the default composite-modulus polynomial setting, on short-horizon, low hop tasks (2--3 hops), the latent channel constructs intermediate ``bridge'' states, while the final input is often delivered to the answer readout via a direct, copy-like pathway. As depth increases in this setting ($n \ge 4$), CODI rarely realizes a full latent rollout; instead, the latent channel shifts toward a partial-reasoning, late-bottleneck trace that retains only the final one or two intermediates.

Our polynomial-iteration tasks operate over integers modulo $m$. 
A striking empirical pattern emerges: mechanistic signatures of latent step-by-step reasoning persist across many composite moduli but largely vanish for prime moduli, where accuracy drops and intermediate-state decodability weakens or disappears. 
This discontinuity suggests that latent-CoT behavior depends not only on task length, but also on effective compressibility: composite moduli introduce many-to-one updates that can erase earlier information and make late-bottleneck strategies viable, whereas prime moduli make nonzero updates bijective, preserving dependence on the full input history and stressing the fixed latent budget. \textbf{Our contributions are threefold}: 
First, we causally trace CODI-style latent CoT on strictly sequential tasks, identifying where intermediate states emerge and how they are routed to the final answer. 
Second, we formalize the above intuition through a theory of task-level compressibility, showing how many-to-one contractions reduce the amount of prefix information needed to predict the final state, whereas bijective dynamics preserve long-range dependence and stress a fixed latent budget. 
Third, we validate this account beyond modular arithmetic using finite-state function composition (FSC), where transition rank directly controls compressibility and varying the fraction of compressive versus bijective transitions reproduces the same mechanistic transition. 
Code for reproducing the main experimental results is provided in the supplementary material.

\section{Related Work}
\label{gen_inst}

To understand how LLMs solve multi-step reasoning problems, a growing line of work applies mechanistic interpretability.
In the implicit-reasoning regime, recent studies argue that layers assume distinct computational roles during multi-hop inference \citep{biran2024hopping, li2024understanding, yu2025back, yang2025internal} and identify fine-grained internal structures that support reasoning \citep{hou2023towards, brinkmann2024mechanistic}.
Other work reports sharp training-phase transitions in which reasoning-like behavior emerges abruptly rather than gradually \citep{wang2024grokked, ye2025transformers, zhang2025complexity}.
Despite this progress, implicit reasoning can be brittle \citep{biran2024hopping, li2024understanding} and susceptible to shortcut solutions \citep{ju2024investigating, yang2025large}.
Several works further suggest that limited depth is a primary bottleneck, clarifying when and why reasoning fails \citep{merrill2023expressive, yu2024llms, guo2025llms, saunshi2025reasoning}.

On the explicit reasoning side, \citet{cabannes2024iteration} and \citet{dutta2024think} identify attention heads that reuse prior CoT tokens to propagate intermediate results, effectively leveraging generated text as an external memory, while complementary evidence from \citet{zhang2025finite} and \citet{rai2024investigation} indicates that CoT models also maintain and update internal state via circuits or neuron activations that encode intermediate variables. Moreover, many more works are investigating \emph{when} \citep{sprague2024cot, suzgun2023challenging} and \emph{why} \citep{li2023dissecting, yang2025chain} CoT enhance reasoning abilities, as well as its faithfulness \citep{kudo2024think, arcuschin2503chain}.

While prior work shows that latent-CoT representations can be steered or decoded \citep{zhang2024uncovering, wang2024latent}, these results are primarily correlational and do not directly reveal the underlying mechanism.
More recent work highlights the theoretical advantages of latent Chain-of-Thought (CoT) for parallel exploration and graph-based search \citep{zhu2025reasoning, gozeten2025continuous, xu2025formal} and uncovers probing inconsistencies in looped transformers \citep{lu2025latent}. 
While other recent mechanistic studies of latent-CoT models have identified structured scratchpad--computation cycles~\citep{goyal2025scratchpad} and shortcut behavior~\citep{zhang2025latent}, the limits of latent reasoning on sequential tasks remain poorly understood. In contrast, we provide a mechanistic account of CODI on controlled sequential tasks and characterize how task depth/difficulty and compressibility shape the transition from faithful iterative computation to partial late-bottleneck reasoning or breakdown of latent reasoning.


\section{Approach}

We train CODI on polynomial-iteration tasks with an explicit-CoT teacher and a latent-thought student, using feature-space distillation at the pre-answer \texttt{[Ans]} boundary to align internal states without emitting text. We then analyze the student’s computation with mechanistic interpretability tools to localize intermediate-state representations and distinguish iterative latent updates from shortcut routing (see \cref{fig:main_figure}).

\textbf{The CODI Framework.} 
\textit{Continuous Chain-of-Thought via Self-Distillation (CODI)} is a single-stage framework that compresses an explicit CoT into a short sequence of continuous/latent ``thought'' vectors while retaining CoT-level accuracy. It trains two modes within the same model: a \emph{teacher} that generates an explicit CoT trace (supervised with cross-entropy on the CoT steps and the final answer), and a \emph{student} that produces a fixed number of latent thought vectors between learned \texttt{[BoT]}/\texttt{[EoT]} (beginning/end-of-thought) markers and is trained with cross-entropy on the final answer. The key supervision is \emph{feature-space self-distillation}: instead of matching a textual rationale, CODI aligns teacher and student hidden representations at a designated pre-answer boundary using an $\ell_1$ loss with stop-gradient on the teacher, so the student learns to reproduce the teacher's CoT-induced internal state shift without emitting the CoT.


\textbf{Polynomial-Iteration Dataset.}
We adopt the polynomial-iteration task from the Iteration Head work~\citep{cabannes2024iteration} as a controlled testbed for mechanistic analysis.
For an $n$-hop task, we sample inputs $x_1,\ldots,x_{n+1}\in\mathbb{Z}_m$ and generate states $s_1,\ldots,s_{n+1}$ by $s_1=x_1$ and
\begin{equation}
s_t \;=\; s_{t-1}x_t + b \pmod m,\qquad t=2,\ldots,n+1,
\end{equation}
with $m=50$ and $b=1$ by default.
To instantiate CODI on this task, we train the \emph{teacher} to emit an explicit state-by-state trace and introduce an explicit pre-answer boundary token \texttt{[Ans]} immediately before the final state.
The \emph{student} follows CODI's latent-thought protocol and predicts the final answer after this boundary; the exact teacher/student serializations and the distillation point are shown in \Cref{fig:seq_figure}.
Training the \emph{student} uses a cross-entropy loss on $s_{n+1}$, the final answer, and a CODI-style feature distillation loss matching teacher and student representations at \texttt{[Ans]} (stop-gradient on the teacher).


\textbf{Mechanistic Interpretability Methods.}
We use four complementary tools to characterize and localize latent computation in CODI; full method descriptions and implementation details are provided in the referenced appendices. \emph{Logit lens} decodes residual-stream activations (via the unembedding) into a distribution over task states; we apply it across latent positions to track when true intermediates become decodable (Appendix~\ref{sec:logit_lens}). \emph{Attention maps} visualize how information is routed across input, latent-thought, and answer positions;  we use them to test whether attention implements step-indexed retrieval through the latent tokens (Appendix~\ref{app:attention_maps}). \emph{Linear probes} fit simple classifiers/regressors on hidden states; we train them to predict intermediates (e.g., $s_t$) at each layer/position to locate where state information is represented and whether the latent trajectory forms an accumulator-like trace (Appendix~\ref{app:probing}). Finally, \emph{activation patching} is a causal intervention that swaps clean activations into corrupted runs; we patch inputs, latent thoughts, and boundary positions to identify which locations are necessary to recover the correct answer and thus which pathway CODI relies on (Appendix~\ref{app:activation_patching}).

\begin{figure}[t]
  \centering

  \begin{minipage}[t]{0.48\linewidth}
    \centering
    \includegraphics[width=\linewidth]{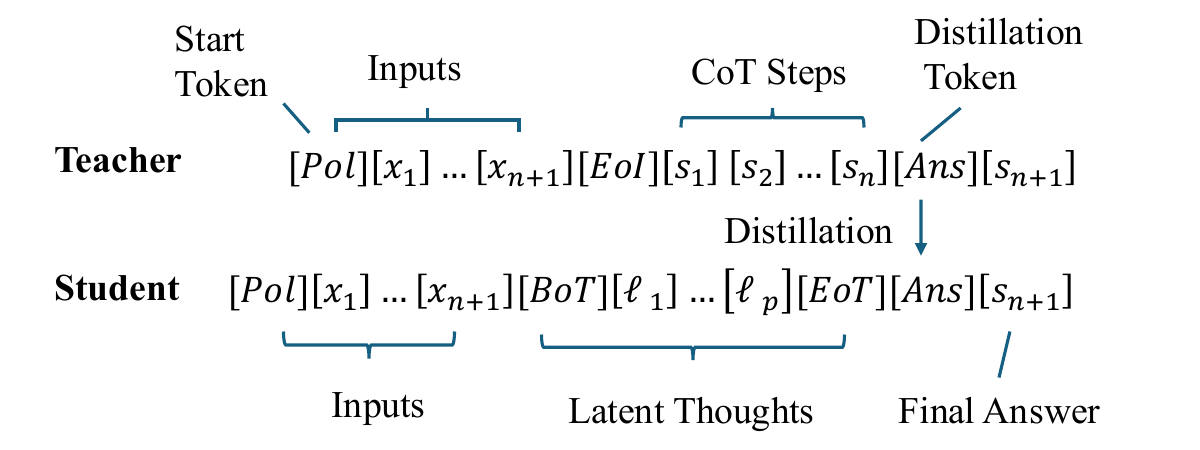}
    \captionof{figure}{
      \textbf{Polynomial tasks for training CODI.}
      The teacher generates an explicit CoT trace, while the student answers after a fixed-length latent-thought trajectory.
    }
    \label{fig:seq_figure}
  \end{minipage}
  \hfill
  \begin{minipage}[t]{0.48\linewidth}
    \centering
    \includegraphics[width=\linewidth]{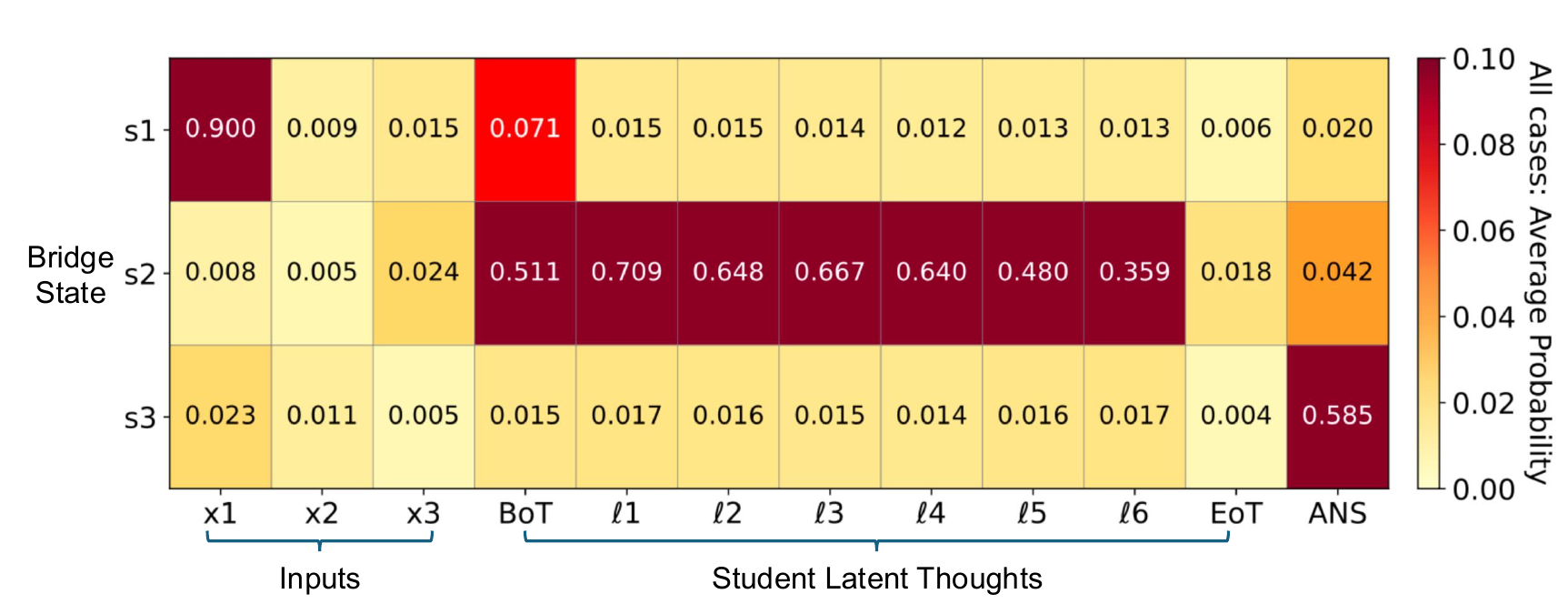}
    \captionof{figure}{
      \textbf{Logit lens on intermediate states $s_1, s_2, s_3$.}
      The bridge state $s_2$ becomes decodable during latent computation. Each cell show average decoding probability across all layers and all test inputs. 
    }
    \label{fig:logic_lens_seq_3}
  \end{minipage}

\end{figure}

\section{Empirical Experiments}

A core objective of this work is to test whether \emph{latent} chain-of-thought (CoT) implements genuinely sequential, step-by-step computation, as opposed to producing correct outputs via shallow heuristics. We operationalize this question through the presence of a \emph{bridge state}: an internal representation of the intermediate variable that must be computed after the first hop and then used to complete the second hop. We begin with the simplest controlled setting—two-hop instances in the polynomial task with sequence length three. The default CODI model is a three-layer, two-head GPT-2–style transformer (see Appendix~\ref{app:CODI_training_detail} for full training details). 

\subsection{Do bridge states form, and where are they represented?}

To determine \emph{whether} and \emph{where} the bridge state forms, we apply the \emph{logit lens} to the residual stream at every latent step and token position and track the decodability of the intermediate tokens $s_1, s_2, s_3$.  In the two-hop version of our polynomial task, the inputs are $x_1, x_2, x_3$, with
{\setlength{\abovedisplayskip}{4pt}
 \setlength{\belowdisplayskip}{4pt}
 \setlength{\abovedisplayshortskip}{2pt}
 \setlength{\belowdisplayshortskip}{2pt}
\[
\begin{aligned}
s_1 &= x_1,\quad s_2 = s_1 x_2 + 1 \pmod{50},\\
s_3 &= s_2 x_3 + 1 \pmod{50},
\end{aligned}
\]
}
where $s_3$ is the final answer. The central question is whether an explicit representation of the bridge state $s_2$ becomes decodable \emph{before} the model outputs the final answer. 

As shown in \cref{fig:logic_lens_seq_3}, the \emph{logit lens} indicates that the token corresponding to $s_2$ is decodable throughout latent steps \texttt{[$\ell_1$]}--\texttt{[$\ell_6$]}, with mean decoded probability ranging from $0.359$ to $0.709$ (averaged over layers and test inputs). We observe the same trend with linear probing (\cref{fig:probing_seq_3} in Appendix \ref{app:probing}): an $s_2$ signal emerges after layer 2 at the \texttt{[BoT]} position and remains decodable across all six latent steps, with probe confidence approaching $1$. Together, these results suggest that the model uses the latent channel to \emph{ maintain} an internal representation of the bridge state $s_2$ prior to producing the final answer.

\subsection{How is the final answer computed, and how does information flow to \texttt{[ANS]}?}
\label{sec:exp_q2}

Having established that the bridge state $s_2$ becomes decodable over the latent trajectory, we next ask how this intermediate representation is \emph{used} to produce the final answer: what routes information into the final readout position, and how are the two required inputs ($s_2$ and $x_3$) combined? 



\textbf{Attention Analysis.} From the attention maps in \cref{fig:2hop_L3H2_attn_map} (Appendix \ref{app:attention_maps}), we observe several heads in which the \texttt{[Ans]} tokens place high attention weight directly on the $x_3$ token. This pattern is consistent with \emph{a copy-like pathway that routes information from $x_3$ into the residual streams of \texttt{[Ans]}}. This interpretation is supported by linear probing \cref{fig:2_hop_probing_x3} (Appendix \ref{app:probing}): $x_3$ is not decodable throughout the latent steps \texttt{[$\ell_1$]}--\texttt{[$\ell_6$]}, but is strongly decodable at \texttt{[Ans]}, with probe probability near 0.85. \textbf{Targeted attention knockouts} show that the bridge-state signal $s_2$ reaches \texttt{[Ans]} primarily through the early-to-middle latent positions:  simultaneously blocking attention from \texttt{[Ans]} to \texttt{[BoT]}, all latent positions $\ell_1$--$\ell_6$, and \texttt{[EoT]} collapses accuracy to 3.8\%. \textbf{In leave-one-in ablations}, we block all of these sources except one and measure whether that single source can recover performance. Only $\ell_1$--$\ell_4$ recover most of the accuracy; \texttt{[BoT]} and
\texttt{[EoT]} alone do not. See Appendix~\ref{app:attention_knockout} for
details.  


Together with the logit-lens, probing, the attention analysis suggests a clear division of labor. The final input $x_3$ is routed to \texttt{[Ans]} through a largely direct, copy-like pathway, whereas the bridge state $s_2$ is formed and maintained in the latent channel and then read out through the early-to-middle latent positions. The final prediction is therefore produced by late fusion at \texttt{[Ans]}, where the residual stream combines the latent bridge-state signal with the directly routed $x_3$ signal before unembedding.

\textbf{Causal Evidence from Interventions.} Activation patching provides complementary causal evidence for a direct $x_3 \rightarrow \texttt{[Ans]}$ routing pathway. From \cref{fig:activation_patching_x3_seq_3}, patching the residual stream from a correct run into an $x_3$-corrupted run at the latent-step positions (\texttt{[$\ell_1$]}--[$\ell_6$], across layers) yields no accuracy recovery ($0\%$ average over all corrupted samples), indicating that the latent computation does not carry the $x_3$ signal in a way that affects the output. In contrast, applying the same patch at the \texttt{[Ans]} position restores performance: for the 3-layer transformer, injecting the patch after the second layer (\texttt{L2-Post} in \cref{fig:activation_patching_x3_seq_3}) recovers $100\%$ accuracy.  
Moreover, \cref{fig:activation_patching_x2_seq_3} shows that patching the residual stream from a correct run into an $x_2$-corrupted run at the \emph{start of the latent computation} yields substantial recovery. The effect is strongest at the \texttt{[BoT]} position and remains large for early latent steps $\ell_1$ and $\ell_2$. In particular, patching at \texttt{L2-Post} on the \texttt{[BoT]} token restores accuracy to $100\%$; recovery is still around $80\%$ for $\ell_1$ across layers and around $70\%$ for $\ell_2$ in earlier layers. Strikingly, patching at later latent steps produces no recovery ($0\%$), suggesting that the causally relevant computation for this task occurs early in the latent trajectory. \textbf{State-swap interventions} provide complementary causal evidence: replacing activations with those from a suffix-matched prompt that shares $x_3$ but differs in $s_2$ flips the prediction primarily at \texttt{[BoT]} and early latent steps ($l_1$--$l_2$), and the effect rapidly disappears in later latent positions; see Appendix~\ref{app:state-swap-intervention} for details. 


Together, these causal effects indicate that the intermediate state is constructed and used at \texttt{[BoT]} and early latent positions in a way that directly impacts the final output. We observe the same qualitative pattern when patching $x_1$, further strengthening this conclusion. Taken with the decodability and attention results, the evidence is most consistent with a two-stream mechanism: the latent trajectory is primarily used to construct the bridge state $s_2$, while $x_3$ is routed through a largely direct pathway and combined with the latent-state readout at \texttt{[Ans]}.

\begin{figure}[ht]
\centering

    \begin{subfigure}[t]{0.49\linewidth}
      \centering
      \includegraphics[width=\linewidth]{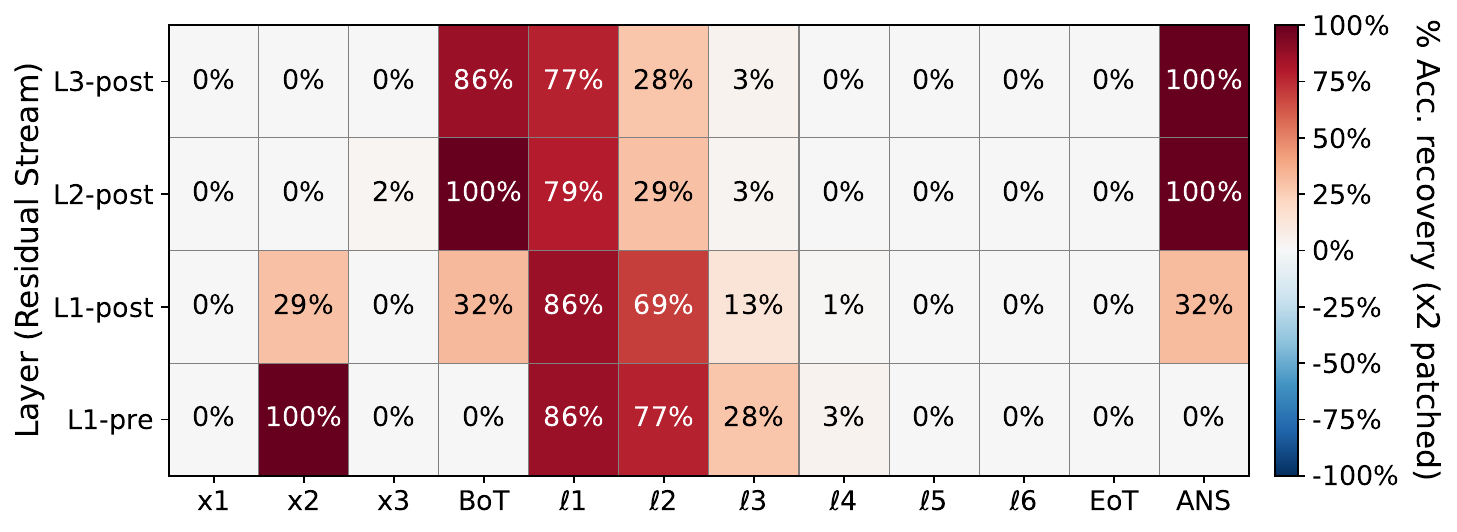}
      \caption{Patching for $x_2$-corrupted runs}
      \label{fig:activation_patching_x2_seq_3}
    \end{subfigure}
    \hfill
    \begin{subfigure}[t]{0.49\linewidth}
      \centering
      \includegraphics[width=\linewidth]{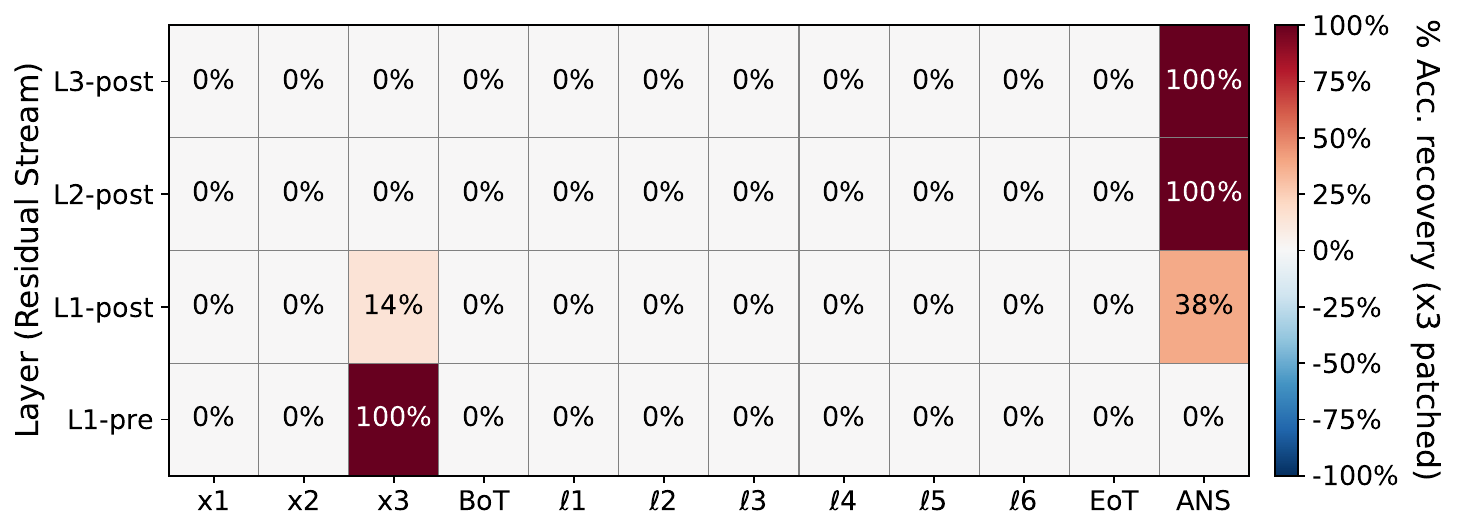}
      \caption{Patching for $x_3$-corrupted runs}
      \label{fig:activation_patching_x3_seq_3}
    \end{subfigure}

    \caption{\textbf{Activation patching for input-token corruptions (two-hop).} Cells show mean accuracy recovery (\%) over corrupted samples. Patching at latent-thought positions rescues $x_2$ corruptions, implicating the latent channel in storing $s_2$. For $x_3$ corruptions, recovery concentrates at \texttt{[Ans]}, consistent with direct routing of $x_3$ to the output.}
    \label{fig:activation_patching_x2_x3_seq_3}
\end{figure}

\subsection{How does compositional depth affect the emergent computation?}

The two-hop analyses reveal a stable mechanism: CODI (i) forms a bridge state in the latent channel and (ii) fuses it with the final input token at \texttt{[Ans]}. We then examine whether greater task depth elicits a longer multi-step latent rollout. To test this, we introduce an $n$-hop variant with $n \in \{3, 4, 5, 7, 9, 31\}.$

\textbf{Three-hop Polynomial Task}. 
For the three-hop polynomial task, there are two intermediate bridge states, $s_2$ and $s_3$. We find that CODI exhibits latent reasoning behavior by forming both intermediate states within its latent channel. At the same time, the model appears to use a direct-copy strategy for the final input, with $x_4$ copied to the \texttt{[Ans]} token to produce the final output. Additional details are provided in Appendix \ref{app:three_hop_analysis}.

\textbf{$n$-hop Polynomial Task, $n \ge 4$.} In the $n$-hop task we have $n+1$ inputs (i.e., $x_1,\ldots,x_{n+1}$) and $n-1$ intermediate bridge states, $s_2,\ldots,s_n$. For $n \ge 4$, increasing hop depth preserves the same two-stream routing pattern but induces a clear \emph{collapse} in intermediate-state visibility. Attention maps continue to reveal an early, direct pathway from the final input $x_{n+1}$ into \texttt{[Ans]} (e.g., strong \texttt{[Ans]}$\rightarrow x_{n+1}$ attention heads in \cref{fig:n_hops_attn_head} in Appendix \ref{app:attention_maps}). However, logit-lens and probing analyses over both input-token positions
$x_1,\ldots,x_{n+1}$ and latent positions
\texttt{[$\ell_1$]}--\texttt{[$\ell_6$]} show that the latent stream does
\emph{not} expose a step-by-step chain of intermediates. Among
$s_2,\ldots,s_n$, only the late intermediate $s_n$ (sometimes the last two intermediates $s_{n-1}, s_n$ as seen in Appendix \ref{app:partial_roll_out_last_two}) is reliably decodable, and this decodable signal
appears primarily in the latent positions, as shown in
\cref{fig:n_hops_logic_lens}.


\begin{figure}[t]
  \centering
  \captionsetup[subfigure]{font=footnotesize,skip=2pt}

  \begin{minipage}{1\linewidth}
    \includegraphics[width=0.5\linewidth,height=0.5cm]{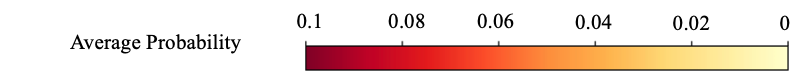}
    \vspace{-0.5em}
  \end{minipage}

  \begin{subfigure}[t]{0.24\linewidth}
    \centering
    \includegraphics[width=\linewidth]{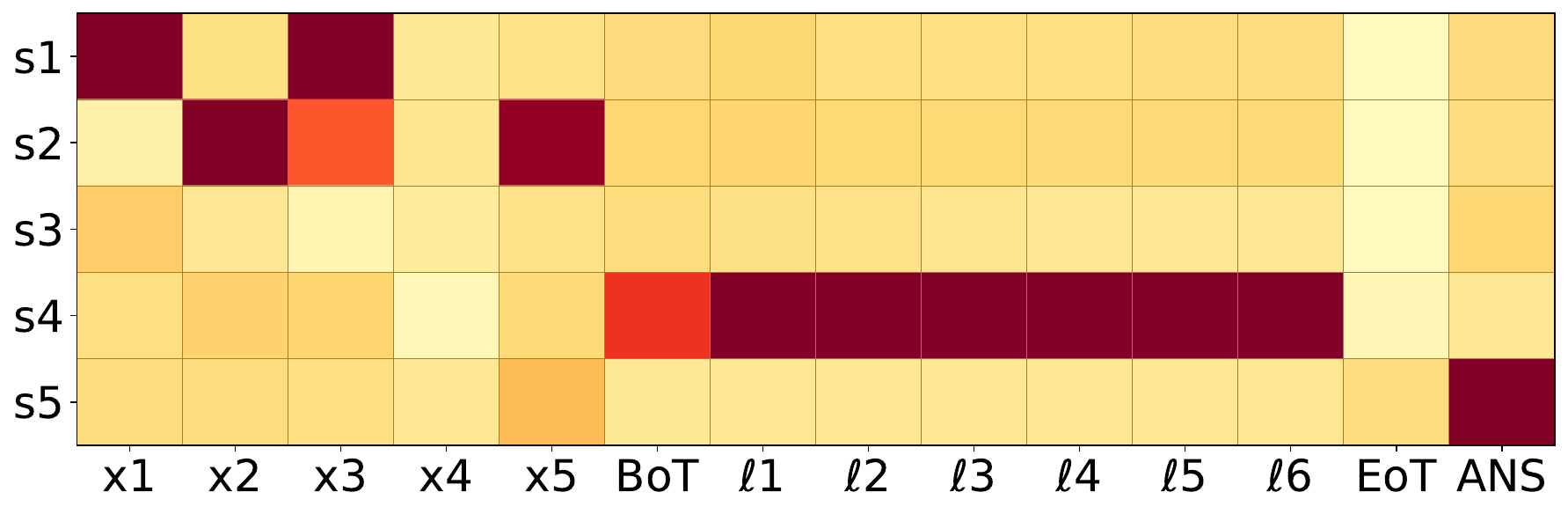}
    \caption{4-hop task}
  \end{subfigure}
  \hfill
  \begin{subfigure}[t]{0.24\linewidth}
    \centering
    \includegraphics[width=\linewidth]{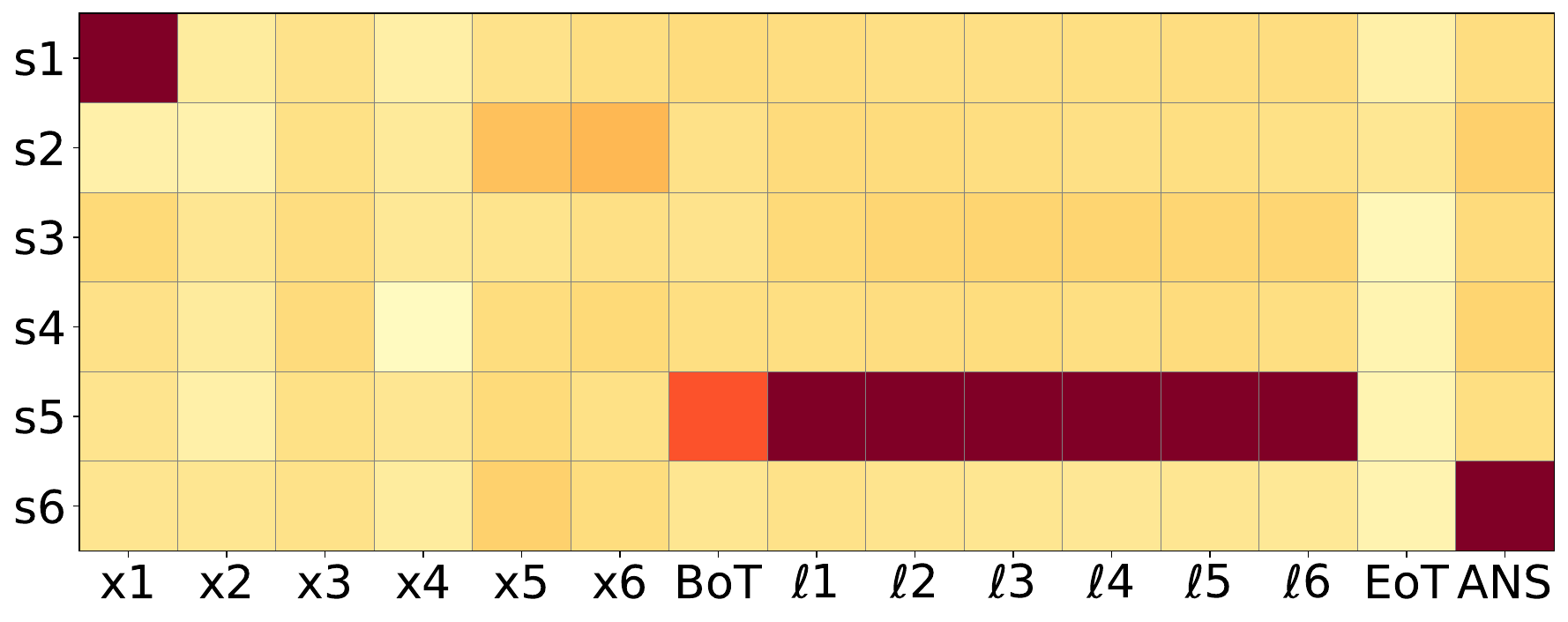}
    \caption{5-hop task}
  \end{subfigure}
  \hfill
  \begin{subfigure}[t]{0.24\linewidth}
    \centering
    \includegraphics[width=\linewidth]{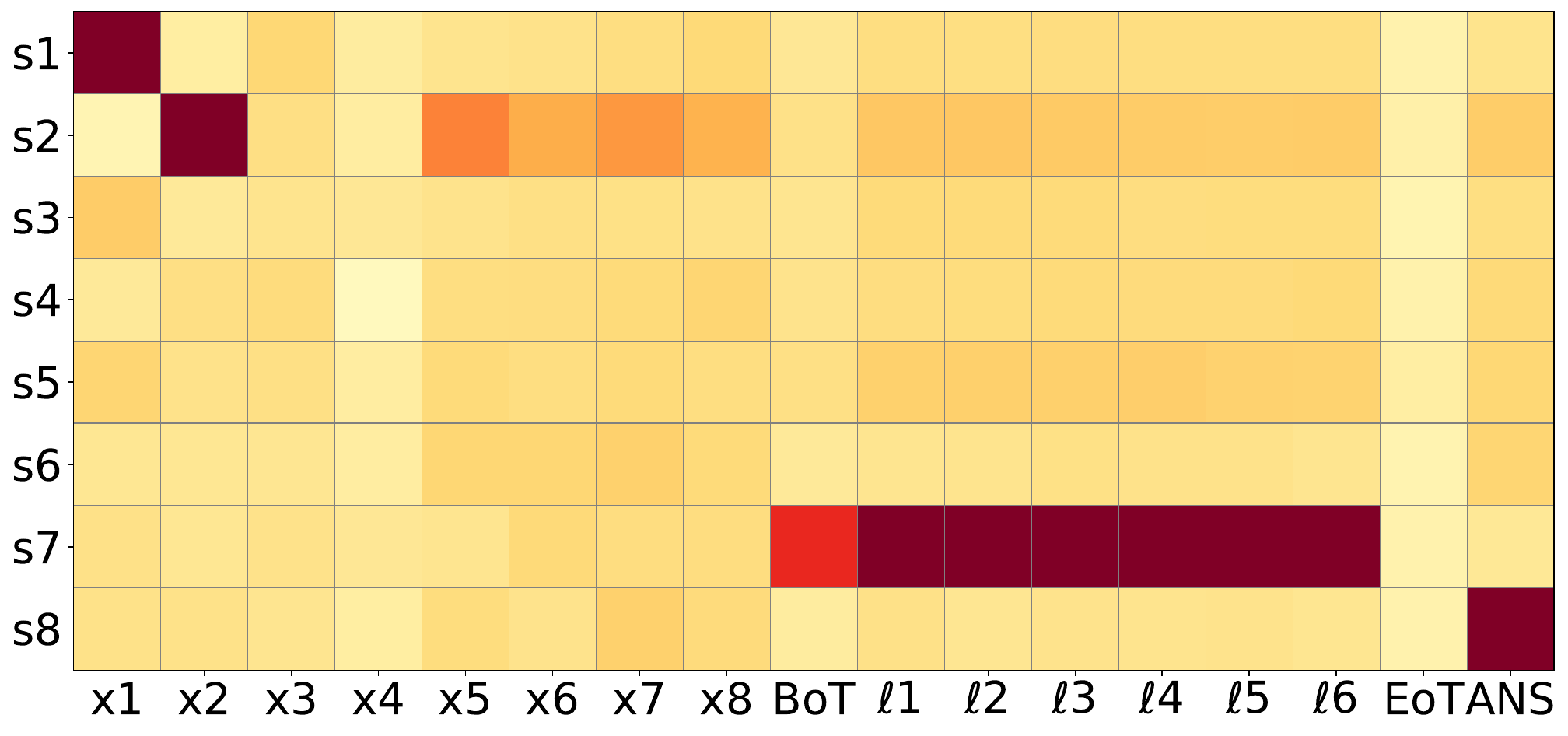}
    \caption{7-hop task}
  \end{subfigure}
  \hfill
  \begin{subfigure}[t]{0.24\linewidth}
    \centering
    \includegraphics[width=\linewidth]{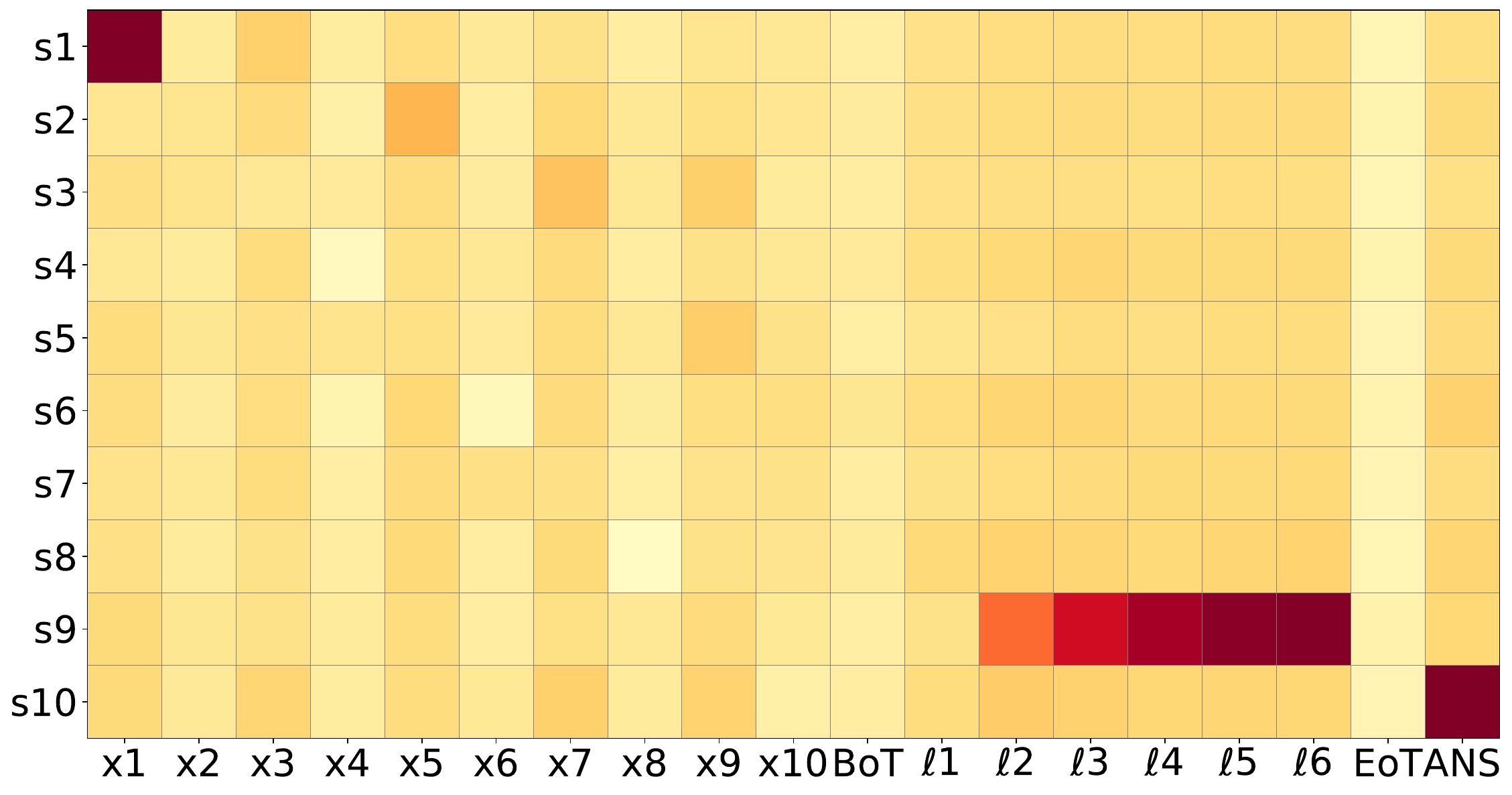}
    \caption{9-hop task}
  \end{subfigure}

  \caption{\textbf{Logit Lens analysis of the $n$-hop polynomial task with modulus $50$.}
  Only the final bridge state $s_n$ appears in the computation pathway, while earlier intermediate bridge states are collapsed.}
  \label{fig:n_hops_logic_lens}
\end{figure}

Activation patching provides complementary causal evidence for this direct $x_{n+1}\rightarrow$\texttt{[Ans]} route. When we patch the residual stream from a clean run into an $x_{n+1}$-corrupted run at the latent-step positions, accuracy does not recover (0\%), indicating that the latent computation does not carry an $x_{n+1}$ signal that causally affects the output. In contrast, applying the same patch at the \texttt{[Ans]} position restores performance in the late layers. In general, patching a clean run into runs corrupted at earlier inputs $x_i$ (for $i \le n$) yields nontrivial recovery  at specific latent reasoning steps, with larger $i$ generally producing stronger recovery.

It is striking that increasing the nominal hop count does not compel CODI to instantiate an explicit multi-step latent chain that sequentially produces $(s_1,\ldots,s_n)$. Instead, greater depth amplifies a \emph{late-stage bottleneck}: the latent trajectory is dominated by forming and refining a near-final intermediate (here, $s_n$), while the final input $x_{n+1}$ is routed through a separate, direct pathway. The model then produces the answer via late fusion at \texttt{[Ans]}. 
Even at a hop count as large as $n=31$, we continue to observe the previously described partial latent pathway. 


\subsection{How does task definition affect the observed mechanism?}

\begin{figure}[t]
  \centering
  \captionsetup[subfigure]{font=footnotesize,skip=2pt}

  \begin{minipage}{\linewidth}
    \includegraphics[width=0.5\linewidth,height=0.5cm]{Other_Figures/Color_bars_horizontal_2.png}
    \vspace{-0.75em}
  \end{minipage}

  \begin{subfigure}[t]{0.24\linewidth}
    \centering
    \includegraphics[width=\linewidth]{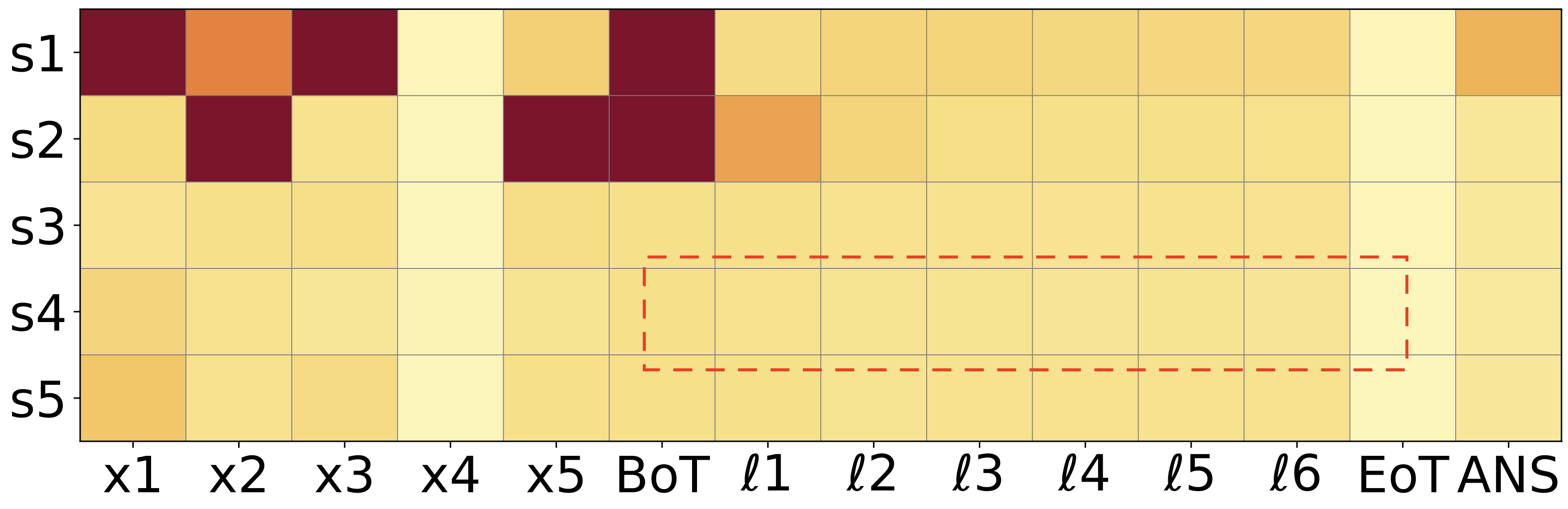}
    \caption{4-hop task}
  \end{subfigure}
  \hfill
  \begin{subfigure}[t]{0.24\linewidth}
    \centering
    \includegraphics[width=\linewidth]{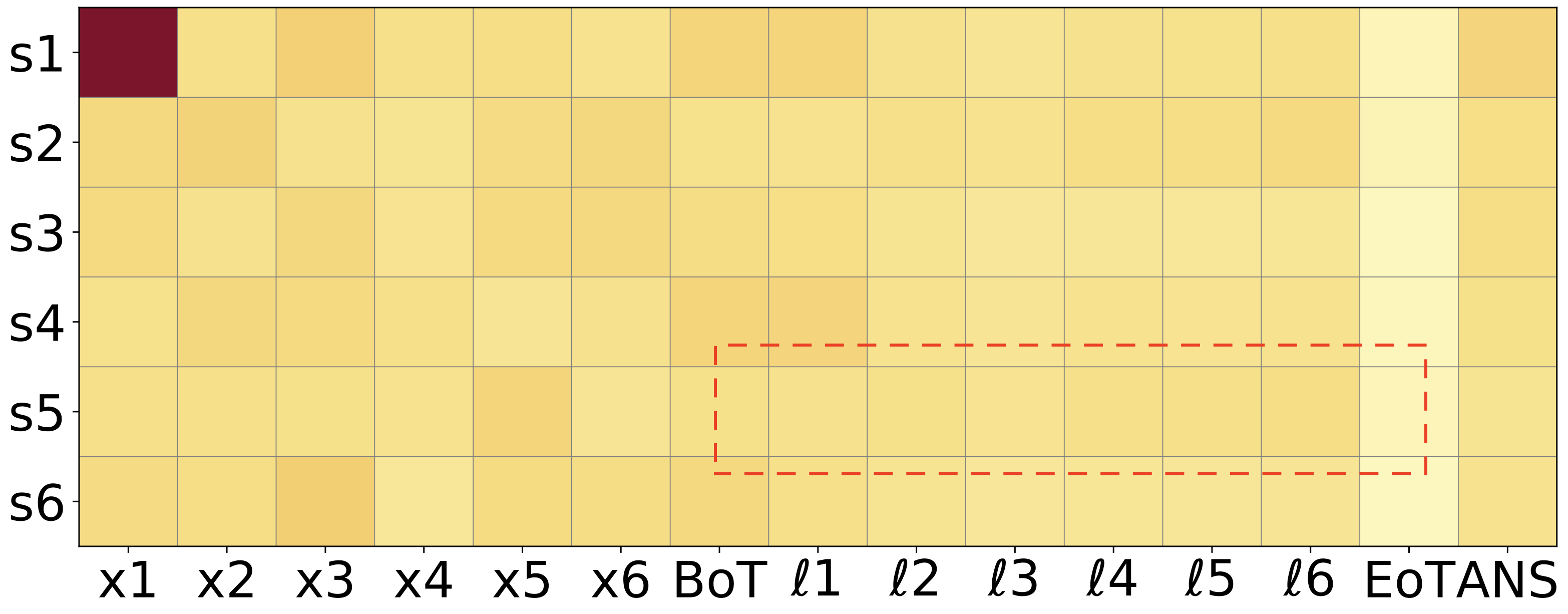}
    \caption{5-hop task}
  \end{subfigure}
  \hfill
  \begin{subfigure}[t]{0.24\linewidth}
    \centering
    \includegraphics[width=\linewidth]{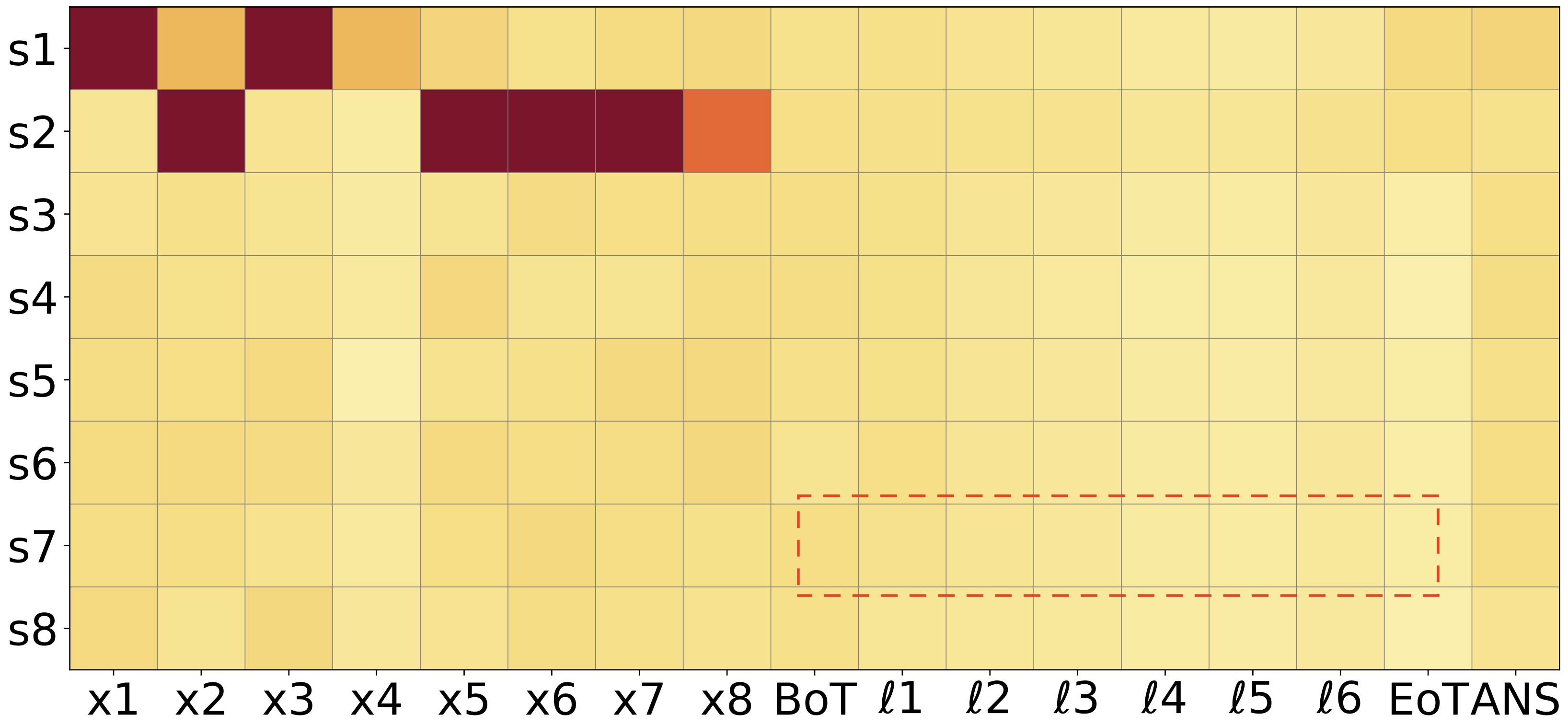}
    \caption{7-hop task}
  \end{subfigure}
  \hfill
  \begin{subfigure}[t]{0.24\linewidth}
    \centering
    \includegraphics[width=\linewidth]{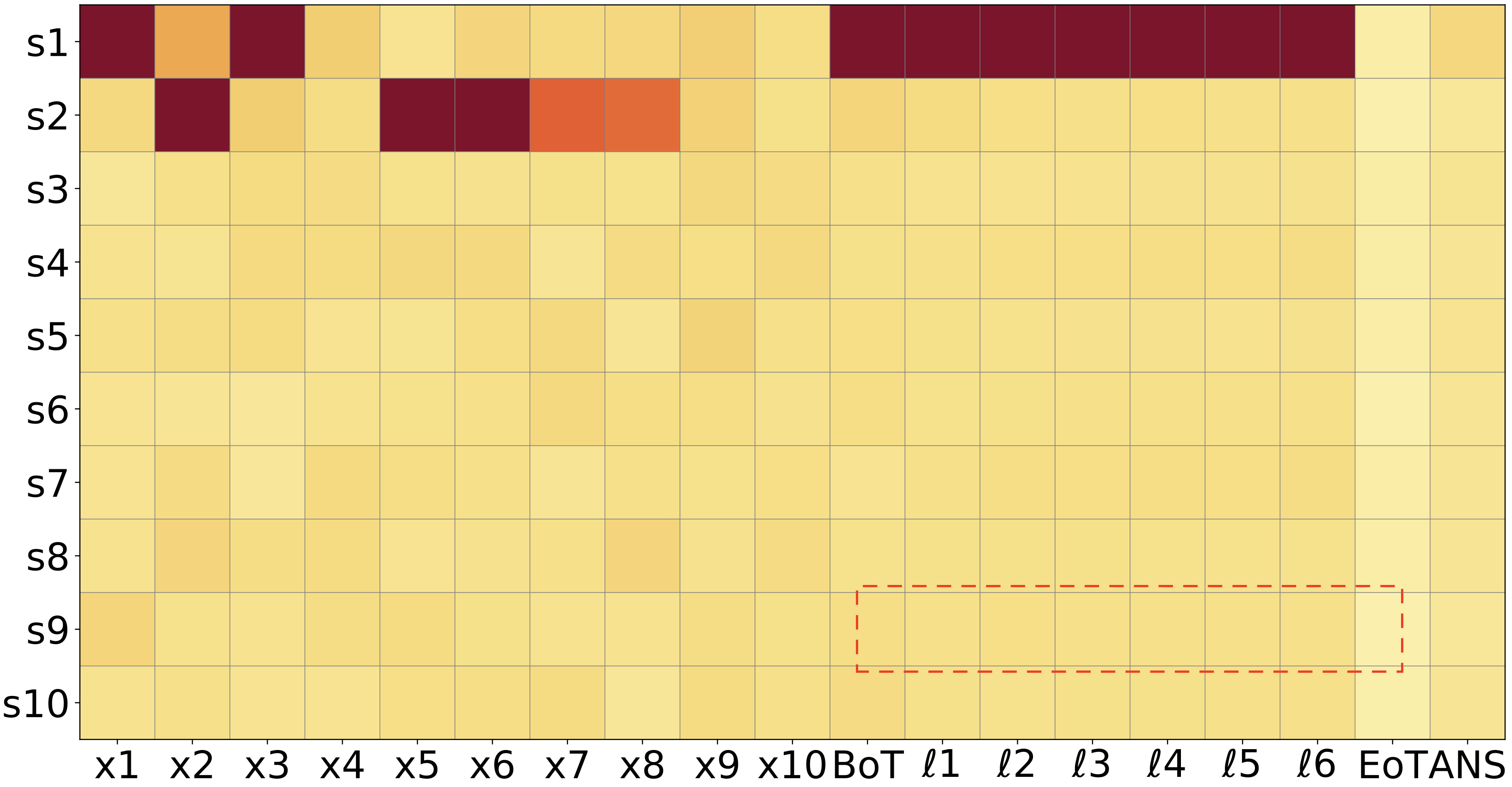}
    \caption{9-hop task}
  \end{subfigure}

  \caption{\textbf{Logit Lens analysis of the $n$-hop polynomial task with modulus $41$.}
  The final bridge state $s_n$ (in the dotted rectangular box) is no longer decodable along the latent trajectory, indicating a breakdown of step-by-step reasoning.}
  \label{fig:n_hops_logic_lens_mod41}
\end{figure}

Our default iteration rule is $
s_t \;=\; F(s_{t-1},x_t)\;=\; (s_{t-1}\cdot x_t + b)\pmod m,
$ 
with $b=1$ and $m=50$. We evaluate robustness to the task specification by varying the additive constant ($b\in\{1,3,4\}$) and the modulus ($m\in\{41,\ldots,50\}$). Changing $b$ does not qualitatively affect the behaviors reported above: across all tested values, we observe the same intermediate-state formation patterns and attention-routing signatures.

However, the choice of $m$ has a much larger effect. Under composite moduli, the mechanistic patterns remain largely unchanged. In contrast, when $m$ is prime, performance drops sharply (\cref{tab:accuracy_across_m} in Appendix \ref{app:full-discussion}) and the late-intermediate ``partial-rollout'' signature vanishes (i.e., $s_n$ is no longer decodable during the latent steps as seen in \cref{fig:n_hops_logic_lens_mod41}). We provide mathematical intuition for this behavior in the next section.

\section{Theoretical Analysis}
\label{sec:prime-vs-composite}

Motivated by our empirical finding that CODI collapses---and that the partial, late-bottleneck reasoning signature disappears---when the modulus is prime, we develop a theoretical analysis of the polynomial sequential tasks. We show that the same task family differs substantially in difficulty under prime versus composite moduli.

\textbf{Task setup.}
Fix a modulus $m\ge 2$ and consider the ring $R_m := \mathbb{Z}/m\mathbb{Z}$.
Given inputs $x_1,\dots,x_T \in \{1,\dots,m-1\}$, define a sequential state by
\begin{equation}
\begin{aligned}
s_1 &:= x_1, \\
s_t &:= f_{x_t}(s_{t-1}), := s_{t-1}x_t + b \pmod m, 
\end{aligned}
\label{eq:recurrence}
\end{equation}
for $t=2,\dots,T,$ and let the label be $y := s_T \in R_m$.
We analyze how the algebraic structure of $R_m$ changes the intrinsic dependency of $y$ on the prefix $(x_1,\dots,x_{T-k})$.

\textbf{Affine maps are permutations iff the multiplier is a unit.}
Let $f_x(s) = sx + b$ over $R_m$.

\begin{lemma}[Bijection criterion]
\label{lem:bijection}
For $x\in R_m$, the map $f_x : R_m \to R_m$ is bijective iff $x$ is a unit in $R_m$, i.e.\ $\gcd(x,m)=1$.
\end{lemma}
\noindent\emph{Proof.} See Appendix \ref{app:lemma_5_1_proof}.

\textbf{Composite moduli create many-to-one ``contractions.''}
When $\gcd(x,m)>1$, multiplication by $x$ collapses information.

\begin{lemma}[Exact contraction factor]
\label{lem:contraction}
Let $d := \gcd(x,m)$. Then the map $s\mapsto sx\ (\mathrm{mod}\ m)$ has image size $m/d$, and every output has exactly $d$ preimages.
Equivalently, $f_x(s)=sx+b$ is a $d$-to-$1$ map.
\end{lemma}

\noindent\emph{Proof.} See Appendix \ref{app:lemma_5_2_proof}.

Lemma~\ref{lem:contraction} gives a rigorous sense in which certain steps \emph{erase} state information: if $d>1$, then $s_{t}$ depends on $s_{t-1}$ only through a coarser equivalence class (a factor-$d$ quotient of the state space).

\textbf{Prime-field regime implies long-range dependence.}
Now suppose $m=p$ is prime and inputs are sampled uniformly from $\{1,\dots,p-1\}=\mathbb{F}_p^\times$.
Then $\gcd(x_t,p)=1$ always, so every step $f_{x_t}$ is a permutation of $\mathbb{F}_p$ (Lemma~\ref{lem:bijection}).
Hence the recurrence~\eqref{eq:recurrence} does \emph{not} contract the state space at any time: information about $s_{t-1}$ is never irreversibly discarded.
Unrolling the recurrence yields
\begin{equation}
s_T \equiv x_1\prod_{i=2}^{T} x_i \;+\; b\sum_{t=2}^{T}\ \prod_{i=t+1}^{T} x_i \pmod p,
\label{eq:unrolled}
\end{equation}
a high-degree polynomial in the inputs over the field $\mathbb{F}_p$ (\noindent\textit{Convention:}  $\prod_{i=T+1}^{T} x_i := 1$).
Because multiplication by nonzero elements is invertible, there is no algebraic mechanism that systematically annihilates the older product terms in~\eqref{eq:unrolled};
thus $y=s_T$ typically retains genuine dependence on the full history.

\textbf{Composite-ring regime induces short effective memory under uniform sampling.}
Now let $m$ be composite and sample $x_t \sim \mathrm{Unif}\{1,\dots,m-1\}$.
Define the unit probability
\begin{equation}
\begin{aligned}
u(m) &:= \Pr(\gcd(x_t,m)=1) = \frac{\varphi(m)}{m-1},\\
q(m) &:= 1-u(m).
\end{aligned}
\end{equation}
where $\varphi$ is Euler's totient function.
With probability $q(m)>0$, the final step multiplier is a non-unit and the last update is a many-to-one contraction by a factor $d=\gcd(x_T,m)$ (Lemma~\ref{lem:contraction}).
More generally, let
\begin{equation}
\tau := \max\{t\le T : \gcd(x_t,m)>1\},
\end{equation}
with the convention $\tau=0$ if all $x_1,\dots,x_T$ are units.
Then the length of the terminal all-unit suffix $L := T-\tau$ satisfies
\begin{equation}
\begin{aligned}
\Pr(L\ge k) &= u(m)^k \qquad (k=0,1,\dots,T),\\
\mathbb{E}[L] &= \sum_{k=1}^{T}u(m)^k
= \frac{u(m)\bigl(1-u(m)^T\bigr)}{1-u(m)}.
\end{aligned}
\label{eq:geom}
\end{equation}

Thus, for composite $m$ with $u(m)\ll 1$, there is typically a \emph{recent} contraction event $\tau$ close to $T$.
Each such event shrinks the number of distinguishable states by a factor $d=\gcd(x_\tau,m)$, so the label $y=s_T$ can often be predicted from a low-entropy summary of the past plus only the last few updates (the short suffix after $\tau$).

\textbf{Implications for learned representations.}
The above analysis is model-agnostic: it characterizes the dependency structure induced by Eq.~\eqref{eq:recurrence} under the data distribution. For prime $m=p$, each update is a bijection (Lemma~\ref{lem:bijection}), so the dynamics never contract; Eq.~\eqref{eq:unrolled} shows $s_T$ retains genuine dependence on the full history via multiplicative chains of length $\Theta(T)$. Mechanistically, long-horizon (large $T$) success therefore requires sustained state propagation and routing so the running state can be updated at each step and read out at \texttt{[Ans]}; with limited effective compute (e.g., few latent steps and limited depth/routing capacity), we expect degraded accuracy and missing latent rollouts, consistent with \cref{tab:accuracy_across_m} in Appendix \ref{app:full-discussion} and \cref{fig:n_hops_logic_lens_mod41}. For composite $m$, non-unit multipliers occur with probability $q(m)>0$ and induce explicit many-to-one contractions (Lemma~\ref{lem:contraction}); the trailing unit suffix length $L$ has geometric tails (Eq.~\eqref{eq:geom}), so recent contractions make $s_T$ largely depend on a low-entropy summary of the prefix plus the last few updates. This favors predictors that emphasize late intermediates, explaining why CODI often adopts late-bottleneck, partial (late-only) internal rollouts under composite moduli.

\section{Beyond Modular Arithmetic: Finite-State Function Composition}
\label{sec:fsc}


Section~\ref{sec:prime-vs-composite} suggests that the key factor behind the prime--composite split is not modular arithmetic itself, but whether state updates preserve or compress information. We therefore introduce FSC, a finite-state function-composition task that controls transition-level compressibility. 

Let $\mathcal{S}=\{0,\ldots,N-1\}$ be a finite state space and let $\mathcal{F}=\{f_1,\ldots,f_K\}$ be a fixed library of functions $f_i:\mathcal{S}\to\mathcal{S}$, sampled once and reused across train and test examples. Given an initial state $s_0$ and operation indices $\sigma_1,\ldots,\sigma_T$, the trajectory is
\[
s_t=f_{\sigma_t}(s_{t-1}), \qquad t=1,\ldots,T.
\]
The model receives $(s_0,\sigma_1,\ldots,\sigma_T)$ and predicts the intermediate trajectory and final state. 
FSC controls compressibility through the transition rank
$ r_i := |\operatorname{Im}(f_i)|. $ A fraction $\alpha$ of the library is sampled as bijections, with $r_i=N$, while the remaining functions are compressive maps with $r_i=m_i<N$. Thus FSC removes primes, composites, and gcd structure, while retaining the central variable from Section~\ref{sec:prime-vs-composite}: whether an update preserves or contracts state information. See Appendix~\ref{app:fsc-definition} for the complete dataset definition and construction details. 


The FSC theory generalizes the polynomial prime--composite analysis by replacing algebraic gcd structure with transition rank. Fully bijective libraries preserve all state distinctions under composition, making long-horizon prediction require sustained state propagation. Compressive transitions, by contrast, reduce the number of distinguishable prefix states, and later deterministic updates cannot recover the lost information. Under uniform operation sampling, the terminal all-bijective suffix has geometric tail, 
so compressive regimes typically contain recent rank contractions that make late-bottleneck reasoning viable. See Appendix~\ref{app:fsc-theory} for full statements and proofs.

FSC reproduces the main polynomial-task phenomena. On short-horizon tasks, CODI forms causally relevant bridge states in the latent trajectory. On longer compressive FSC tasks, it shifts to a late-bottleneck strategy in which only the final bridge state remains decodable, mirroring composite-modulus polynomial tasks. In the fully bijective regime ($\alpha=1$), this late-intermediate trace disappears on long horizons, paralleling the prime-modulus failure mode. Full experimental details are provided in Appendix~\ref{app:fsc-results}. A combined analysis of the polynomial and FSC tasks suggests that the transition
from faithful latent rollouts to partial or collapsed reasoning is not determined
by hop count alone, but by the effective task difficulty relative to the model
and training budget, with compressibility shifting this transition point. See
Appendix~\ref{app:transition_point} for details.

Moreover, FSC captures deterministic finite-state, operation-indexed tasks: any update rule $z_t=\delta(z_{t-1},a_t)$ can be represented by defining one transition function $f_i(z)=\delta(z,a_i)$ per operation. This includes finite automata, lookup-table systems, bounded finite-register systems, and the modular polynomial task itself. FSC does not directly cover unbounded-memory, continuous-state, stochastic, or richly structured-operation settings without additional modeling choices. Thus, FSC is a controlled abstraction for testing whether the compressibility mechanism observed in polynomial iteration persists beyond modular arithmetic; details are in Appendix~\ref{app:fsc-generality}. 


\section{Ablation Study}

We evaluate robustness on long-horizon instances ($n=31$) with a 3-layer, 2-head student and $m=50$. Varying the number of latent steps ($p\in\{1,2,3,6,9,12,20\}$) and sweeping the architecture (2--7 layers, 2--8 heads) does not change our qualitative mechanistic picture: CODI primarily encodes and routes late intermediates ($s_n$, or both $s_{n-1}, s_n$) and increasing $p$ does not reliably induce a deeper rollout. In loss ablations, removing feature distillation alone preserves these signatures, whereas removing both distillation and the teacher objective eliminates them, indicating that the teacher loss is the key driver of the late-bottleneck latent mechanism. More details are provided Appendix \ref{sec: ablation_study}.

\section{Conclusion}


We mechanistically analyzed CODI on polynomial-iteration tasks. With composite moduli, CODI is step-by-step on 2--3 hops but shifts to a late-intermediate partial rollout on longer horizons; with prime moduli, it largely fails to learn and shows neither signature. Our theory explains this split via \emph{compressibility}: composite moduli induce many-to-one contractions that bias the label toward late updates, making the computation effectively compressible and enabling CODI to succeed with a fixed latent budget. Prime-modulus instances lack this structure and are substantially less compressible. We further validate this account on finite-state function
composition tasks. These findings clarify when CODI-style latent CoT yields faithful iterative computation versus compressed late-stage solutions, and they highlight a key failure mode on sequential tasks that cannot be stably compressed. Additional discussion of latent CoT comparisons with non-CoT and CoT models, computational resources, and broader impacts is provided in Appendix~\ref{app:full-discussion}, Appendix~\ref{app:compute}, and Appendix~\ref{app:broad_impact}, respectively. 



\paragraph{Limitations \& future directions.}
Our study focuses on CODI-style latent CoT, where latent reasoning is learned through teacher--student distillation into a fixed number of continuous thought tokens. This controlled setting enables detailed mechanistic analysis, but further work is needed to determine how broadly the observed mechanisms extend to other latent-reasoning architectures, such as looped transformers, recurrent latent-state models, or hybrid latent-text approaches. In addition, our polynomial and finite-state sequential tasks provide ground-truth intermediate states and allow us to vary compressibility precisely, but they remain synthetic abstractions of reasoning. An important direction for future work is to test whether the \emph{compressibility principle} predicts latent-reasoning behavior in more naturalistic datasets, where ambiguity, redundancy, and many-to-one structure arise from language or real-world tasks rather than from a hand-designed construction.

\bibliographystyle{plainnat}
\bibliography{references}


\appendix

\section{Logit Lens}
\label{sec:logit_lens}

\paragraph{Goal.}
We use the \emph{logit lens} to test whether an internal representation already contains the correct intermediate state for the polynomial iteration task.
The logit lens is a lightweight diagnostic: it applies the model's own output decoder (the \emph{unembedding}) to hidden states from earlier layers and positions, yielding a distribution over discrete tokens/states that can be compared against ground truth.
In our setting, this lets us localize \emph{when} the model commits to the correct intermediate value and whether that commitment emerges progressively over latent computation (consistent with iterative updates) or unusually early (suggesting shortcut behavior), which is particularly informative for CODI-style latent reasoning.

\paragraph{Background: output logits.}
For a decoder-only Transformer with residual-stream vectors \(h^{(L)}_t \in \mathbb{R}^{d}\) at position \(t\) after the final layer \(L\), next-token logits are typically computed as
\begin{equation}
z_t \;=\; W_U\,\mathrm{LN}\!\left(h^{(L)}_t\right) + b_U,
\qquad
p(x_t=v \mid x_{<t}) \;=\; \mathrm{softmax}(z_t)_v,
\end{equation}
where \(W_U \in \mathbb{R}^{|V|\times d}\) is the unembedding matrix, \(b_U\) is a bias, \(\mathrm{LN}(\cdot)\) denotes the final layer normalization used before decoding, and \(V\) is the vocabulary.

\paragraph{Logit lens construction.}
At a chosen intermediate layer \(l\) and token position \(t\) (e.g., a pre-answer boundary token or a specific state/thought position), we take the residual stream \(h^{(l)}_t\) and map it to a distribution using the same decoder:
\begin{equation}
\tilde{z}^{(l)}_t \;=\; W_U\,\mathrm{LN}\!\left(h^{(l)}_t\right) + b_U,
\qquad
\tilde{p}^{(l)}_t \;=\; \mathrm{softmax}\!\left(\tilde{z}^{(l)}_t\right).
\end{equation}
We interpret \(\tilde{p}^{(l)}_t\) as: \emph{if the model were forced to decode a token from layer \(l\) at position \(t\) using its final unembedding (or an equivalent lightweight readout), which discrete value would it predict?}

\paragraph{Metrics \& Visualization.}
For the polynomial sequential task, given the ground-truth intermediate state \(s^\star\) (e.g., \(s_t\) or the final \(s_n\) in an \(n\)-hop instance), we measure the logit-lens probability assigned to \(s^\star\), i.e., \(\tilde{p}^{(l)}_t(s^\star)\).
We plot \(\tilde{p}^{(l)}_t(s^\star)\) (averaged over examples and all layers) as a function of token position \(t\), including CODI latent-thought positions, to localize \emph{when} the correct intermediate value becomes decodable and to distinguish progressive latent construction from early commitment.

\section{Attention Maps}
\label{app:attention_maps}

\paragraph{What an attention map shows.}
An \emph{attention map} visualizes where a self-attention head reads from when updating each token representation.
For a fixed layer \(l\) and head \(h\), each query position \(t\) forms a weighted average of information from (allowed) key/value positions \(j\).
The attention map is the matrix of these normalized weights, with rows corresponding to queries \(t\) and columns to keys \(j\).

\paragraph{Self-attention mechanics.}
Let \(x_t \in \mathbb{R}^{d}\) denote the residual-stream vector at position \(t\) entering the attention sublayer.
A single head constructs
\begin{equation}
q_t = W_Q x_t,\qquad k_j = W_K x_j,\qquad v_j = W_V x_j,
\end{equation}
and computes scaled dot-product attention scores followed by a softmax:
\begin{equation}
\alpha_{t \rightarrow j}
=\mathrm{softmax}_{j}\!\left(\frac{q_t^\top k_j}{\sqrt{d_k}} + m_{t j}\right),
\end{equation}
where \(m_{t j}\) is an attention mask (e.g., a causal mask sets \(m_{t j}=-\infty\) for future positions \(j>t\)).
The head output at position \(t\) is
\begin{equation}
\mathrm{Attn}(x)_t
=\sum_{j} \alpha_{t \rightarrow j}\, v_j.
\end{equation}
Stacking \(\alpha_{t\rightarrow j}\) over all \(t\) and \(j\) yields an attention matrix
\(A^{(l,h)} \in \mathbb{R}^{T \times T}\), where each row sums to \(1\).

\paragraph{How we visualize it.}
We plot \(A^{(l,h)}\) as a heatmap: \emph{rows} index query positions \(t\) (the positions being updated) and \emph{columns} index key/value positions \(j\) (the positions being attended to).
Bright entries indicate large \(\alpha_{t\rightarrow j}\), i.e., strong routing from position \(j\) into the update at position \(t\).
In our setting, positions include input tokens, structural boundary tokens (e.g., \texttt{[BoT]}, \texttt{[EoI]}, \texttt{[Ans]}), and (for CODI) continuous thought tokens ($[\ell_1]...[\ell_6]$. Unless otherwise noted, we visualize individual heads and often average \(A^{(l,h)}\) over a batch of examples to highlight stable routing structure. We annotate rows/columns by token type (input vs.\ thought vs.\ boundary) to clarify information flow. 

\cref{fig:2hop_L3H2_attn_map} visualizes attention maps from the student of a 3-layer, 2-head Transformer trained with the CODI objective on the two-hop polynomial task. Several heads exhibit strong attention from the \texttt{[Ans]} position (immediately before answer generation) to the final input token \(x_3\), consistent with a copy-like pathway that routes \(x_3\) directly into the \texttt{[Ans]} residual stream, and \cref{fig:n_hops_attn_head} shows that The \texttt{[Ans]} token places strong attention on the final input token $x_{n+1}$ across multiple n-hop tasks, indicating a copy-like pathway that routes $x_{n+1}$ directly into the answer readout.

\begin{figure}[ht]
  \begin{center}
    \centerline{\includegraphics[width=0.6\columnwidth]{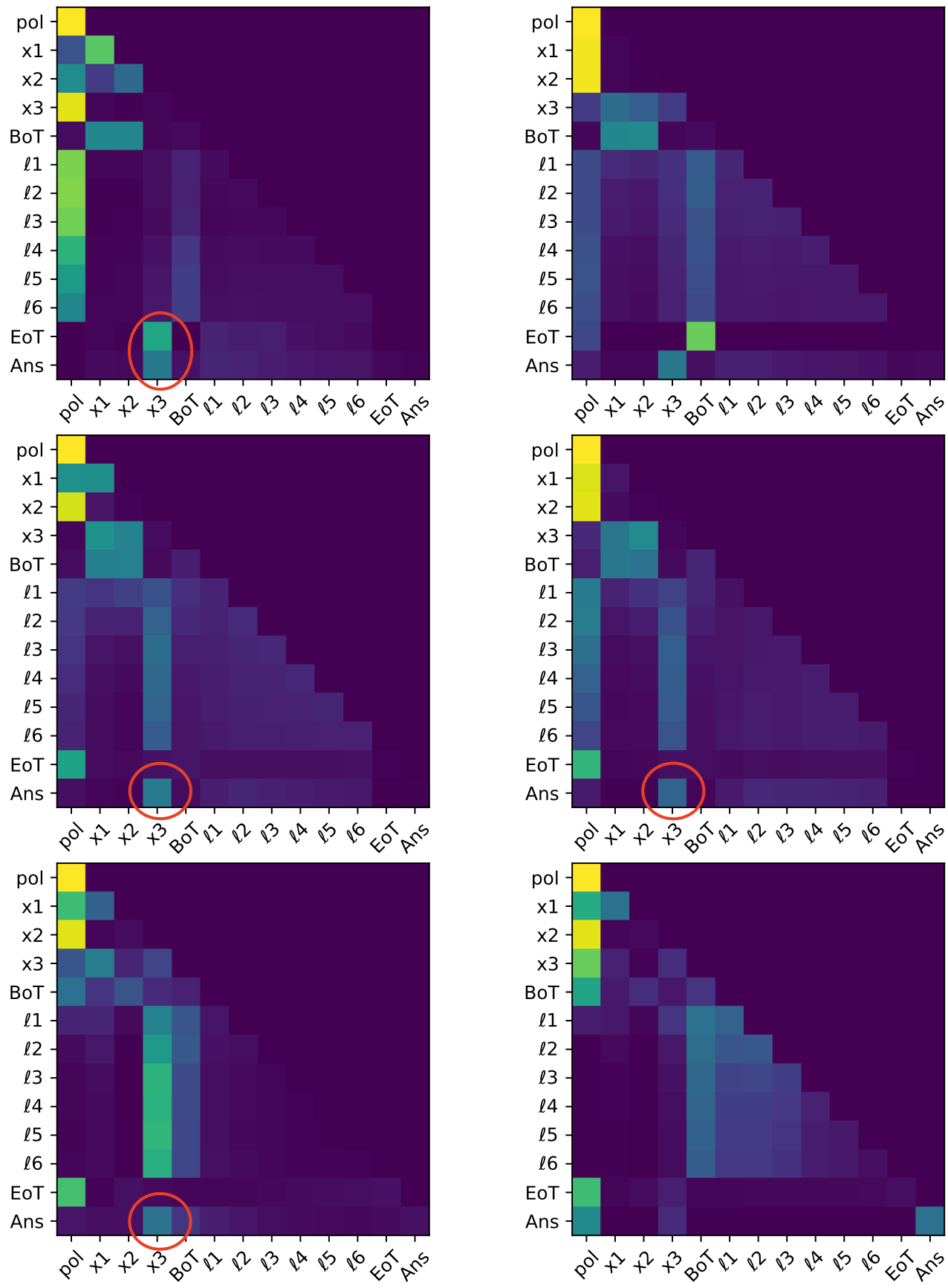}}
    \caption{
\textbf{Attention maps for the 3-layer, 2-head transformer on the two-hop polynomial task.}
Rows correspond to layers (Layer 1 on top through Layer 3 on bottom), and columns correspond to heads (Head 1 on the left, Head 2 on the right). Across multiple heads, the \texttt{[Ans]} token places strong attention directly on the $x_3$ position (circled in red ovals), consistent with a copy-like pathway that routes $x_3$ into the \texttt{[Ans]} residual stream. 
}
    \label{fig:2hop_L3H2_attn_map}
  \end{center}
\end{figure}

\begin{figure}[ht]
  \centering

  \begin{subfigure}[t]{0.24\linewidth}
    \centering
    \includegraphics[width=\linewidth]{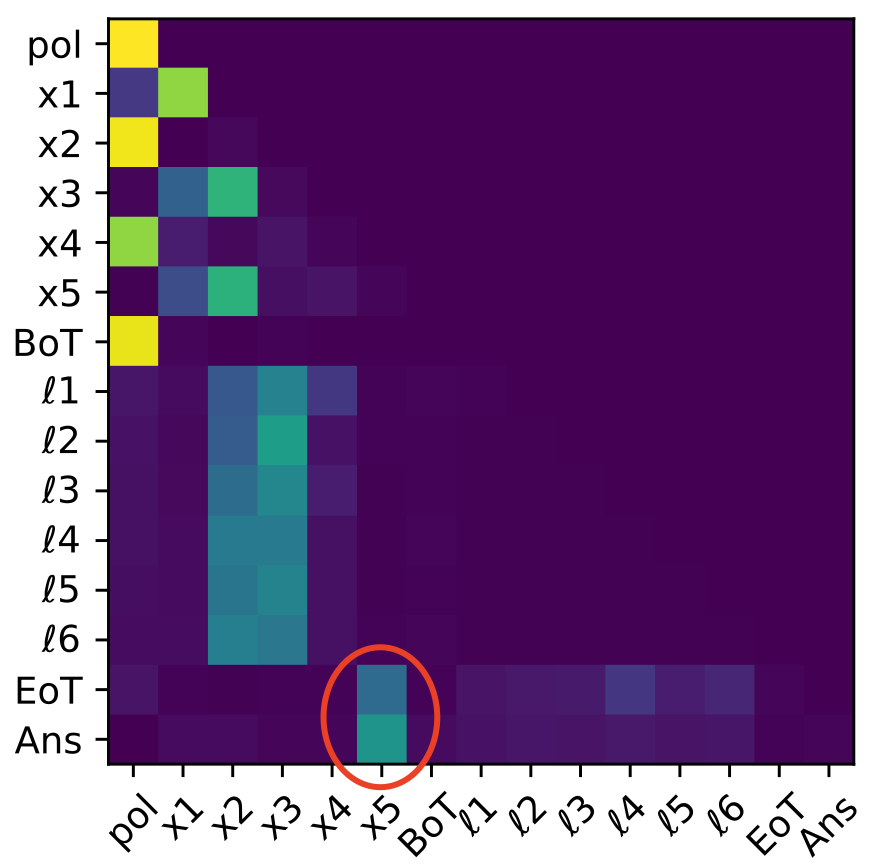}
    \caption{4-hop task}
  \end{subfigure}\hfill
  \begin{subfigure}[t]{0.24\linewidth}
    \centering
    \includegraphics[width=\linewidth]{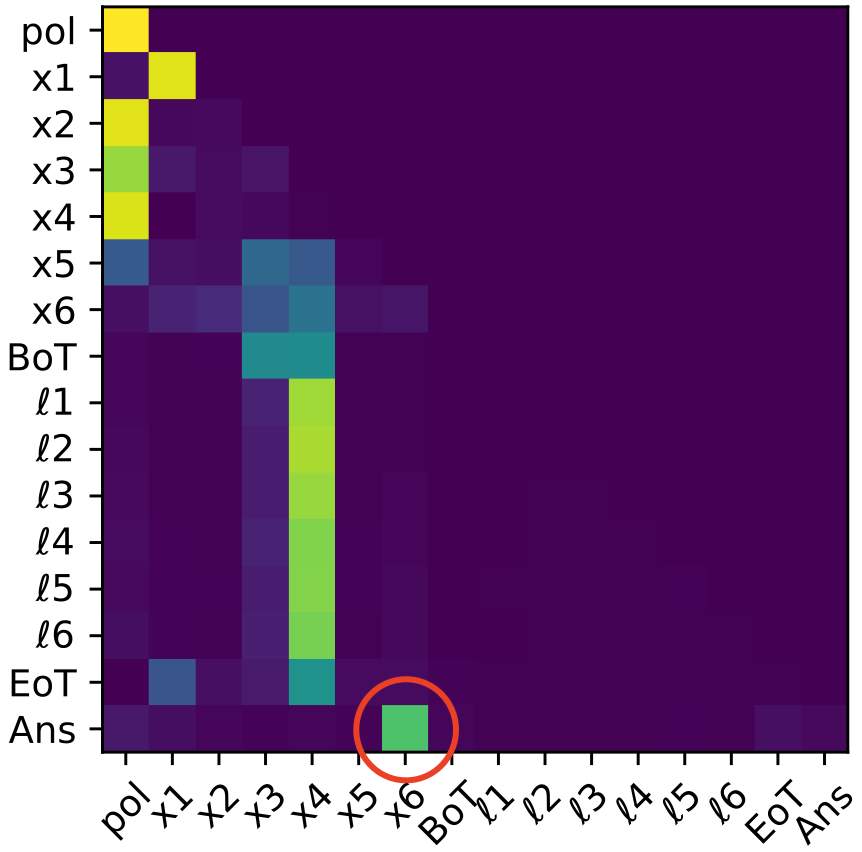}
    \caption{5-hop task}
  \end{subfigure}\hfill
  \begin{subfigure}[t]{0.24\linewidth}
    \centering
    \includegraphics[width=\linewidth]{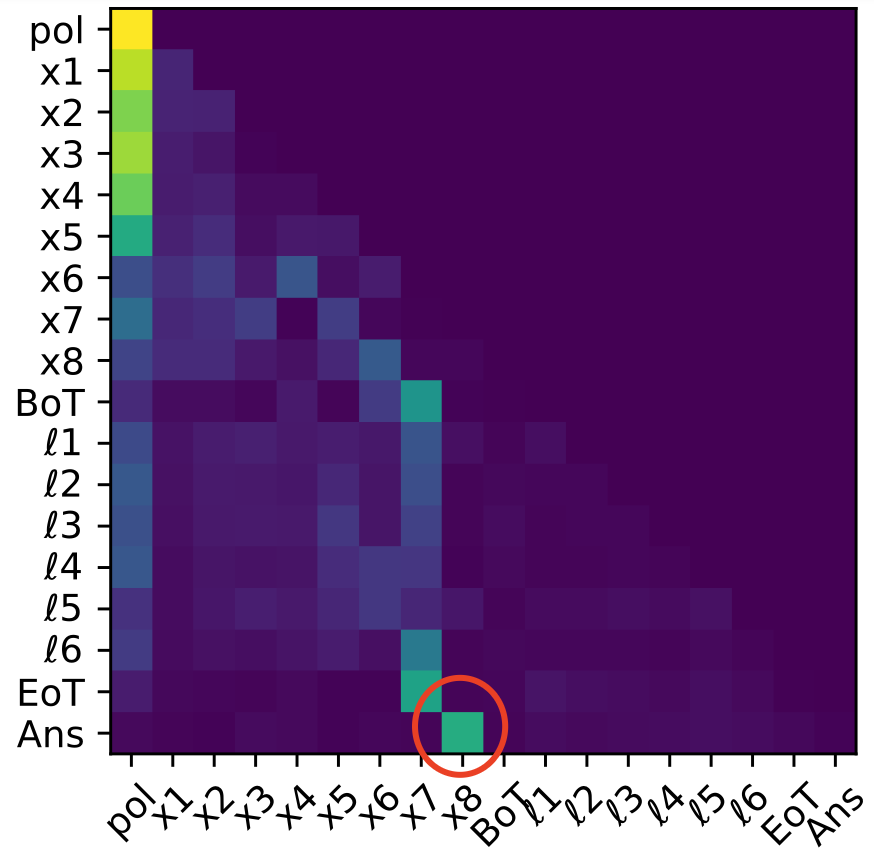}
    \caption{7-hop task}
  \end{subfigure}\hfill
  \begin{subfigure}[t]{0.24\linewidth}
    \centering
    \includegraphics[width=\linewidth]{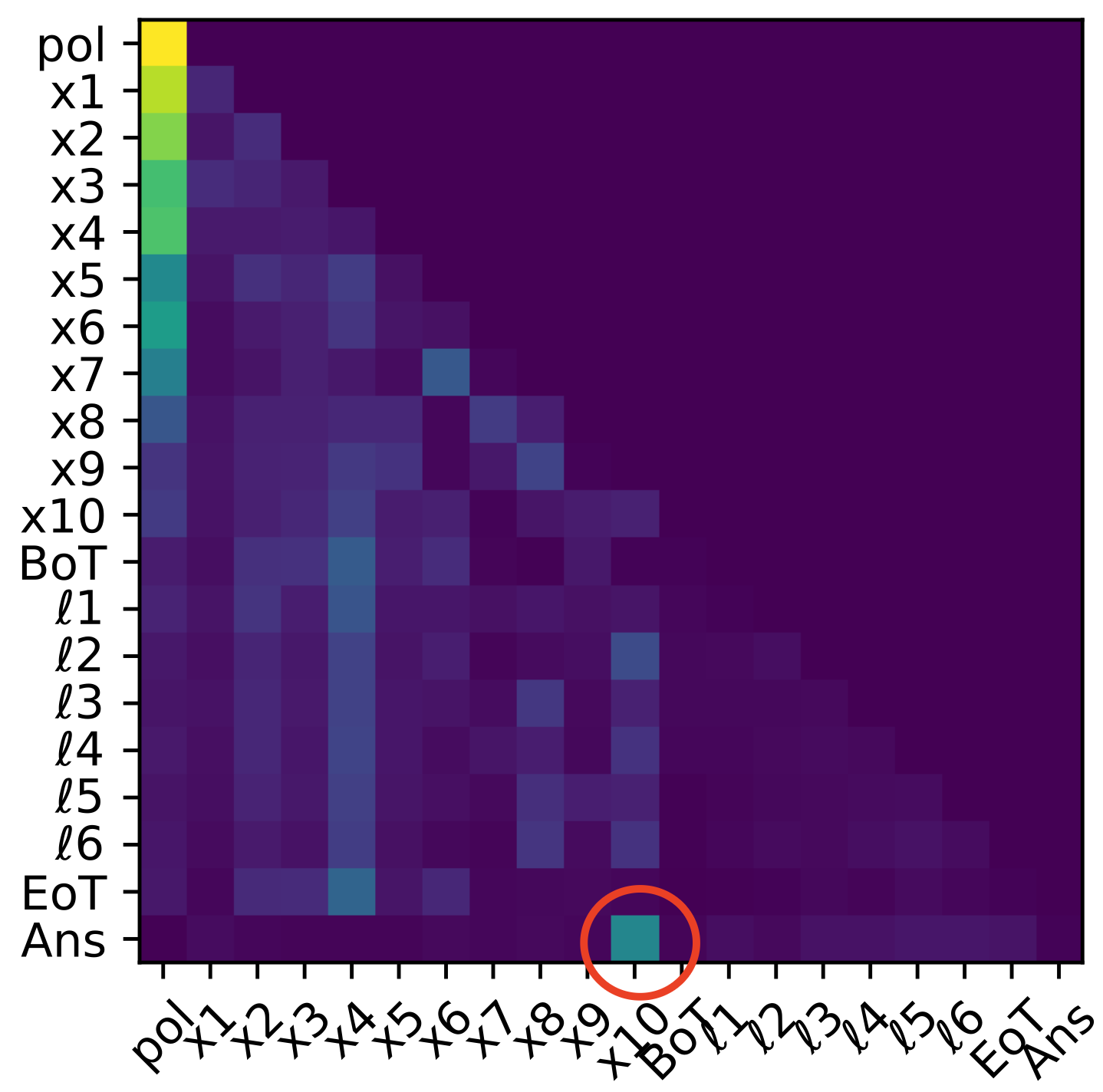}
    \caption{9-hop task}
  \end{subfigure}

  \caption{\textbf{Specialized attention head on the $n$-hop polynomial task}. The \texttt{[Ans]} token attends strongly to the final input token, $x_{n+1}$, suggesting a generalized pathway in which $x_{n+1}$ is directly routed to \texttt{[Ans]}.}
  \label{fig:n_hops_attn_head}
\end{figure}

\section{Probing Intermediate Representations}
\label{app:probing}

\paragraph{Goal.}
\emph{Probing} is a diagnostic technique for quantifying what information is present in a model's hidden states.
Given intermediate activations (e.g., residual-stream vectors) produced while the model processes an input, we train a simple supervised predictor (a \emph{probe}) to recover a task-relevant variable from those activations.
If a low-capacity probe can accurately predict a variable (e.g., an intermediate state), this provides evidence that the variable is encoded in the representation in an easily accessible (often approximately linear) form.

\paragraph{Representations and targets.}
For each example, we extract hidden states \(h^{(l)}_t \in \mathbb{R}^d\) at layer \(l\) and token position \(t\).
In the polynomial sequential task, we train probes to predict task-relevant variables, including the current state \(s_t\), the next state \(s_{t+1}\), the final answer \(s_{n+1}\), and (for an \(n\)-hop instance) each input token \(x_1,\ldots,x_{n+1}\).
We evaluate probes across token types/positions, including input tokens, CODI continuous thought tokens, and pre-answer boundary tokens.

\subsection{Linear Probing Implementation Details}
\label{app:linear_probing_impl}

\paragraph{Overview.}
We use linear probes to quantify which task variables are linearly decodable from CODI's internal representations.
Probes are trained on frozen activations extracted from runs where the model predicts the correct final answer, and we report held-out classification accuracy.

\paragraph{Probe architecture.}
Each probe is a single linear classifier with no bias and no nonlinearity:
\begin{equation}
f(h)=Wh,
\end{equation}
where \(h \in \mathbb{R}^{256}\) is the hidden vector and \(W \in \mathbb{R}^{50 \times 256}\).
The 50 classes correspond to discrete values \(\{0,1,\ldots,49\}\).

\paragraph{Activation extraction (where we probe).}
For a sequence length \(n\), we probe representations at the following positions:
\begin{itemize}
  \item \textbf{Input positions:} $x_i$ for \(i \in \{1,\ldots,n+1\}\) (residual stream at each input token)
  \item \textbf{BoT (Beginning of Thought):} final encoder position before latent processing
  \item \textbf{Latent thought tokens:} \(\ell_j\) for \(j \in \{1,\ldots,6\}\) (six learned latent tokens)
  \item \textbf{EoT (End of Thought):} first decoder position after latent processing
  \item \textbf{Ans:} decoder position used for answer generation
\end{itemize}
At each position, we probe the residual stream at four depths: before the first layer (\texttt{L1-Pre}) and after layer 1, 2, 3 (\texttt{L\{0,1,2\}-post}).

\paragraph{Probe targets (what we decode).}
We train separate probes for each of the following ground-truth labels:
\begin{itemize}
  \item \textbf{Inputs:} \(x_1,\ldots,x_n\) with \(x_i \in \{1,\ldots,49\}\)
  \item \textbf{Intermediate states:} \(s_1,\ldots,s_{n-1}\) with \(s_i \in \{0,\ldots,49\}\)
  \item \textbf{Final answer:} \(\texttt{ans} \in \{0,\ldots,49\}\)
\end{itemize}

\paragraph{Dataset and filtering.}
For each sequence length, we extract activations from approximately \(5{,}000\) test examples.
To focus on representations associated with successful computation, we retain only examples where the model's final prediction is correct, yielding \(N_{\text{correct}}\) samples.

\paragraph{Train/validation/test splits.}
We split \(N_{\text{correct}}\) into 80\% train and 20\% test.
The training portion is further split into 80\% for optimization and 20\% for validation, giving an overall 64\%/16\%/20\% train/val/test split.

\paragraph{Optimization.}
Each probe is trained independently using Adam~\citep{kingma2014adam} (learning rate \(10^{-3}\)), batch size 64, for 100 epochs, minimizing cross-entropy loss.
We select the checkpoint with the highest validation accuracy and report test accuracy for that probe.

\paragraph{Number of probes.}
For sequence length \(n\), we train probes over:
(i) \textbf{locations:} \(n\) encoder positions + 1 BoT + 6 latents + 1 EoT + 1 Ans \(= n+9\);
(ii) \textbf{depths:} 4 residual-stream depths; and
(iii) \textbf{labels:} \(n\) inputs + \((n-1)\) intermediate states + 1 answer \(=2n\).
This yields \((n+9)\times 4 \times 2n\) probe fits per sequence length (e.g., 896 probes when \(n=7\)).

\paragraph{Metric.}
We report classification accuracy on the held-out test split.
Chance performance is \(1/50 = 0.02\); accuracy near 1.0 indicates the target variable is fully linearly decodable from the probed representation.

\paragraph{Probing Visualization.}
\cref{fig:probing_seq_3} and \cref{fig:2_hop_probing_x3} visualize linear-probe results on the two-hop polynomial task (modulus 50), with probe targets \(s_2\) (the intermediate bridge state) and \(x_3\) (the final input), respectively. Each cell reports classification accuracy on the held-out test split; the \(x\)-axis indexes the token positions from which activations are extracted, and the \(y\)-axis indexes the four probed residual-stream depths at each position. In \cref{fig:probing_seq_3}, \(s_2\) is highly decodable across the latent thought tokens \([\ell_1]\)–\([\ell_6]\), indicating that the latent channel carries the intermediate-state information. In contrast, \cref{fig:2_hop_probing_x3} shows that \(x_3\) is largely not decodable within the latent tokens, but becomes strongly decodable at \texttt{[EoT]} and \texttt{[Ans]}. This pattern is consistent with the copy-like route suggested by the attention maps, where \(x_3\) is routed directly into the final readout.

\begin{figure}[ht]
  \begin{center}
    \centerline{\includegraphics[width=0.75\columnwidth]{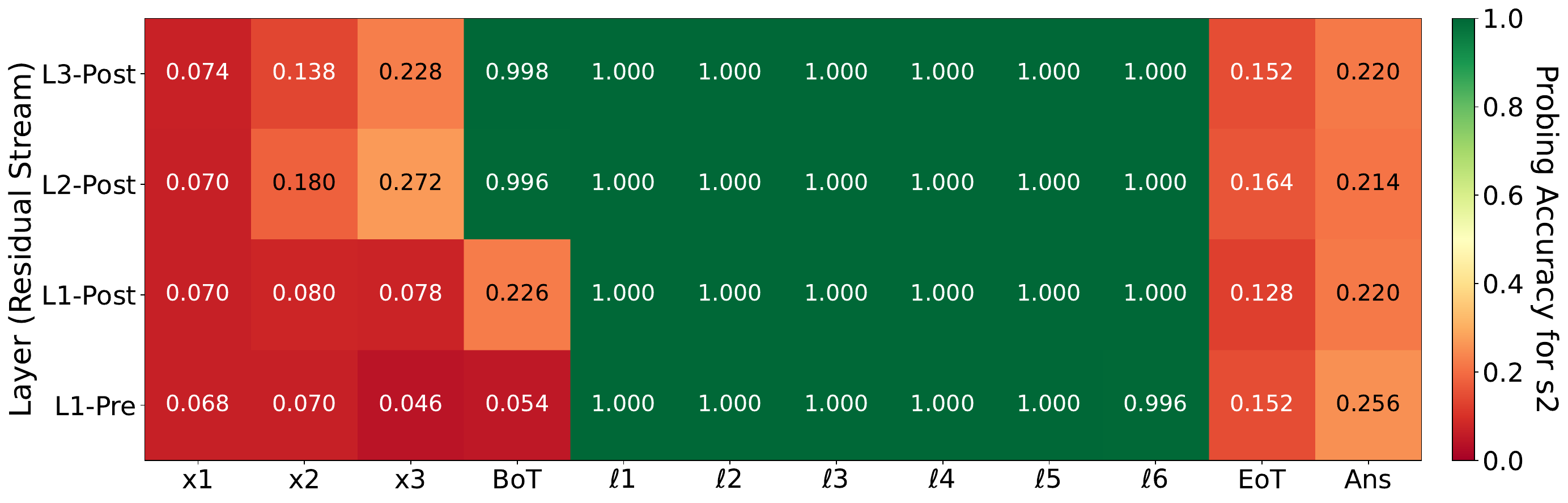}}
    \caption{
\textbf{Probing the intermediate state $s_2$ in the two-hop polynomial task.}
Consistent with the logit-lens results, $s_2$ is readily decodable throughout the latent computation, indicating that the bridge state is formed before the model produces the final answer and supporting a sequential reasoning strategy.}
    \label{fig:probing_seq_3}
  \end{center}
\end{figure}

\begin{figure}[ht]
  \begin{center}
    \centerline{\includegraphics[width=0.75\columnwidth]{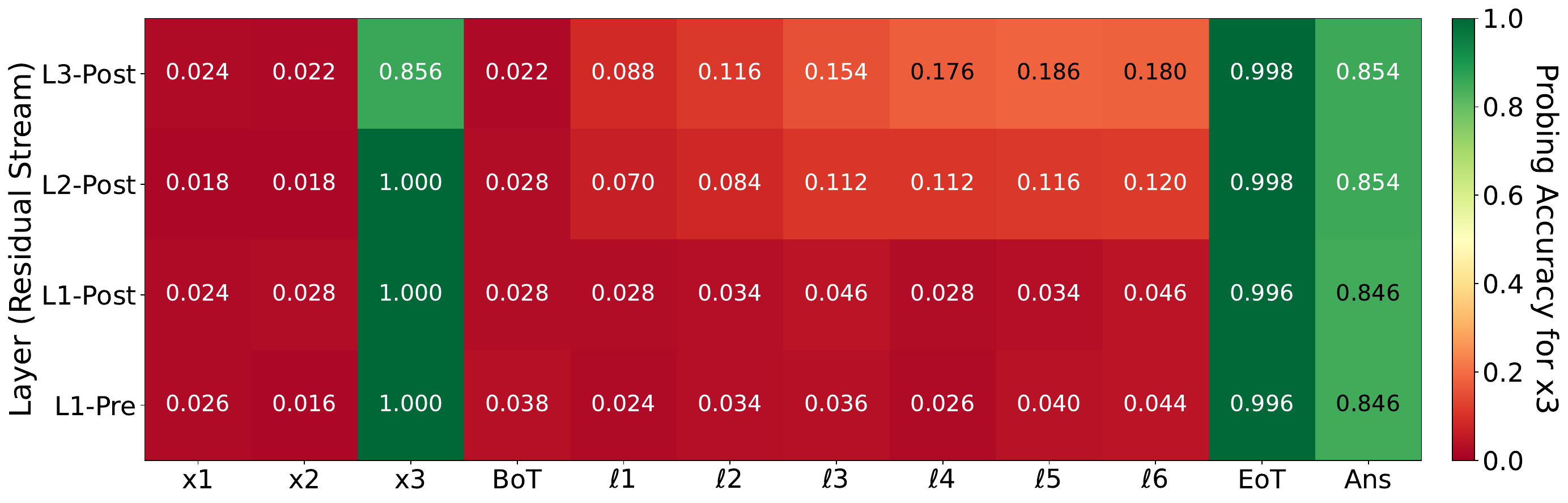}}
    \caption{
\textbf{Probing the input token $x_3$ in the two-hop polynomial task.}
$x_3$ is strongly decodable at its own position and at the \texttt{[EoT]} and \texttt{[Ans]} tokens (probe confidence $\approx 1$ at $x_3$ and \texttt{[EoT]}, and $\approx 0.85$ at \texttt{[Ans]}). Consistent with the attention pattern in \cref{fig:2hop_L3H2_attn_map}, this supports a direct routing (copy-like) pathway from $x_3$ into the \texttt{[EoT]} and \texttt{[Ans]} residual streams.}
    \label{fig:2_hop_probing_x3}
  \end{center}
\end{figure}

\section{Activation Patching}
\label{app:activation_patching}

\paragraph{Goal.}
\emph{Activation patching} is a causal intervention used to identify which internal components of a model are necessary (or sufficient) for a particular behavior.
The core idea is to compare a \emph{clean} run of the model (which produces the correct answer) to a \emph{corrupted} run (where an input is perturbed and the model fails), and then \emph{replace} a chosen internal activation in the corrupted run with the corresponding activation from the clean run.
If this replacement restores the correct output, it provides causal evidence that the patched activation carries information that the model needs to solve the task.

\subsection{Activation Patching Implementation Details}

\paragraph{Overview.}
We employ activation patching~\citep{meng2022locating,wanginterpretability}, also known as causal tracing or interchange intervention, to identify which model components are causally necessary for correct predictions. Our implementation extends standard activation patching to handle the encoder-latent-decoder architecture with multiple latent reasoning tokens.

\paragraph{Clean and Corrupted Runs.}
For each test example with input sequence $\mathbf{x} = (x_1, \ldots, x_n)$ and correct answer $y$, we define:

\begin{itemize}
    \item \textbf{Clean run}: Forward pass with correct inputs, storing all intermediate activations $\mathcal{A}^{\text{clean}} = \{a_{\ell}^{(p)} : \ell \in \mathcal{L}, p \in \mathcal{P}\}$, where $\mathcal{L}$ indexes layers and $\mathcal{P}$ indexes positions/phases.
    \item \textbf{Corrupted run}: Forward pass with perturbed input $\tilde{\mathbf{x}}$, where we corrupt one token: $\tilde{x}_i = (2 \cdot x_i) \bmod 50$ for some position $i$.
\end{itemize}

We cache activations from three computation phases: encoder positions $\{1, \ldots, n\}$, latent tokens $\{\ell_1, \ldots, \ell_6\}$, and decoder positions $\{\text{EoT}, \text{Ans}\}$.

\paragraph{Intervention Procedure.}
Given corrupted input $\tilde{\mathbf{x}}$ and clean activation cache $\mathcal{A}^{\text{clean}}$, we perform the following intervention at component $c = (\ell, p)$ (layer $\ell$, position/phase $p$):

\begin{equation}
a_{\ell}^{(p)}(\tilde{\mathbf{x}}) \leftarrow a_{\ell}^{(p)}(\mathbf{x})
\end{equation}

We implement this via forward hooks in TransformerLens~\citep{nanda2022transformerlens}. For position-specific patching at encoder token $i$ or decoder position $j$, we patch only the corresponding slice: $a_{\ell, i}^{\text{enc}}(\tilde{\mathbf{x}}) \leftarrow a_{\ell, i}^{\text{enc}}(\mathbf{x})$.

\paragraph{Components Analyzed.}
We systematically patch residual stream activations at the following components:
\begin{itemize}[leftmargin=*,itemsep=0pt]
    \item \textbf{Layers}: Pre-residual (before layer 1: \texttt{L1-pre}) and post-residual after each layer $k \in \{ 1, 2, 3 \}$ (\texttt{Lk-post})
    \item \textbf{Phases}: For sequence length $n$, we test $n$ encoder positions + 1 BoT position + 6 latent tokens + 2 decoder positions, yielding $n+9$ phases
\end{itemize}

For a model with $L=3$ layers and sequence length $n$, this produces $4 \times (n+9)$ total intervention experiments per corrupted example.

\paragraph{Metrics.}
For each intervention, we compute:
\begin{align}
\text{Baseline} &= \mathbb{P}(\hat{y} = y \mid \tilde{\mathbf{x}}) \quad \text{(no patching)} \\
\text{Patched}_{c} &= \mathbb{P}(\hat{y} = y \mid \tilde{\mathbf{x}}, \text{patch at } c) \\
\text{Clean} &= \mathbb{P}(\hat{y} = y \mid \mathbf{x}) \quad \text{(upper bound)} \\
\text{Lift}_{c} &= \frac{\text{Patched}_{c} - \text{Baseline}}{\text{Clean} - \text{Baseline}} \, .
\end{align}

The \emph{lift} metric normalizes the recovery to $[0, 1]$, where $\text{Lift}_c = 1$ indicates complete recovery and $\text{Lift}_c = 0$ indicates no improvement. We report lift as the primary metric, as it is invariant to baseline difficulty and allows comparison across different corruptions.

\paragraph{Evaluation Protocol.}
We filter test examples to include only those where the model initially predicts correctly ($\mathbb{P}(\hat{y} = y \mid \mathbf{x}) = 1$), ensuring clean runs provide a valid counterfactual. We corrupt each input position independently, running separate patching experiments for $x_1, x_2, \ldots, x_n$. For each corruption, we compute the mean accuracy and lift over all correctly-predicted examples. 

\subsection{Visualization Details.} 

\cref{fig:activation_patching_x2_seq_3,fig:activation_patching_x3_seq_3} report activation-patching lift for runs corrupted at \(x_2\) and \(x_3\), respectively. The \(x\)-axis indexes the token position (inputs and latent tokens) where we patch, and the \(y\)-axis indexes the probed residual-stream depth. Each cell shows mean percent recovery, \(\mathrm{Acc\ Recovery} = 100\% \cdot \mathrm{Lift}_c\), averaged over clean-correct examples. In \cref{fig:activation_patching_x2_seq_3}, patching into the \(x_2\)-corrupted run yields the largest recovery at latent positions \([\ell_1]\) and \([\ell_2]\), indicating that early latent steps carry causally necessary intermediate information. In contrast, \cref{fig:activation_patching_x2_seq_3} shows negligible recovery from patching at latent positions for the \(x_3\)-corrupted run, while patching at \texttt{[Ans]} produces strong recovery, consistent with a copy-like route from the final input \(x_3\) to the \texttt{[Ans]} readout.

\section{Training Configuration}
\label{app:CODI_training_detail}

We train a 3-layer transformer with 2 attention heads on the polynomial reasoning dataset. The model uses 6 latent tokens and a 256-dimensional projection layer. We train for 1{,}000 epochs with batch size 256, using AdamW with learning rate $3\times 10^{-4}$ and a cosine annealing schedule. We use a warmup ratio of 0.03, weight decay of 0.1, and gradient clipping with max norm 2.0. The objective combines cross-entropy and distillation losses with equal weights (1.0 each); the distillation loss is normalized by its standard deviation. Each run uses 2{,}500 examples per sequence length. For an $n$-hop task, we train on a curriculum of sequence lengths $1,2,\ldots,n{+}1$ (e.g., 4-hop training includes lengths 1--5). All models are trained in BF16 precision for computational efficiency.

\section{Mechanistic Analysis on Three-hop Polynomial  task.}
\label{app:three_hop_analysis}

From the logit-lens analysis in \cref{fig:logic_lens_3hop_L3H2}, we observe a clear temporal progression of intermediate-state decodability across latent steps. The first state \(s_1\) becomes decodable early in the latent trajectory (specifically from \texttt{[BoT]} through \texttt{[L6]} in this setting). Later in the trajectory, the second intermediate state \(s_2\) becomes decodable at the \texttt{[EoT]} token. Together, these patterns are consistent with step-by-step computation on the three-hop polynomial task.

Activation patching provides causal support for this interpretation. As shown in \cref{fig:3_hop_patch_x1,fig:3_hop_patch_x2}, patching from a correct run into an \(x_1\)-corrupted or \(x_2\)-corrupted run yields substantial accuracy recovery primarily in early latent steps (\(l_1\) and \(l_2\)). Because \(s_2\) is computed from \(x_1\) and \(x_2\), this indicates that the latent representation supporting \(s_2\) is causally relevant in these steps. Moreover, \cref{fig:3_hop_patch_x3} shows that patching \(x_3\) has its strongest effect at \texttt{[EoT]}, consistent with \(s_3\) being formed and used late in the trajectory. Overall, both analyses support a sequential, step-by-step reasoning process.

It is important to note that this step-by-step reasoning pattern does not emerge uniformly across all polynomial-task runs. In some cases, the model exhibits a \emph{partial} latent trajectory: only \(s_3\) becomes decodable in the latent steps, without clear evidence of \(s_2\) formation. That said, across the different random seeds and small model architectures we tested, the full step-by-step pattern appears reliably in the majority of settings.

\begin{figure*}[ht]
  \centering
  \begin{subfigure}[t]{0.49\textwidth}
    \centering
    \includegraphics[width=\linewidth]{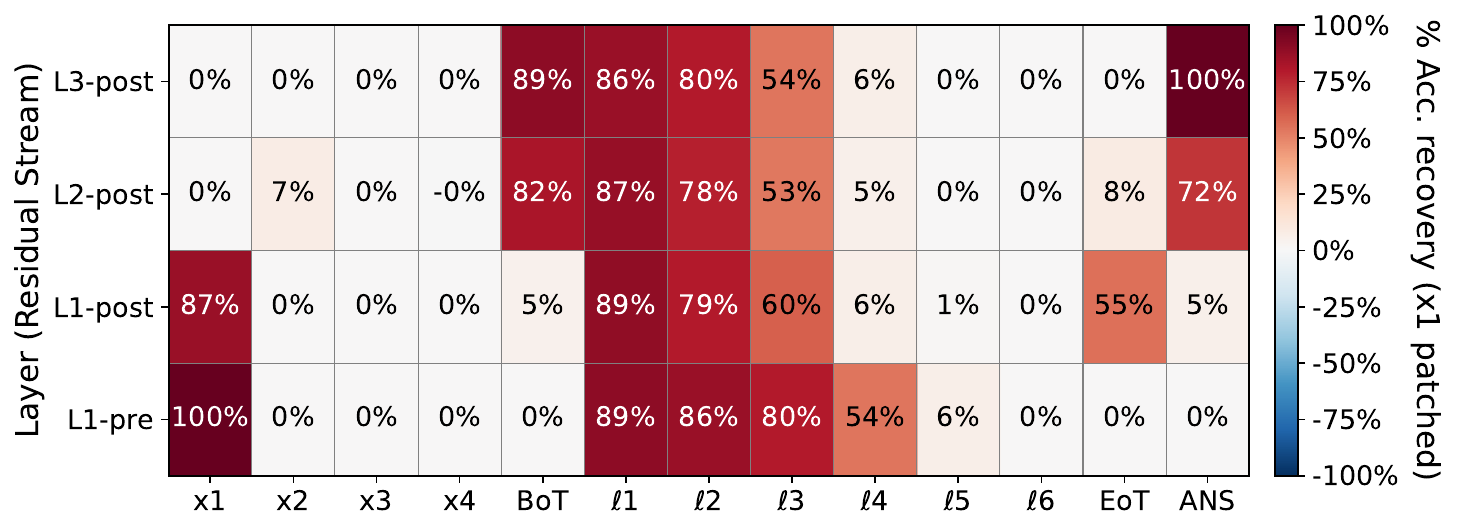}
    \caption{patching $x_1-$ corrupted run}
    \label{fig:3_hop_patch_x1}
  \end{subfigure}\hfill
  \begin{subfigure}[t]{0.49\textwidth}
    \centering
    \includegraphics[width=\linewidth]{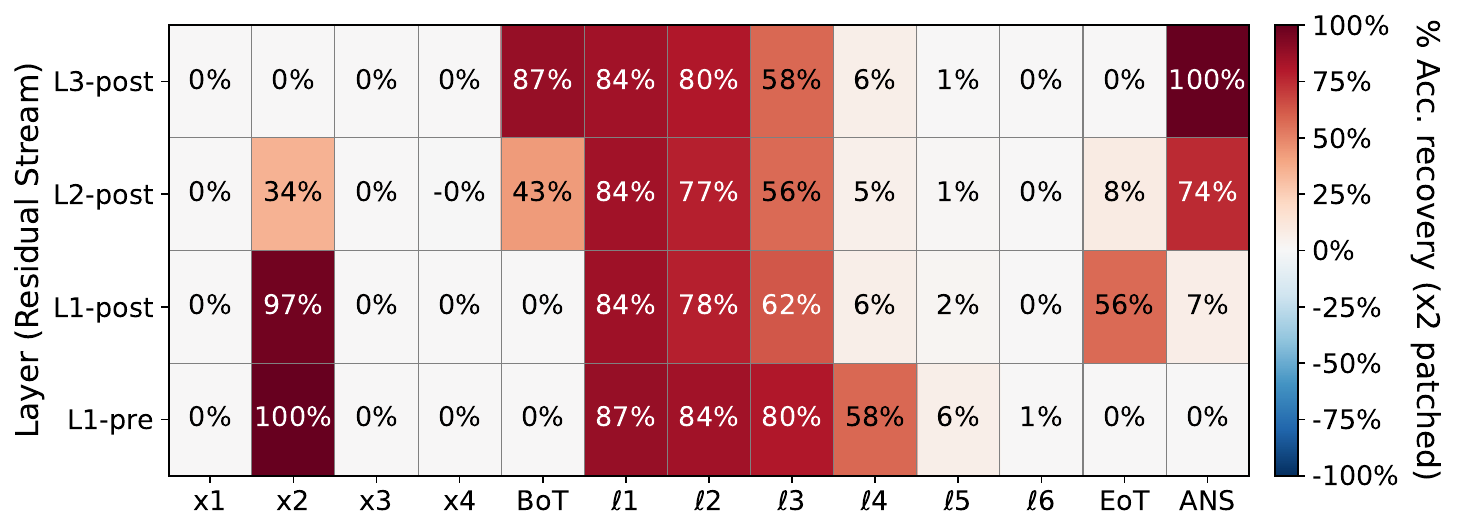}
    \caption{patching $x_2-$ corrupted run}
    \label{fig:3_hop_patch_x2}
  \end{subfigure}

  \vspace{0.5em} 

  \begin{subfigure}[t]{0.49\textwidth}
    \centering
    \includegraphics[width=\linewidth]{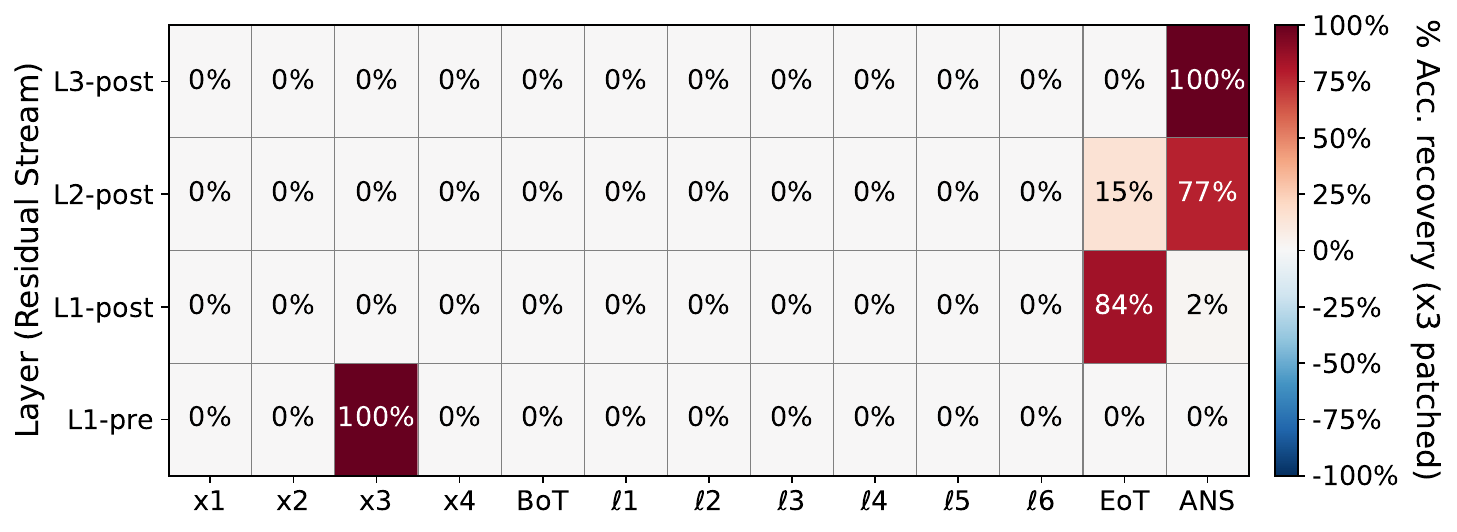}
    \caption{patching $x_3-$ corrupted run}
    \label{fig:3_hop_patch_x3}
  \end{subfigure}\hfill
  \begin{subfigure}[t]{0.49\textwidth}
    \centering
    \includegraphics[width=\linewidth]{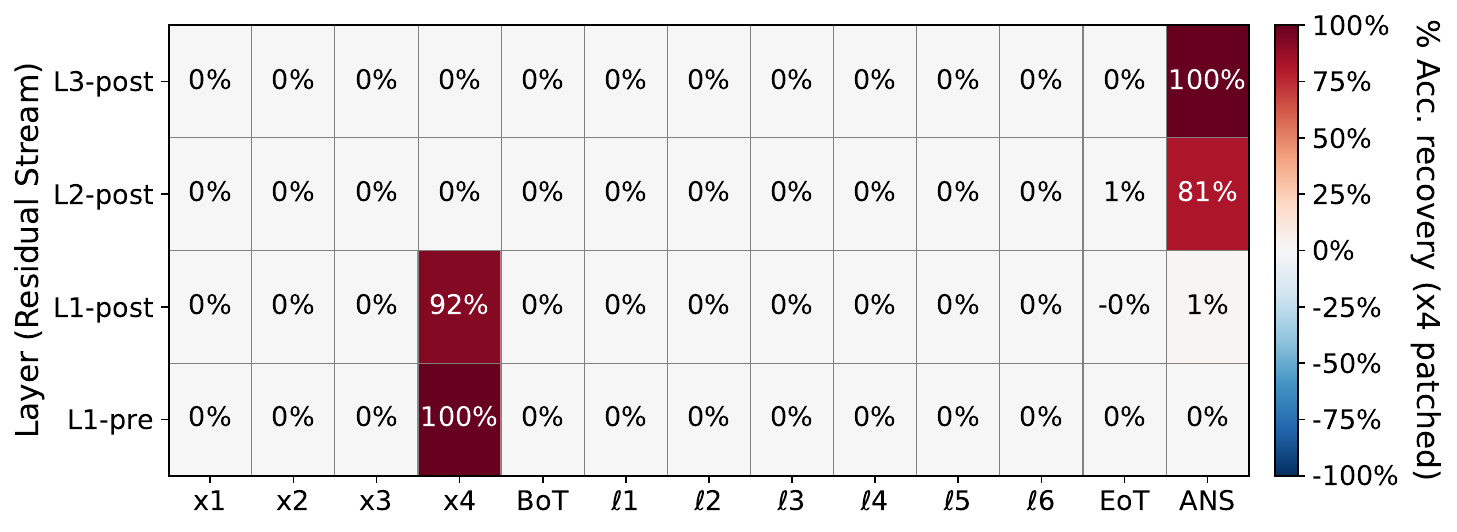}
    \caption{patching $x_4-$ corrupted run}
    \label{fig:3_hop_patch_x4}
  \end{subfigure}

  \caption{\textbf{Activation patching for input-token corruptions in the three-hop polynomial task.}
When $x_1$ or $x_2$ is corrupted, patching clean activations into the latent-thought positions ($\ell_1,\ell_2,\ell_3$) yields substantial accuracy recovery, indicating that the latent channel carries the intermediate information needed downstream (notably $s_2$, which is required to compute $s_3$).
When $x_3$ is corrupted, recovery concentrates at \texttt{[EoT]}, suggesting that $s_3$ is computed or consolidated near the end of the latent segment.
Finally, when $x_4$ is corrupted, recovery localizes at \texttt{[Ans]}, consistent with a direct (copy-like) route that delivers $x_4$ to the answer readout.
}
  \label{fig:patching_3hop_2x2}

\end{figure*}

\begin{figure}[h]
  \centering
  \includegraphics[width=0.75\linewidth]{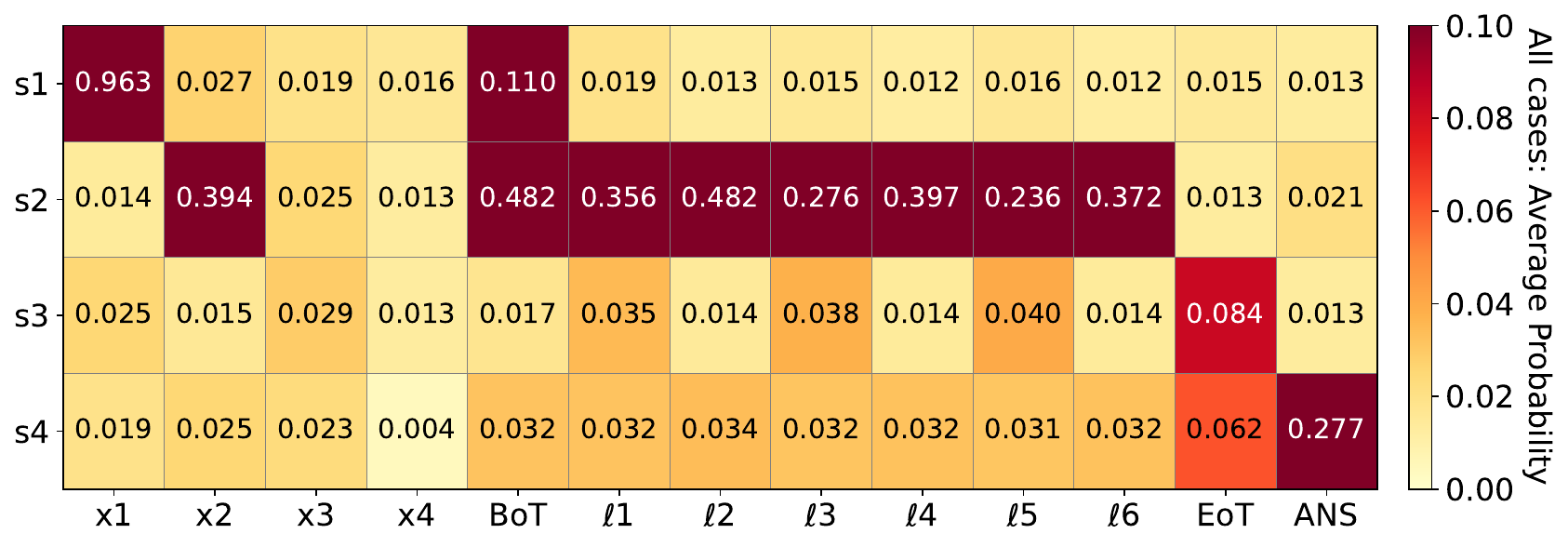}
  \caption{\textbf{Logit lens on intermediate states $s_1,s_2,s_3,s_4$ in the three-hop polynomial task.}
For modulus $m=50$, the logit lens shows that the bridge state $s_2$ becomes decodable early in the latent trajectory, indicating that it is formed and maintained in the latent channel. In contrast, $s_3$ becomes most decodable near the \texttt{[EoT]} boundary, suggesting that the next intermediate is computed or consolidated at the end of the latent segment. Each cell show average decoding probability across all layers and all test inputs.
  }
  \label{fig:logic_lens_3hop_L3H2}
\end{figure}

\section{Attention Knockout Analysis}
\label{app:attention_knockout} 

In \cref{sec:exp_q2}, to answer how the intermediate state s2 is routed to \texttt{[Ans]} token we conducted \textbf{targeted attention knockouts} to isolate the routing pathways. During the decoder phase, the \texttt{[Ans]} token attends over the full KV cache, which contains entries from the encoder tokens (\texttt{[POL]}, $x_1, \ldots, x_{N}$, \texttt{[BoT]}), latent tokens ($\ell_1, \ldots, \ell_K$), and the \texttt{[EoT]} token. We selectively block \texttt{[Ans]} from attending to specific subsets of these KV positions by setting their pre-softmax attention scores to $-\infty$ at every decoder layer and head. This forces the softmax to assign exactly zero weight to those positions. We then measure accuracy on correctly predicted test examples (baseline = 100\%).

We found that jointly knocking out \texttt{[Ans]} attention to \texttt{[BoT]}, all latent steps $\ell_1$–$\ell_6$, and \texttt{[EoT]} reduces accuracy to 3.8\%, indicating that the information required for the final prediction is routed to \texttt{[Ans]} through this set of positions. To further isolate the routing pathway, we performed leave-one-in ablations: we blocked attention to all of \texttt{[BoT]}, $\ell_1$–$\ell_6$, and \texttt{[EoT]}, while leaving exactly one position available at a time. Keeping \texttt{[BoT]} alone supports only 4.9\% accuracy, showing that it is not a primary routing source. In contrast, recovery comes almost entirely from the early-to-middle latent steps: $\ell_1$ recovers 87\% accuracy, $\ell_2$ 97\%, $\ell_3$ 94\%, $\ell_4$ 80\%, $\ell_5$ 47\%, and $\ell_6$ 28\%, whereas [EoT] alone recovers only 4\%. This experiment shows that the bridge-state signal $s_2$ reaches \texttt{[Ans]} primarily through the early-to-middle latent positions.

\section{State-Swap Intervention}
\label{app:state-swap-intervention} 

To further test whether the decoded bridge state $s_2$ in \cref{sec:exp_q2} is causally used rather than merely correlated with the output, we perform a state-swap intervention on sequence-length-3 examples. We construct suffix-matched prompt pairs $(A,B)$ that share the final input token $x_3$ but differ in the bridge state $s_2$ and therefore in the ground-truth answer $s_3$. At each latent position, we replace the residual stream of prompt $A$ with the cached activation from a clean forward pass of prompt $B$, and measure whether the model's prediction changes away from $A$'s answer and whether it flips specifically to $B$'s answer.

The intervention has a strong causal effect at the beginning of the latent trajectory. Swapping at \texttt{[BoT]}, $l_1$, and $l_2$ changes the prediction away from $A$ in 98.7\%, 97.5\%, and 91.0\% of cases, respectively, and flips it to $B$ in 90.1\%, 80.1\%, and 71.5\% of cases. The effect then decays sharply: swaps at $l_3$ still affect about half of the examples, but by $l_5$ the flip-to-$B$ rate drops to 0.5\%, and by $l_6$ and \texttt{[EoT]} it vanishes. This shows that the bridge-state representation is an active causal driver of the computation in \texttt{[BoT]} and the early latent steps, while later latent positions no longer serve as the main route by which $s_2$ determines the final answer.

\section{Partial Latent Rollouts Concentrate on Late Intermediates for Longer Hops.}
\label{app:partial_roll_out_last_two}

As shown in \cref{fig:n_hops_logic_lens}, we previously observed that, on $n$-hop tasks, the final intermediate state $s_n$ often becomes decodable within CODI's latent-thought positions. For longer horizons (typically $n \ge 4$ under composite moduli), we also observe a stronger \emph{two-step} variant of this behavior: both $s_{n-1}$ and $s_n$ become decodable across the latent trajectory, with $s_{n-1}$ emerging earlier and $s_n$ emerging later.

\Cref{fig:5_hop_partial_2_logic_lens} illustrates this pattern on a $5$-hop task: $s_4$ is decodable in the early latent steps, while $s_5$ becomes decodable only toward the end of the latent trajectory. \Cref{fig:7_hop_partial_2_logic_lens} shows an analogous effect for a $7$-hop task, where $s_6$ appears earlier and $s_7$ appears later in the latent steps. This ordering is consistent with the analysis in \cref{sec:prime-vs-composite}: under composite moduli, the task becomes increasingly biased toward the last few updates, making late intermediates disproportionately predictive of the final answer $s_{n+1}$. CODI's latent channel appears to exploit this structure by allocating its limited latent compute to tracking the last one or two intermediates rather than maintaining a full step-by-step rollout.

\begin{figure*}[ht]
  \centering
  \begin{subfigure}[t]{0.49\textwidth}
    \centering
    \includegraphics[width=\linewidth]{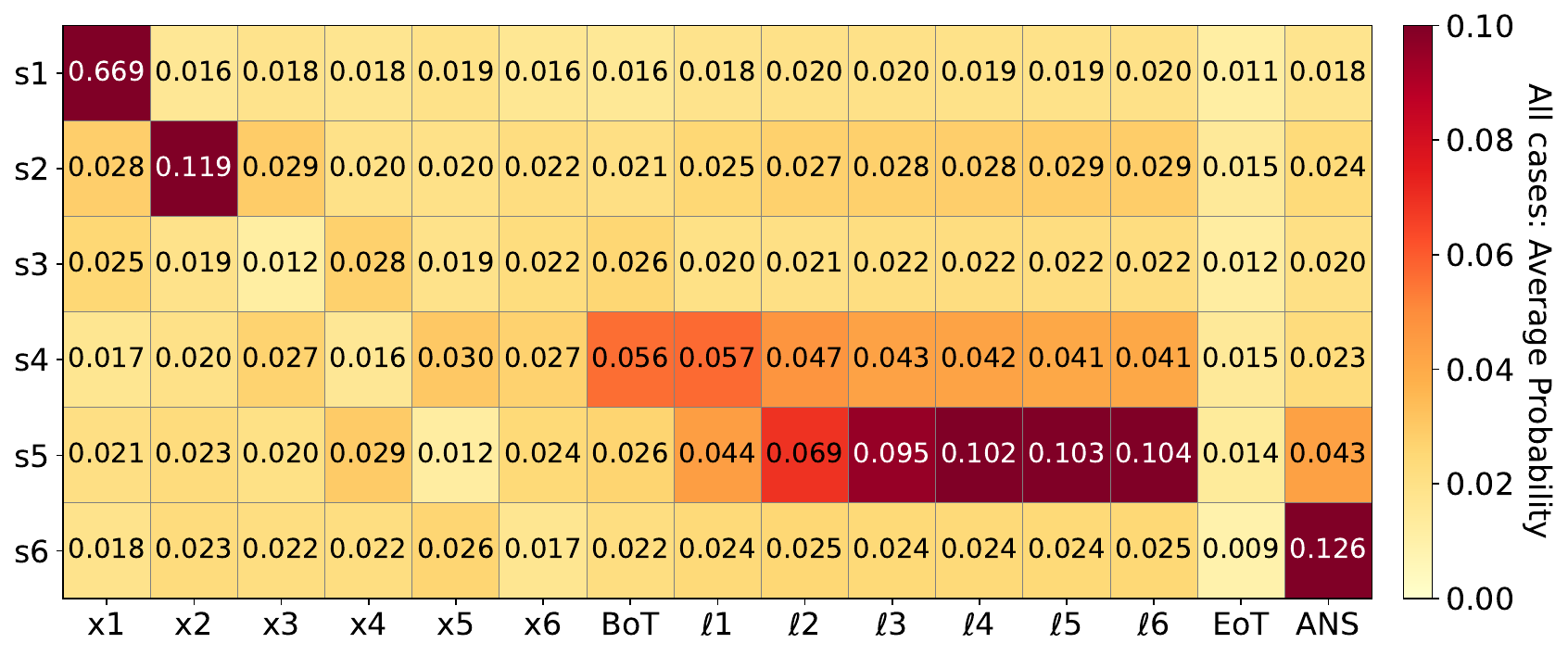}
    \caption{$5$-hop task}
    \label{fig:5_hop_partial_2_logic_lens}
  \end{subfigure}\hfill
  \begin{subfigure}[t]{0.49\textwidth}
    \centering
    \includegraphics[width=\linewidth]{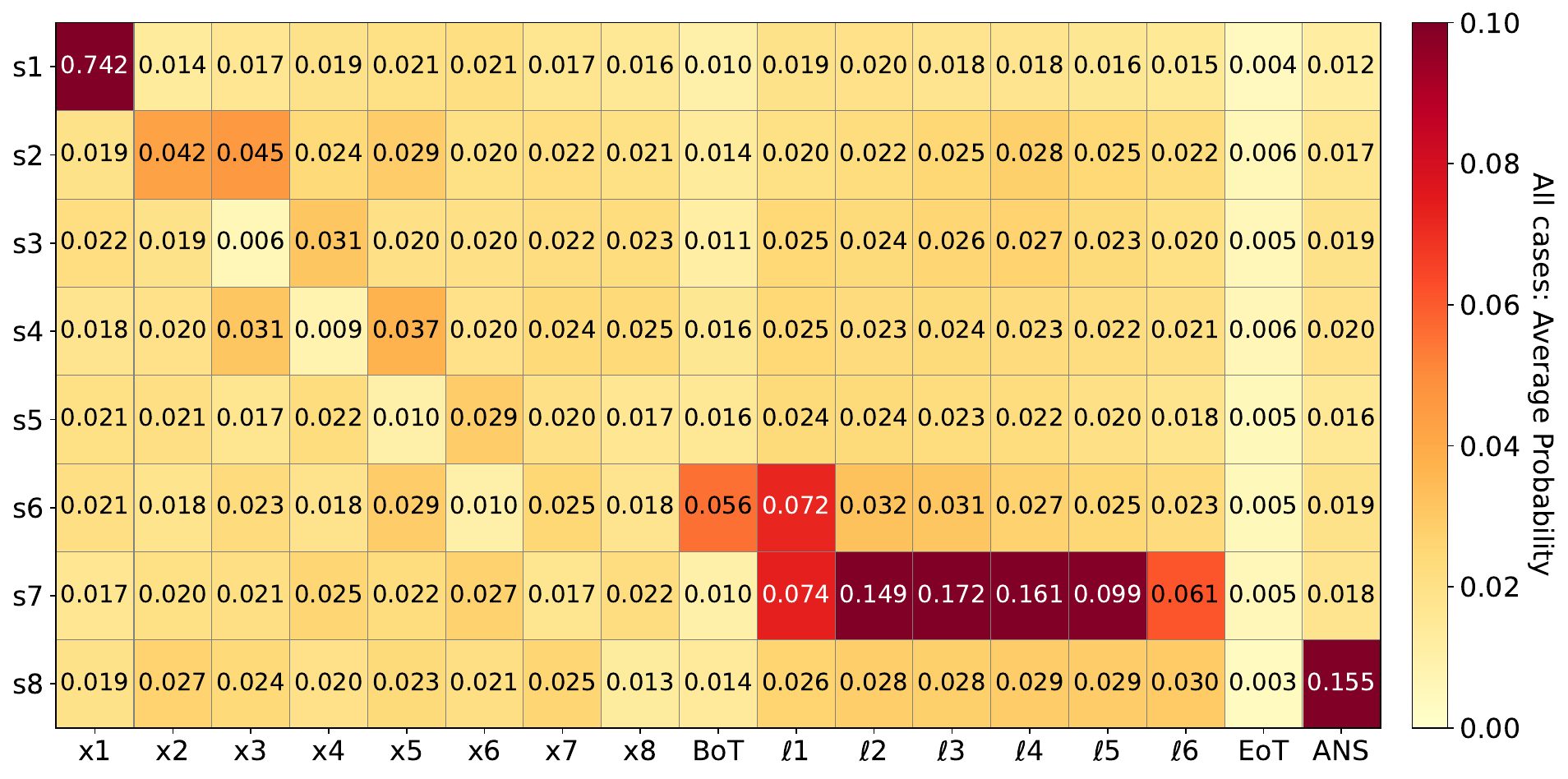}
    \caption{$7$-hop task}
    \label{fig:7_hop_partial_2_logic_lens}
  \end{subfigure}

  \caption{\textbf{Logit-lens evidence for a two-step partial latent rollout.} For longer-hop tasks (composite moduli), CODI often makes the last two intermediate states decodable within the latent-thought trajectory: $s_{n-1}$ appears earlier in the latent steps, followed by $s_n$ later.}
  \label{fig:partial_two_step_logit_lens}
\end{figure*}

\section{Proofs for Theoretical Results in \cref{sec:prime-vs-composite}}

\subsection{Proof of Lemma 5.1:} 
\label{app:lemma_5_1_proof}

\textbf{Lemma 5.1 [Bijection criterion]}: 
For $x\in R_m$, the map $f_x : R_m \to R_m$ is bijective iff $x$ is a unit in $R_m$, i.e.\ $\gcd(x,m)=1$.

\begin{proof}
The additive shift by $b$ is a bijection, so $f_x$ is bijective iff $s\mapsto sx$ is bijective.
Multiplication by $x$ is bijective over $R_m$ iff there exists $x^{-1}$ with $xx^{-1}\equiv 1\ (\mathrm{mod}\ m)$, which holds iff $\gcd(x,m)=1$.
\end{proof}

\subsection{Proof of Lemma 5.2:}
\label{app:lemma_5_2_proof}

\textbf{Lemma 5.2 [Exact contraction factor]}: Let $m\ge 1$ and fix $x,b\in R_m:=\mathbb{Z}/m\mathbb{Z}$. Let $d:=\gcd(x,m)$.
Define $\mu_x:R_m\to R_m$ by $\mu_x(s)=sx\ (\mathrm{mod}\ m)$ and $f_x(s)=sx+b\ (\mathrm{mod}\ m)$.
Then $|\mathrm{im}(\mu_x)|=m/d$, and every fiber of $\mu_x$ has size $d$. Equivalently, $f_x$ is a $d$-to-$1$ map.

\begin{proof}
We view $R_m$ as an additive group. The map $\mu_x$ is a group homomorphism since
\[
\mu_x(s+t)=(s+t)x \equiv sx+tx \equiv \mu_x(s)+\mu_x(t)\pmod m.
\]
Let $K:=\ker(\mu_x)=\{s\in R_m: sx\equiv 0\ (\mathrm{mod}\ m)\}$.
Write $x=dx'$ and $m=dm'$ with $\gcd(x',m')=1$. Then
\[
sx\equiv 0 \pmod m
\iff dm' \mid s(dx')
\iff m' \mid s x'
\iff m' \mid s,
\]
where the last equivalence uses $\gcd(x',m')=1$ (so multiplication by $x'$ is invertible modulo $m'$).
Thus
\[
K=\{0,\, m',\, 2m',\,\dots,\,(d-1)m'\}\subset R_m,
\]
so $|K|=d$.

For a group homomorphism $\varphi:G\to H$, each fiber $\varphi^{-1}(y)$ (for $y\in \mathrm{im}(\varphi)$) is a coset of $\ker(\varphi)$:
if $\varphi(s_0)=y$, then $\varphi^{-1}(y)=s_0+\ker(\varphi)$. Hence every fiber has size $|\ker(\varphi)|$.
Applying this to $\mu_x$, every fiber has size $|K|=d$.

Since $R_m$ is finite, the fibers partition $R_m$ into $|\mathrm{im}(\mu_x)|$ disjoint sets of equal size $d$, so
\[
|\mathrm{im}(\mu_x)|=\frac{|R_m|}{d}=\frac{m}{d}.
\]

Finally, $f_x=\tau_b\circ \mu_x$ where $\tau_b(y)=y+b$ is a bijection on $R_m$.
Therefore $f_x$ has the same fiber sizes as $\mu_x$, i.e., $f_x$ is also $d$-to-$1$.
\end{proof}

\section{Ablation Study}
\label{sec: ablation_study}

For all subsequent ablations, we use the $n$-hop setting with $n=31$, a 3-layer transformer with 2 attention heads per layer (unless stated otherwise), and modulus $m=50$.

\paragraph{Varying the number of latent steps.}

We ablate the number of latent steps over $p\in\{1,2,3,6,9,12,20\}$.
Across all values of $p$ tested, our main mechanistic observations remain mostly stable: on long $n$-hop tasks, the model consistently forms and propagates the \emph{late} intermediate state---most reliably the final state $s_n$, and occasionally the last two states $(s_{n-1}, s_n)$.
This aligns with the analysis in \cref{sec:prime-vs-composite}: under composite moduli, contraction events make the label depend primarily on the terminal suffix, favoring late-bottleneck solutions.
Accordingly, within this range, increasing $p$ does not reliably induce a longer internal rollout; the emergence of latent-step structure appears to be driven more by the task distribution (long-horizon updates under composite $m$) than by the available latent-step budget.

\paragraph{Varying model depth and width.}

We further test architectural robustness by sweeping depth and width, varying the number of layers in $\{2,3,4,5,6,7\}$ and attention heads in $\{2,4,8\}$ (up to a 7-layer, 8-head student).
Across these configurations, our mechanistic picture is mostly unchanged: on long $n$-hop tasks the model primarily represents and routes late intermediates---most reliably $s_n$, and sometimes $(s_{n-1},s_n)$.
This is consistent with \cref{sec:prime-vs-composite}: under composite moduli, frequent contractions make the label depend mainly on the terminal suffix, which naturally favors late-bottleneck solutions.
Consequently, scaling the student up or down within this range does not reliably elicit a deeper internal rollout; the observed latent-step behavior appears driven more by the task distribution than by the specific model configuration.
 
\paragraph{Ablation on the loss function} 

Recall that CODI is trained with three losses: (i) a \emph{teacher} objective that predicts an explicit CoT / state trace,
(ii) a \emph{student} objective that predicts the final answer after the latent-thought steps, and
(iii) a \emph{feature-space distillation} term that aligns teacher and student representations near the answer boundary.

\emph{Ablating distillation only.}
We first remove the distillation term while keeping the teacher and student losses.
This ablation is still meaningful because the teacher and student share the same backbone and hyperparameters, so the teacher objective continues to shape the shared representation space that the student can leverage.
Empirically, our main mechanistic signatures persist: the model still forms the late intermediate state $s_n$ (and occasionally both $s_{n-1}$ and $s_n$) in the final latent steps, and we still observe a specialized attention head that routes information from the final input token $x_{n+1}$ to the \texttt{[Ans]} boundary.
This suggests that the distillation term is not the primary driver of the partial (late-only) latent rollout in this setting.

\emph{Ablating both distillation and the teacher loss.}
Next, we remove both the distillation term and the teacher objective, leaving only the student loss.
In this case, we no longer observe the late-step reasoning trace: the final latent steps do not reliably encode $s_n$ (or $(s_{n-1},s_n)$).
This indicates that the teacher loss is important for inducing the late mechanism---likely because it pressures the shared backbone to represent stepwise state information, which the student objective can then compress into a short, late-bottleneck computation that still supports accurate answers.

\section{Comparing CODI-style latent CoT with CoT and Non-CoT Baselines} 
\label{app:full-discussion}

\begin{table*}[ht]
\centering
\caption{\textbf{Accuracy (\%) of CoT-trained models, a non-CoT standard transformer, and CODI across moduli $m$}. CODI and the non-CoT baseline degrade sharply when $m$ is prime. All results use a 3-layer, 2-head transformer with total sequence length 32, and report accuracy aggregated over all sequence lengths.}
\setlength{\tabcolsep}{5pt}
\begin{tabular}{@{}lrrrrrrrrrr@{}}
\toprule
Mod $m$ & 41 & 42 & 43 & 44 & 45 & 46 & 47 & 48 & 49 & 50 \\ \midrule
Full-CoT & 100.00 & 100.00 & 28.66 & 97.30 & 100.00 & 93.60 & 90.85 & 100.00 & 100.00 & 100.00 \\
Non-CoT  & 5.44   & 35.73  & 5.31  & 18.44 & 39.94  & 5.93  & 2.08  & 45.62  & 8.50   & 25.81  \\
CODI     & 8.67   & 66.40  & 8.59  & 34.19 & 67.05  & 16.74 & 8.31  & 91.07  & 19.97  & 43.95  \\ \bottomrule
\end{tabular}

\label{tab:accuracy_across_m}
\end{table*}

\textbf{Comparison with Non-CoT.}
A standard next-token transformer trained \emph{without} an explicit or latent CoT channel can nevertheless learn a full internal rollout on this task.
In the three-layer, two-head model, the intermediate state $s_i$ becomes decodable at the next input position $x_{i+1}$ (\cref{fig:non_cot_layer3_head2}), consistent with an implicit iterative update carried in the residual stream despite supervision only on the final answer.
This step-by-step trace is \emph{brittle}, however: it can vanish under modest architectural or optimization changes (\cref{fig:non_cot_layer5_head2}), where intermediate-state decodability disappears in a five-layer, two-head model. By contrast, CODI under composite moduli exhibits a \emph{late-bottleneck} behavior on long $n$-hop tasks: typically only $s_n$---and occasionally $(s_{n-1}, s_n)$---becomes decodable in the final latent steps.
We attribute this to \emph{teacher-guided compression}: explicit trace supervision pushes the student to compress multi-step computation into a fixed latent-thought budget, and when the task permits short effective memory (\cref{sec:prime-vs-composite}), this pressure favors late-only rollouts.

\textbf{Compare with CoT.} In contrast to both the non-CoT baseline and CODI-style latent CoT, an explicitly CoT-trained transformer can reliably recover the full intermediate trajectory $s_1,\ldots,s_n$ across model scales and optimization settings \citep{cabannes2024iteration}. Consistent with this, \Cref{tab:accuracy_across_m} shows that CoT training substantially outperforms both CODI and a standard non-CoT transformer when the modulus $m$ is prime. These results expose a key limitation of CODI-style latent CoT on strictly sequential tasks: with a fixed latent-compute budget, CODI can struggle when the underlying computation is effectively incompressible. This mirrors prior observations that implicit reasoning in standard transformers is constrained by model depth; latent-CoT models inherit a similar constraint, now jointly bounded by architectural depth and the number of latent-thought steps. Conversely, CODI recovers much of the performance gap in more compressible regimes, such as composite moduli. 

\textbf{Composite Rings as a Mechanistic Lens on Latent Reasoning Compression.} Studying composite moduli is valuable because their built-in $d$-to-$1$ contractions offer a clean analogue of a common feature of language: distinct underlying interpretations can map to similar surface forms. Such ambiguity can reduce the effective value of preserving fine-grained early context once later cues dominate. Composite moduli make this explicit: when an update uses a non-unit multiplier, it collapses $d$ distinct states into one, inducing genuine information loss and shifting the final answer toward dependence on a short terminal suffix. This perspective clarifies when latent CoT is advantageous: in settings with many-to-one mapping from histories (or hidden states) to the observed target/label, additional internal steps may be better used to \emph{compress} history than to preserve it exactly, and teacher-guided latent CoT can act as a structured bottleneck that distills multi-step traces into a small number of latent updates that retain only the intermediates most predictive of the answer (in our polynomial setting, typically the final one or two). We view this as a hypothesis for how latent CoT may help in more naturalistic settings.

\begin{figure}[ht]
  \centering
  \begin{subfigure}[t]{0.49\textwidth}
    \centering
    \includegraphics[width=\linewidth]{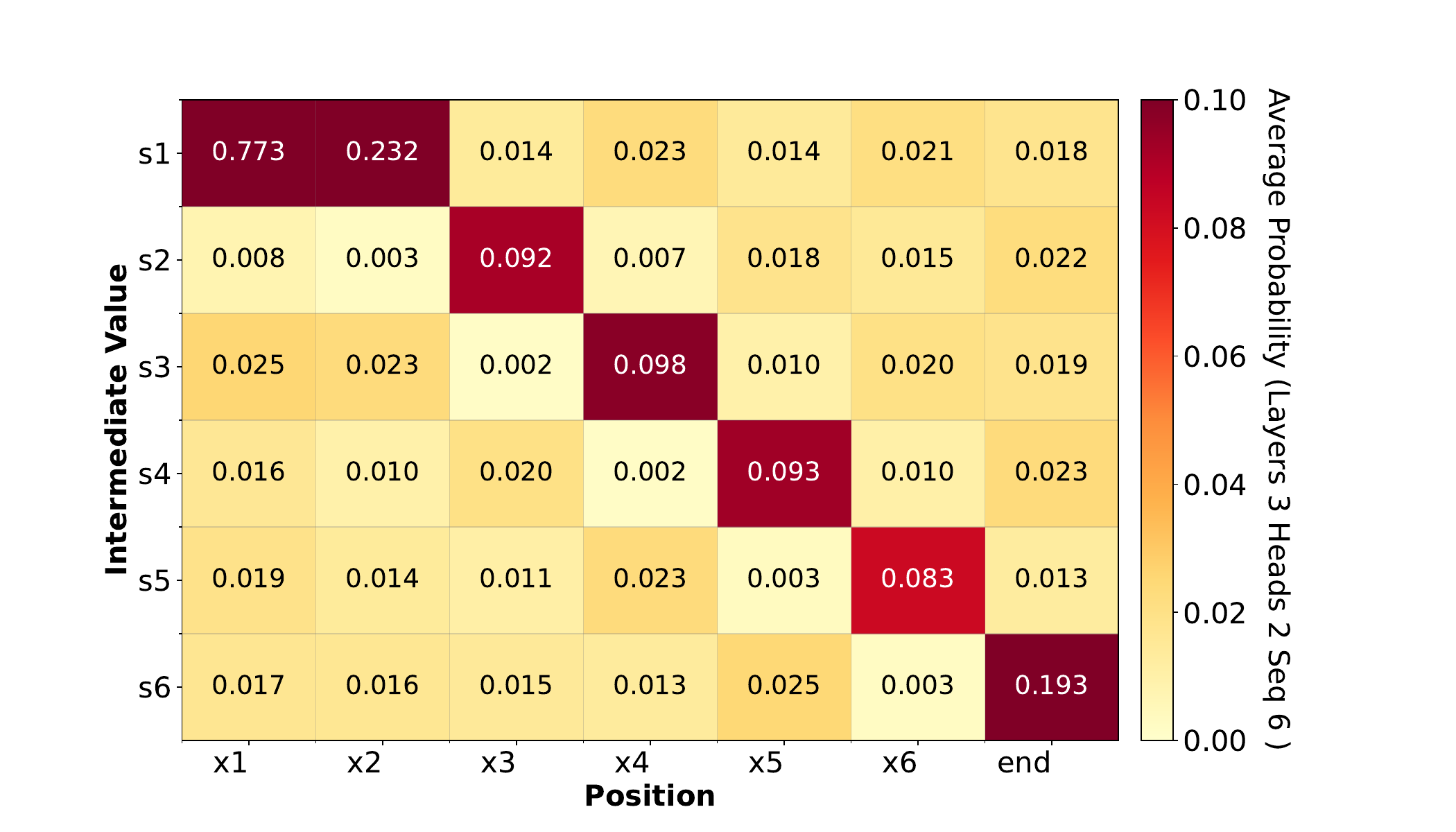}
    \caption{Layer-3, head-2.}
    \label{fig:non_cot_layer3_head2}
  \end{subfigure}\hfill
  \begin{subfigure}[t]{0.49\textwidth}
    \centering
    \includegraphics[width=\linewidth]{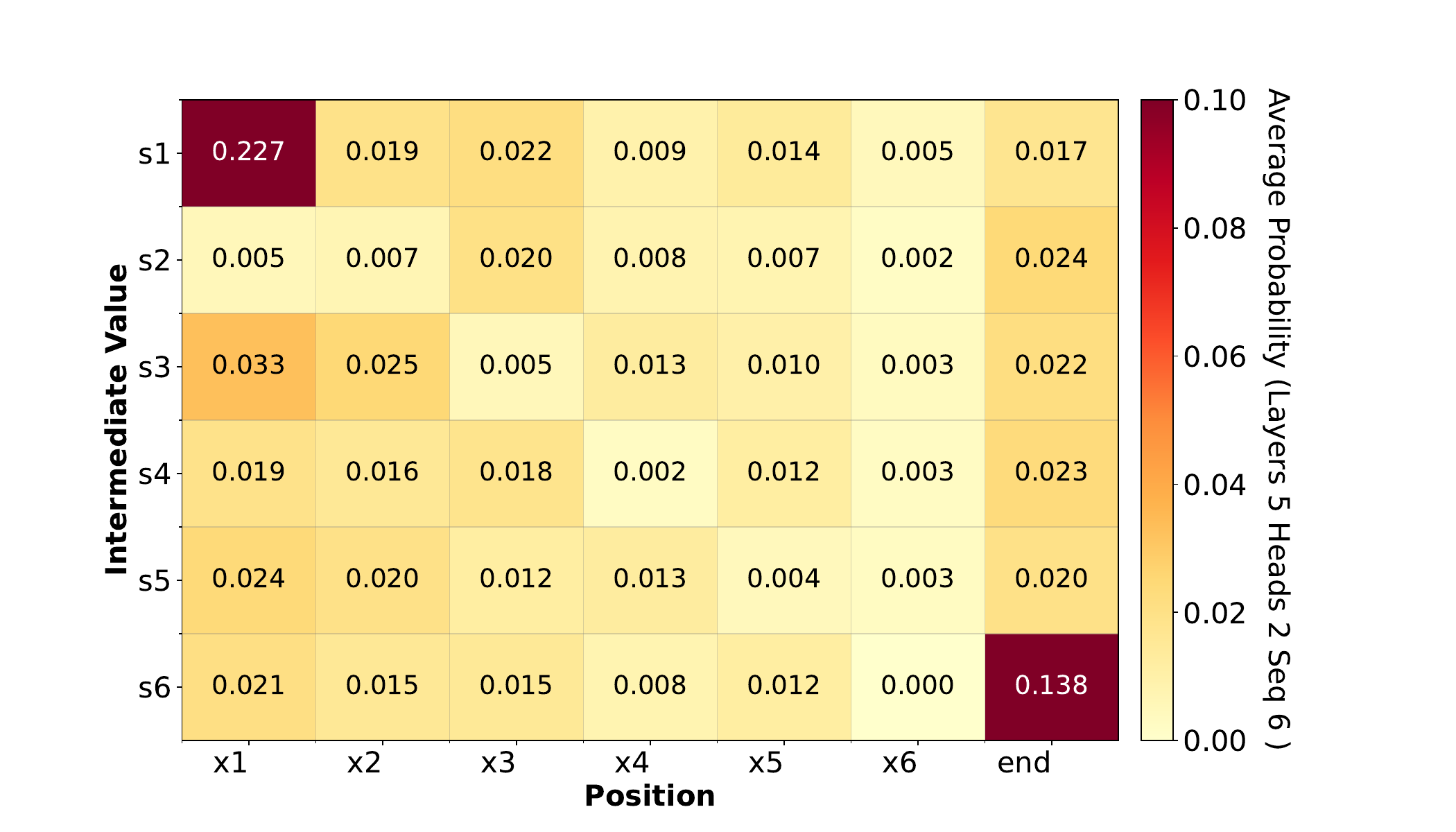}
    \caption{Layer-5, head-2.}
    \label{fig:non_cot_layer5_head2}
  \end{subfigure}

  \caption{\textbf{Logit-lens visualizations for the 5-hop task on all inputs.}
We compare a layer-3/head-2 model to a layer-5/head-2 model. The layer-3/head-2 model exhibits a clear rollout of intermediate states, whereas the layer-5/head-2 model shows little to no intermediate-state trace. This highlights the brittleness of step-by-step reasoning in standard (Non-CoT) Transformers and contrasts with latent-CoT models, which often converge to a late-bottleneck strategy.}

  \label{fig:non_cot_layer35_heads22}
\end{figure}

\section{Finite-State Function Composition: Definition, Theory, Experiments, and Generality}
\label{app:fsc}

This appendix gives the full definition, theoretical analysis, and experimental results for FSC, the finite-state function-composition task used in Section~\ref{sec:fsc}. The purpose of FSC is to test whether the compressibility effects observed in polynomial iteration are specific to modular arithmetic or instead reflect a more general property of sequential computation.

\subsection{Dataset Definition}
\label{app:fsc-definition}

\paragraph{State space and function library.}
Fix a finite state space
\[
\mathcal{S}=\{0,1,\ldots,N-1\}, \qquad |\mathcal{S}|=N.
\]
FSC defines a fixed library of $K$ transition functions
\[
\mathcal{F}=\{f_1,\ldots,f_K\}, \qquad f_i:\mathcal{S}\to\mathcal{S}.
\]
The function library is constructed once, typically with a fixed random seed, and reused across all train and test examples. Thus the model is not asked to infer a fresh random function for each example; it must learn to use a fixed operation library.

\paragraph{Bijective and compressive functions.}
Let $\alpha\in[0,1]$ denote the nominal fraction of bijective functions. Since $K$ is finite, we instantiate
\[
K_{\mathrm{bij}}=\operatorname{round}(\alpha K),
\qquad
K_{\mathrm{comp}}=K-K_{\mathrm{bij}},
\]
or another explicitly specified rounding rule. We define the realized bijective fraction as
\[
\beta=\frac{K_{\mathrm{bij}}}{K}.
\]
This realized fraction $\beta$ may differ slightly from the nominal value $\alpha$.

For each of the $K_{\mathrm{bij}}$ bijective functions, we sample a random permutation of $\mathcal{S}$. These functions have image size $N$ and preserve all state distinctions.

For each of the $K_{\mathrm{comp}}$ compressive functions, we sample an image size
\[
m_i \in \{m_{\min},\ldots,m_{\max}\},
\qquad
1\leq m_{\min}\leq m_i\leq m_{\max}\leq N-1,
\]
and construct $f_i$ so that
\[
|\operatorname{Im}(f_i)|=m_i.
\]
Smaller $m_i$ corresponds to stronger compression. We define the transition rank
\[
r_i := |\operatorname{Im}(f_i)|.
\]
Thus $r_i=N$ for bijections and $r_i=m_i<N$ for compressive functions. The quantity $N/r_i$ is the effective contraction factor of the transition.

\paragraph{Implementation note.}
For each compressive function, we first sample a target image size
$m_i \in \{m_{\min},\ldots,m_{\max}\}$ and then sample $m_i$ distinct
output states from $\mathcal{S}$. The construction explicitly ensures
that every selected output state appears at least once in the function
table: we initialize the outputs with the selected image values, fill the
remaining $N-m_i$ domain points by sampling from these values, and then
shuffle the resulting list. Therefore each compressive function satisfies
\[
|\operatorname{Im}(f_i)| = m_i
\]
exactly. 

\paragraph{Example generation.}
An example consists of an initial state and a sequence of operation indices:
\[
s_0\in\mathcal{S},
\qquad
\sigma_1,\ldots,\sigma_T\in\{1,\ldots,K\}.
\]
The hidden trajectory is generated by repeated function composition:
\[
s_t=f_{\sigma_t}(s_{t-1}),\qquad t=1,\ldots,T.
\]
The model receives
\[
(s_0,\sigma_1,\ldots,\sigma_T)
\]
and predicts the trajectory
\[
(s_1,\ldots,s_T).
\]

\paragraph{CODI-style serialization.}
For explicit-CoT teacher training, we use the serialization
\[
\texttt{[FSC]} \; s_0 \; f_{\sigma_1}\;\cdots\; f_{\sigma_T}\; \texttt{[EoI]}\;
s_0\;s_1\;\cdots\;s_{T-1}\;\texttt{[Ans]}\;s_T.
\]
The student replaces the explicit intermediate chain with latent thought tokens:
\[
\texttt{[FSC]} \; s_0 \; f_{\sigma_1}\;\cdots\; f_{\sigma_T}\;
\texttt{[BoT]}\;\ell_1\;\cdots\;\ell_p\;\texttt{[EoT]}\;\texttt{[Ans]}\;s_T.
\]
This mirrors the polynomial-task setup: the teacher is trained to emit the full state trace, while the latent student must compress the computation into continuous thought tokens before predicting the final state.

\subsection{Theoretical Analysis}
\label{app:fsc-theory}

The polynomial analysis in Section~\ref{sec:prime-vs-composite} uses algebraic structure: over prime moduli, all nonzero multipliers are units and hence updates are bijections; over composite moduli, non-unit multipliers induce exact many-to-one contractions. FSC replaces this algebraic distinction with the transition-rank distinction $r_i=|\operatorname{Im}(f_i)|$.




\begin{theorem}[Fully bijective FSC has non-contracting long-horizon dynamics]
\label{thm:fsc-bijective-long-range}
Suppose every function $f_i\in\mathcal F$ is bijective. Then for any operation
sequence $\sigma_{1:T}$, the composed transition
\[
F_{\sigma_{1:T}}
=
f_{\sigma_T}\circ\cdots\circ f_{\sigma_1}
\]
is bijective. Hence, conditioned on $\sigma_{1:T}$, the final state $s_T$
preserves all distinctions between initial states.

Moreover, fix any step $t$ and compare two operation choices $i$ and $j$ at
that step, holding the prefix and suffix fixed. Let $s_{t-1}$ be the prefix
state before step $t$, and let $B_t$ be the suffix map after step $t$. Then
the two final states satisfy
\[
B_t(f_i(s_{t-1}))=B_t(f_j(s_{t-1}))
\quad\Longleftrightarrow\quad
f_i(s_{t-1})=f_j(s_{t-1}).
\]
Thus any distinction introduced by the operation choice at time $t$ is
preserved by the later bijective suffix. In this sense, the fully bijective
FSC regime has no intrinsic mechanism for erasing earlier state or operation
distinctions.
\end{theorem}

\begin{proof}
The first claim follows because a composition of bijections is a bijection.
Therefore the map $s_0\mapsto s_T$ is injective for every fixed operation
sequence.

For the second claim, define
\[
B_t=f_{\sigma_T}\circ\cdots\circ f_{\sigma_{t+1}}.
\]
Since all functions are bijective, $B_t$ is bijective and hence injective.
The final state after choosing operation $i$ at step $t$ is
\[
y_i=B_t(f_i(s_{t-1})),
\]
while the final state after choosing operation $j$ is
\[
y_j=B_t(f_j(s_{t-1})).
\]
Because $B_t$ is injective,
\[
y_i=y_j
\quad\Longleftrightarrow\quad
f_i(s_{t-1})=f_j(s_{t-1}).
\]
Therefore, if the two candidate operations differ on the prefix state
$s_{t-1}$, the final states must differ as well. Later bijective updates
cannot erase the distinction.
\end{proof}

\paragraph{Connection to the prime-modulus regime.}
This theorem generalizes the no-contraction component of our prime-modulus analysis beyond modular arithmetic. In the polynomial task over a prime field, all nonzero multipliers are units, so every update $s\mapsto sx+b$ is a permutation and the recurrence never discards state information. The unrolled polynomial expression further shows that early inputs remain algebraically present in the final answer through multiplicative chains. FSC removes this algebraic structure, so there is no analogous closed-form expansion. Nevertheless, when $\alpha=1$, the same information-preservation principle holds at the level of finite-state dynamics: every update is bijective, every suffix map is injective, and distinctions introduced at any earlier step cannot be merged by later updates. Thus, the fully bijective FSC regime serves as a non-algebraic counterpart to the prime-modulus regime: both lack many-to-one contractions, preserve long-range state distinctions, and therefore expose the fixed-budget limitations of latent-CoT models on incompressible sequential computation.

\begin{lemma}[Compressive updates induce rank contraction]
\label{lem:fsc-rank-contraction}
Let $f_i:\mathcal{S}\to\mathcal{S}$ be a compressive FSC function with image size
\[
m_i=|\operatorname{Im}(f_i)|<N.
\]
Then $f_i$ maps all $N$ possible previous states into only $m_i$ possible next states. In particular, at least one output state has at least $\lceil N/m_i\rceil$ preimages. If $s$ is uniform over $\mathcal{S}$, then
\[
H(f_i(s))\leq \log m_i,
\]
so the update removes at least
\[
\log N-\log m_i = \log(N/m_i)
\]
bits of state capacity relative to a uniform $N$-state variable.
\end{lemma}

\begin{proof}
The image of $f_i$ contains exactly $m_i$ states, so $f_i(s)$ is supported on at most $m_i$ values. Therefore its entropy is at most $\log m_i$. Since a uniform variable over $\mathcal{S}$ has entropy $\log N$, the update reduces the maximum state capacity by at least $\log(N/m_i)$. The preimage statement follows from the pigeonhole principle.
\end{proof}

\paragraph{Interpretation.}
The quantity $N/m_i$ is an effective contraction factor. This directly generalizes the $d=\gcd(x,m)$ contraction factor in the modular polynomial task, except that arbitrary FSC maps need not have equal-sized fibers.

\begin{lemma}[Lost distinctions cannot be restored by later deterministic maps]
\label{lem:fsc-no-restore}
Let $r_i=|\operatorname{Im}(f_i)|$. For any subsequence of updates
\[
F_{a:b}:=f_{\sigma_b}\circ \cdots \circ f_{\sigma_a},
\]
we have
\[
|\operatorname{Im}(F_{a:b})|
\leq
\min_{a\leq t\leq b} r_{\sigma_t}.
\]
In particular, once a compressive function with image size $m_i$ occurs, the number of distinguishable prefix states after that point is at most $m_i$, and later deterministic updates cannot increase it.
\end{lemma}

\begin{proof}
For any deterministic function $g$ and set $A$, we have $|g(A)|\leq |A|$. After applying any transition $f_{\sigma_t}$, the reachable set has size at most $r_{\sigma_t}$. All later deterministic maps can only map this reachable set to another set of size no larger than its current size. Therefore the image size of the whole subsequence is bounded by the minimum transition rank encountered along the subsequence.
\end{proof}

\paragraph{Interpretation.}
This lemma is the formal basis for the late-bottleneck hypothesis. After a compressive update, the final answer can depend on the earlier prefix only through a compressed summary with at most $m_i$ states, plus the later suffix of updates.

\begin{lemma}[Terminal bijective suffix]
\label{lem:fsc-terminal-suffix}
Assume operation indices $\sigma_t$ are sampled independently and uniformly from the fixed library. Let
\[
\beta=\frac{K_{\mathrm{bij}}}{K}
\]
be the realized fraction of bijective functions. Define the last compressive time
\[
\tau := \max\{t\leq T: f_{\sigma_t}\text{ is compressive}\},
\]
with $\tau=0$ if no compressive update occurs, and define the terminal all-bijective suffix length
\[
L := T-\tau.
\]
Then
\[
\Pr(L\geq k)=\beta^k,
\]
and, for $\beta<1$,
\[
\mathbb{E}[L]
=
\sum_{k=1}^{T}\beta^k
=
\frac{\beta(1-\beta^T)}{1-\beta}.
\]
\end{lemma}

\begin{proof}
The event $L\geq k$ occurs exactly when the final $k$ sampled operations are all bijective. Since operations are sampled independently and the probability of drawing a bijective operation is $\beta$, this event has probability $\beta^k$. The expectation follows from the tail-sum formula for a nonnegative integer-valued random variable:
\[
\mathbb{E}[L]=\sum_{k=1}^{T}\Pr(L\geq k)=\sum_{k=1}^{T}\beta^k
=
\frac{\beta(1-\beta^T)}{1-\beta}.
\]
\end{proof}

\paragraph{Interpretation.}
When $\alpha=1$, we have $\beta=1$ and the whole sequence is bijective. When $\alpha<1$, we have $\beta<1$, and a recent compressive update appears with high probability. This makes the final state depend on a compressed prefix summary plus a short terminal suffix, paralleling the terminal all-unit suffix argument in the polynomial theory.




\subsection{Experimental Results}
\label{app:fsc-results}

We next empirically test the predictions from Appendix~\ref{app:fsc-theory}: compressive FSC regimes should make late-bottleneck latent reasoning viable, whereas fully bijective regimes should require sustained state propagation and therefore pose a harder setting for fixed-budget latent CoT.

We use a finite state space of size $N=50$ and a fixed library of $K=8$ transition functions with $m_{min} = 1 $ and $m_{max} = 49$. We vary the fraction of bijective functions over
\[
\alpha \in \{0.0, 0.1, 0.2, 0.6, 0.8, 1.0\}.
\]
We further evaluate hop lengths
\[
n \in \{2, 3, 4, 5, 7, 9\}
\]
to compare short- and long-horizon behavior. When $\alpha=1.0$, all transitions are bijective, yielding an incompressible regime analogous to the prime-modulus setting in the polynomial task. When $\alpha<1.0$, the library contains at least one compressive transition, yielding a compressible regime analogous to the composite-modulus setting.

\subsubsection{Compressive Region Where $\alpha < 1$} 

In the compressive regime, CODI exhibits the same short-horizon step-by-step pattern observed in the polynomial task. As shown in \cref{fig:fsc_logit_seq3_4}, the logit-lens analysis reveals that the bridge state $s_1$ becomes decodable along the token/latent trajectory for 2-hop FSC tasks. For 3-hop tasks, both bridge states are recovered: $s_1$ is decodable from the early token/latent positions and remains accessible through the latent trajectory, while $s_2$ emerges within the latent channel. This pattern extends to the 4-hop setting, where we observe a full-latent rollout of the intermediate states. These results indicate that, when the horizon is short, CODI can form and maintain the ground-truth intermediate states required for faithful latent step-by-step computation.

\begin{figure}[ht]
  \centering
  \begin{subfigure}[t]{0.49\textwidth}
    \centering
    \includegraphics[width=\linewidth]{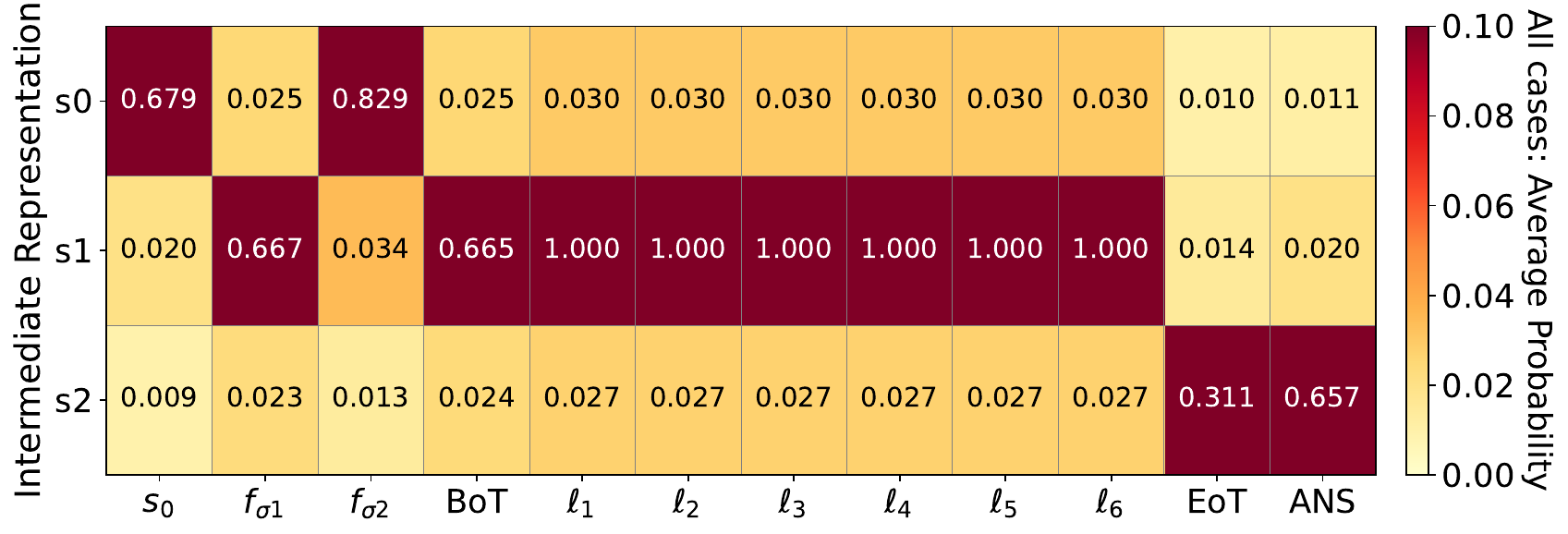}
    \caption{2-hop task and $\alpha = 0.2$}
    \label{fig:fsc_logit_3_0.2}
  \end{subfigure}\hfill
  \begin{subfigure}[t]{0.49\textwidth}
    \centering
    \includegraphics[width=\linewidth]{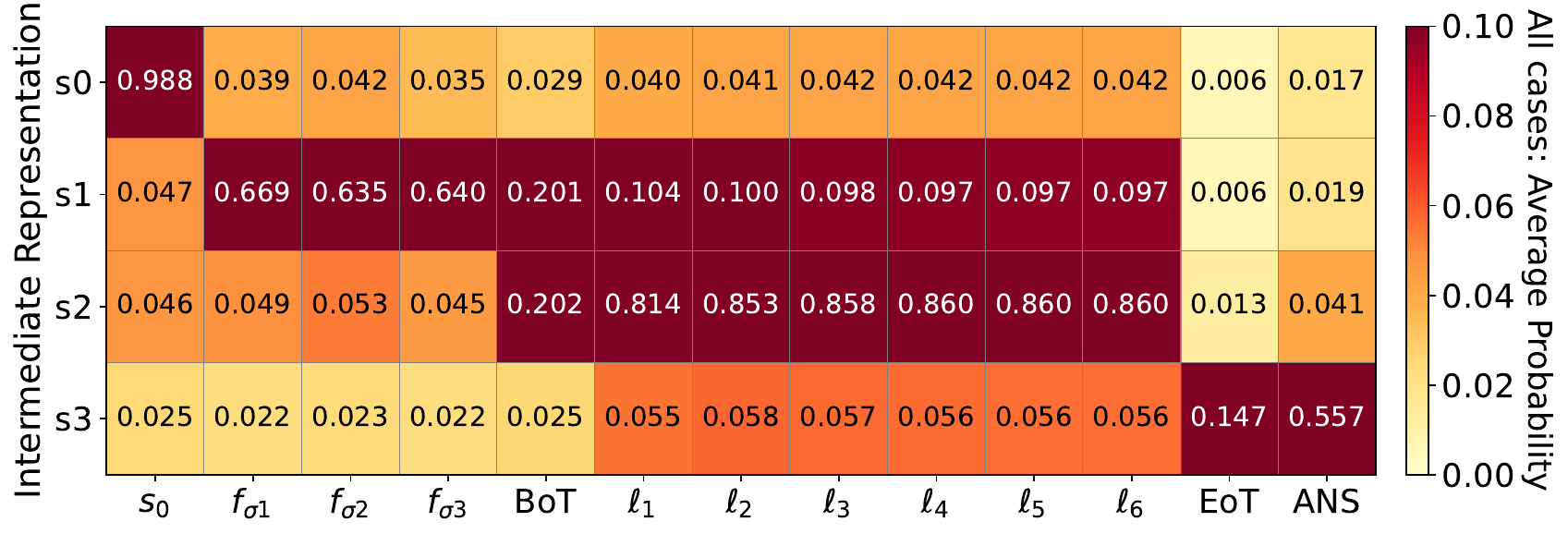}
    \caption{3-hop task and $\alpha = 0.6$}
    \label{fig:fsc_logit_4_0.6}
  \end{subfigure}

  \caption{\textbf{Logit-lens analysis on short-horizon compressive FSC tasks.}
For the 2-hop task in (a), the bridge state $s_1$ becomes decodable early  in the latent trajectory when $\alpha=0.2$, indicating that CODI forms the required intermediate state within at $f_{\sigma_1}$ token position and the latent channel. For the 3-hop task in (b), $s_1$ is already decodable at the $f_{\sigma_1}$ token position and remains decodable throughout the latent trajectory, while the next bridge state $s_2$ emerges within the latent channel. We observe qualitatively similar patterns across the tested values of $\alpha$ for 2-hop and 3-hop tasks, suggesting that CODI can realize faithful step-by-step latent reasoning in short-horizon FSC settings.}

  \label{fig:fsc_logit_seq3_4}
\end{figure}

Activation patching further shows that the decoded bridge states are not merely correlational, but causally involved in the computation. On the two-hop FSC task, corrupting the first transition $f_{\sigma_1}$ disrupts the formation of the bridge state $s_1$. As shown in \cref{fig:fsc_patch_3_0.2_f1}, patching clean activations into the latent-thought positions substantially restores accuracy, indicating that the latent channel stores information about $s_1$ that is necessary for the final prediction. In contrast, when the final transition $f_{\sigma_2}$ is corrupted, recovery concentrates at \texttt{[Ans]} rather than in the latent trajectory (\cref{fig:fsc_patch_3_0.2_f2}). This mirrors the two-route mechanism observed in the polynomial task: the latent channel carries the intermediate state, while the final operation is routed more directly to the answer readout.

We observe the similar pattern on the three-hop FSC task. Corrupting the initial state $s_0$, the first transition $f_{\sigma_1}$, or the second transition $f_{\sigma_2}$ disrupts the computation of the later bridge state $s_2$. Patching clean activations into the latent-thought positions substantially restores accuracy, indicating that the latent channel carries information about $s_2$ that is causally necessary for the final prediction. By contrast, when the final transition $f_{\sigma_3}$ is corrupted, recovery again concentrates at \texttt{[Ans]} rather than in the latent trajectory. Together, these interventions show that short-horizon FSC tasks reproduce the same division of labor observed in the polynomial setting: latent thoughts compute and maintain bridge states, while the final transition is routed directly to the answer boundary. Details can be seen in \cref{fig:patching_3hop_FSC_2x2}.

\begin{figure}[ht]
  \centering
  \begin{subfigure}[t]{0.49\textwidth}
    \centering
    \includegraphics[width=\linewidth]{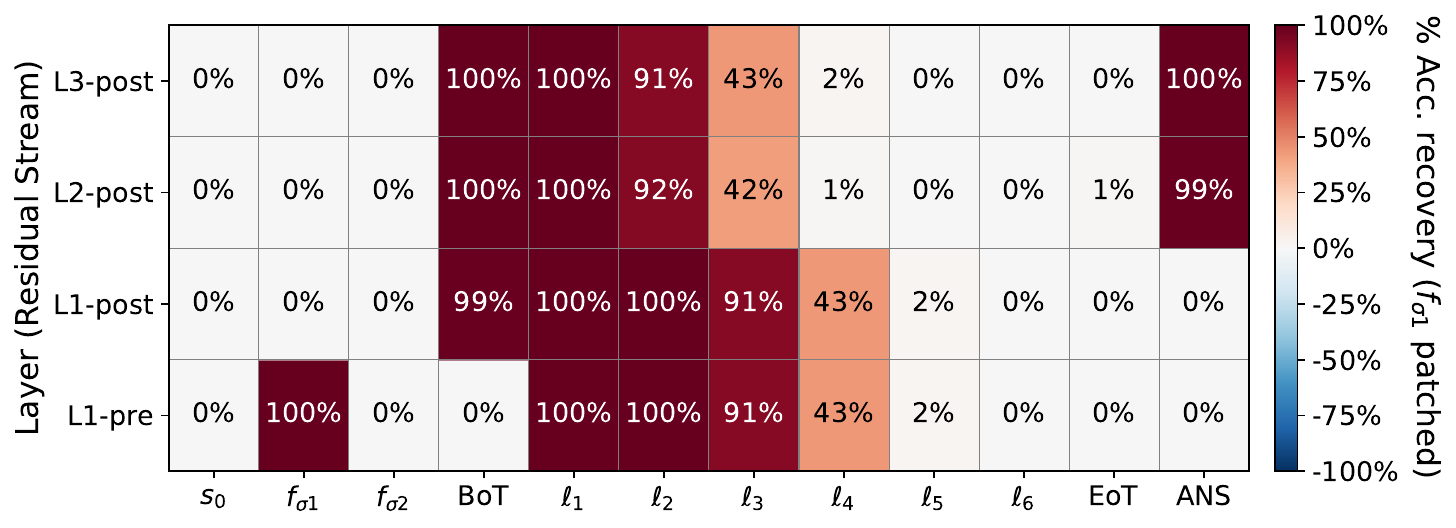}
    \caption{$f_{\sigma_1}-$corrupted}
    \label{fig:fsc_patch_3_0.2_f1}
  \end{subfigure}\hfill
  \begin{subfigure}[t]{0.49\textwidth}
    \centering
    \includegraphics[width=\linewidth]{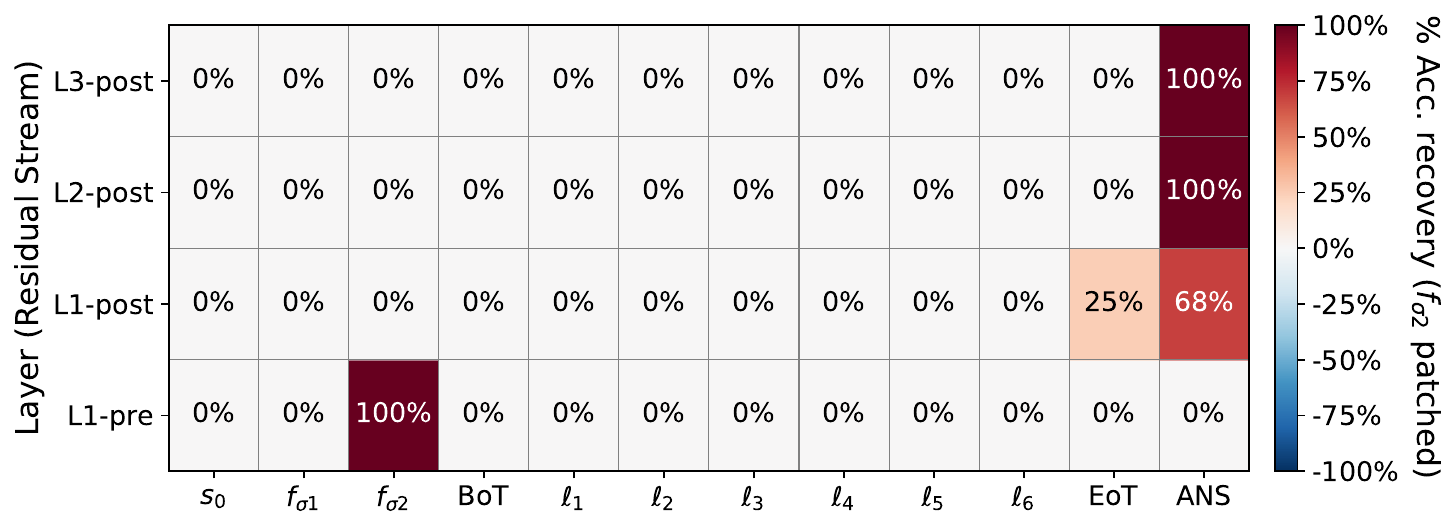}
    \caption{$f_{\sigma_2}-$corrupted}
    \label{fig:fsc_patch_3_0.2_f2}
  \end{subfigure}

  \caption{\textbf{Activation patching for transition-function corruptions (two-hop).} Cells show mean accuracy recovery (\%) over corrupted samples. When the first transition $f_{\sigma_1}$ is corrupted, patching clean activations into latent-thought positions restores performance, indicating that the latent channel carries the intermediate state $s_1$ needed for the downstream computation. In contrast, when the final transition $f_{\sigma_2}$ is corrupted, recovery is concentrated at \texttt{[Ans]}, consistent with a direct route that delivers the final operation to the answer readout.}

  \label{fig:fsc_logit_seq3_a_0.2}
\end{figure}

\begin{figure*}[ht]
  \centering
  \begin{subfigure}[t]{0.49\textwidth}
    \centering
    \includegraphics[width=\linewidth]{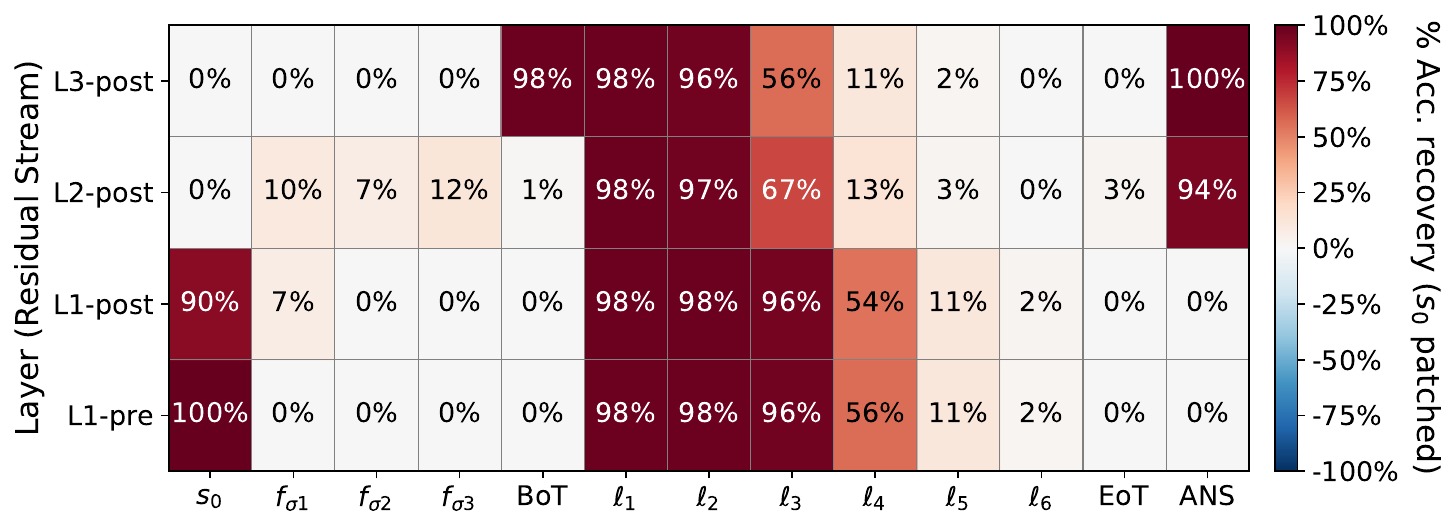}
    \caption{patching $s_0-$ corrupted run}
    \label{fig:3_hop_patch_s0}
  \end{subfigure}\hfill
  \begin{subfigure}[t]{0.49\textwidth}
    \centering
    \includegraphics[width=\linewidth]{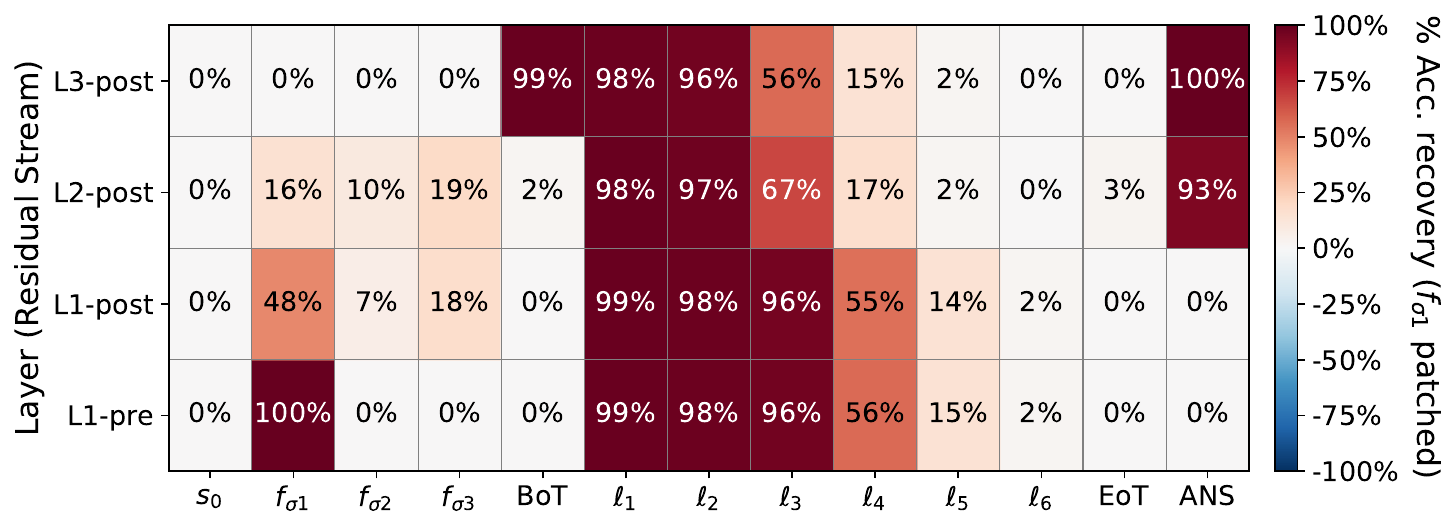}
    \caption{patching $f_{\sigma_1}-$ corrupted run}
    \label{fig:3_hop_patch_f1}
  \end{subfigure}

  \vspace{0.5em} 

  \begin{subfigure}[t]{0.49\textwidth}
    \centering
    \includegraphics[width=\linewidth]{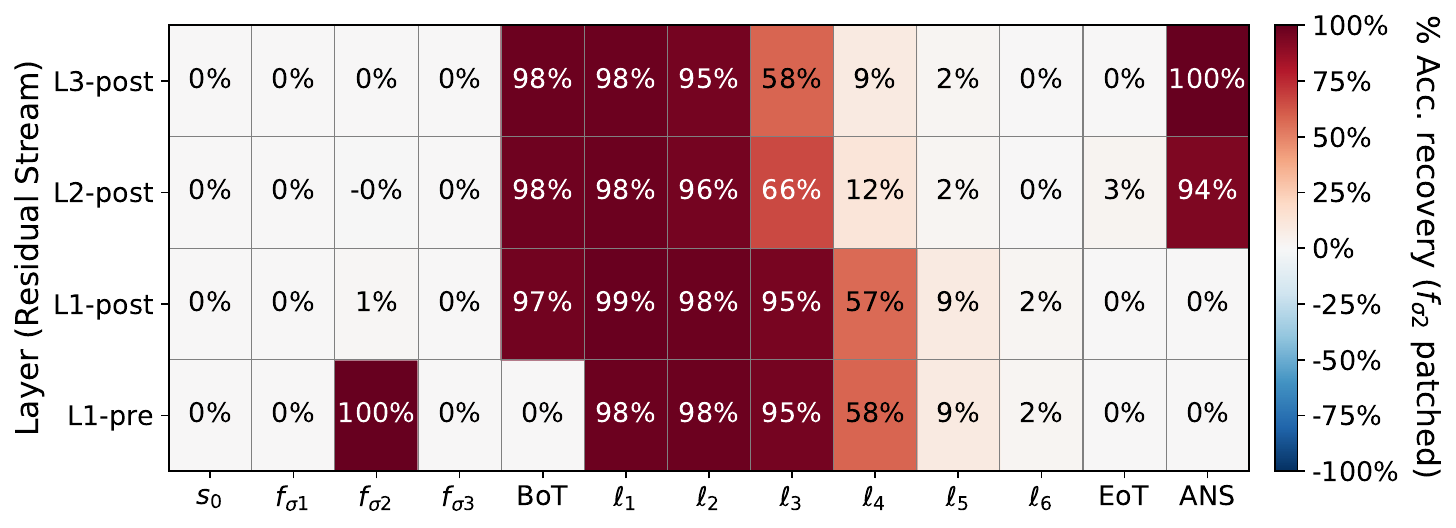}
    \caption{patching $f_{\sigma_2}-$ corrupted run}
    \label{fig:3_hop_patch_f2}
  \end{subfigure}\hfill
  \begin{subfigure}[t]{0.49\textwidth}
    \centering
    \includegraphics[width=\linewidth]{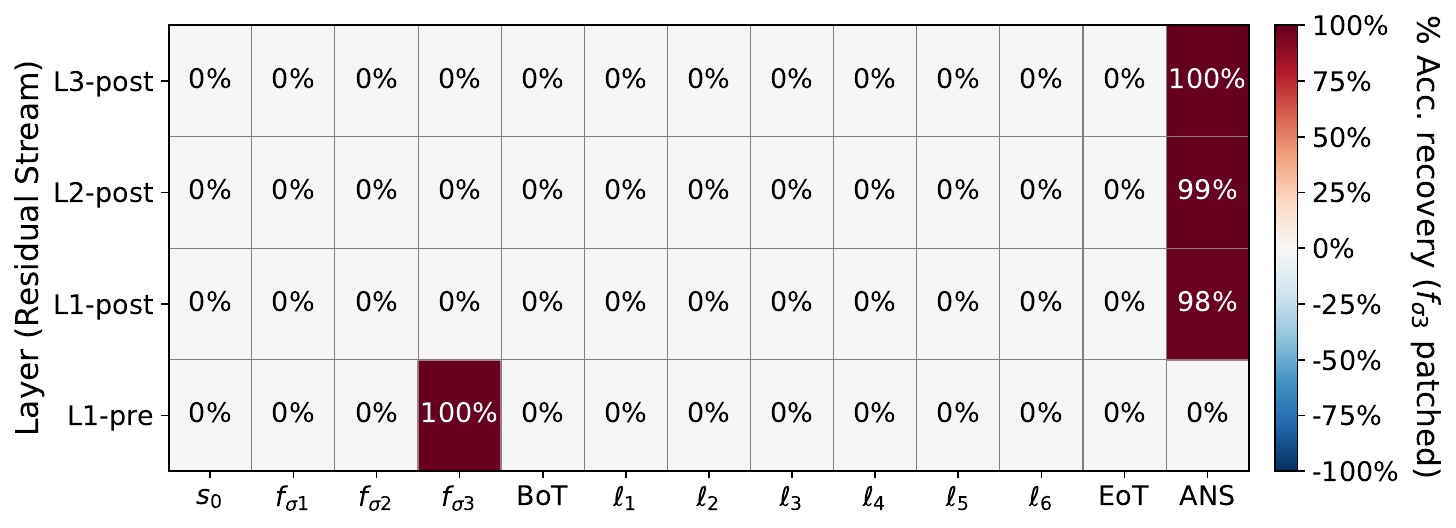}
    \caption{patching $f_{\sigma_3}-$ corrupted run}
    \label{fig:3_hop_patch_f3}
  \end{subfigure}

  \caption{\textbf{Activation patching for input-token corruptions in the three-hop FSC task.}
When $s_0$, $f_{\sigma_1}$, or $f_{\sigma_2}$ are corrupted, patching clean activations into the latent-thought positions ($\ell_1,\ell_2,\ell_3$) yields substantial accuracy recovery, indicating that the latent channel carries the intermediate information needed downstream. When $f_{\sigma_3}$ is corrupted, recovery localizes at \texttt{[Ans]}, consistent with a direct (copy-like) route that delivers $f_{\sigma_3}$ to the answer readout.
}
  \label{fig:patching_3hop_FSC_2x2}

\end{figure*}

We defer discussion of the 5-hop task, which lies near the transition point between short- and long-horizon behavior, to \Cref{app:transition_point}.

On longer-horizon FSC tasks, such as the 7-hop and 9-hop settings, we again observe a late-bottleneck latent reasoning pattern. As shown by the logit-lens analysis in \cref{fig:fsc_logit_seq_long}, the 7-hop task makes only the final bridge state $s_6$ readily decodable along the latent trajectory, while earlier intermediate states are largely collapsed. Similarly, in the 9-hop task, the latent trajectory primarily exposes the final bridge state $s_8$, with little evidence that the earlier states are explicitly maintained.

This behavior is consistent with the theoretical analysis in Appendix~\ref{app:fsc-theory}: in compressive FSC regimes, low-rank transitions reduce the number of distinguishable prefix states, so the final state can often be determined from a compressed prefix summary together with a short terminal suffix. It also mirrors the mechanism observed in polynomial tasks with composite moduli, where non-unit multipliers induce exact algebraic contractions and CODI concentrates its latent computation on late intermediates rather than sustaining a full step-by-step rollout.  

\begin{figure}[ht]

  \begin{minipage}{1\linewidth}
    \includegraphics[width=0.5\linewidth,height=0.5cm]{Other_Figures/Color_bars_horizontal_2.png}
    \vspace{-0.5em}
  \end{minipage}

  \centering
  \begin{subfigure}[t]{0.49\textwidth}
    \centering
    \includegraphics[width=\linewidth]{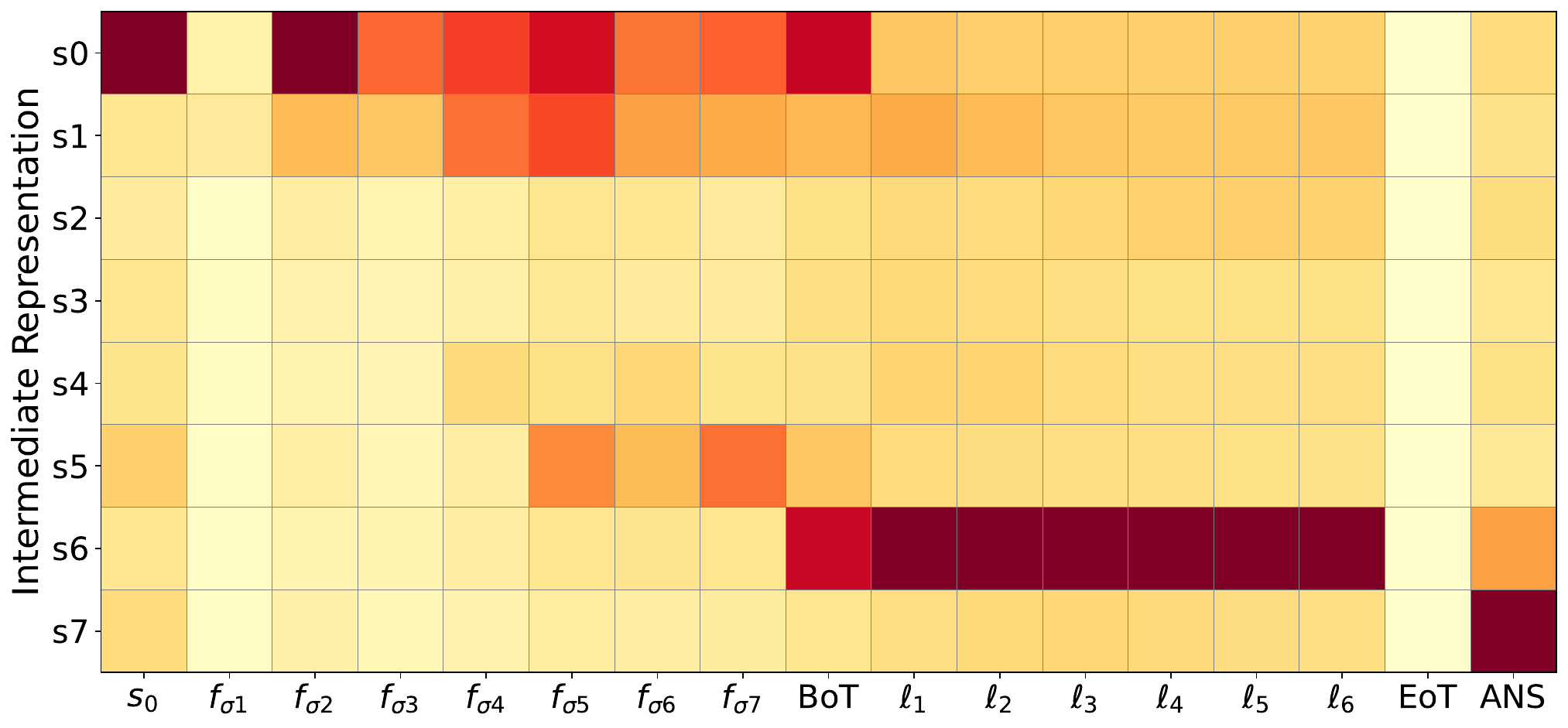}
    \caption{7-hop, $\alpha = 0.1$}
    \label{fig:fsc_logit_8_0.1}
  \end{subfigure}\hfill
  \begin{subfigure}[t]{0.49\textwidth}
    \centering
    \includegraphics[width=\linewidth]{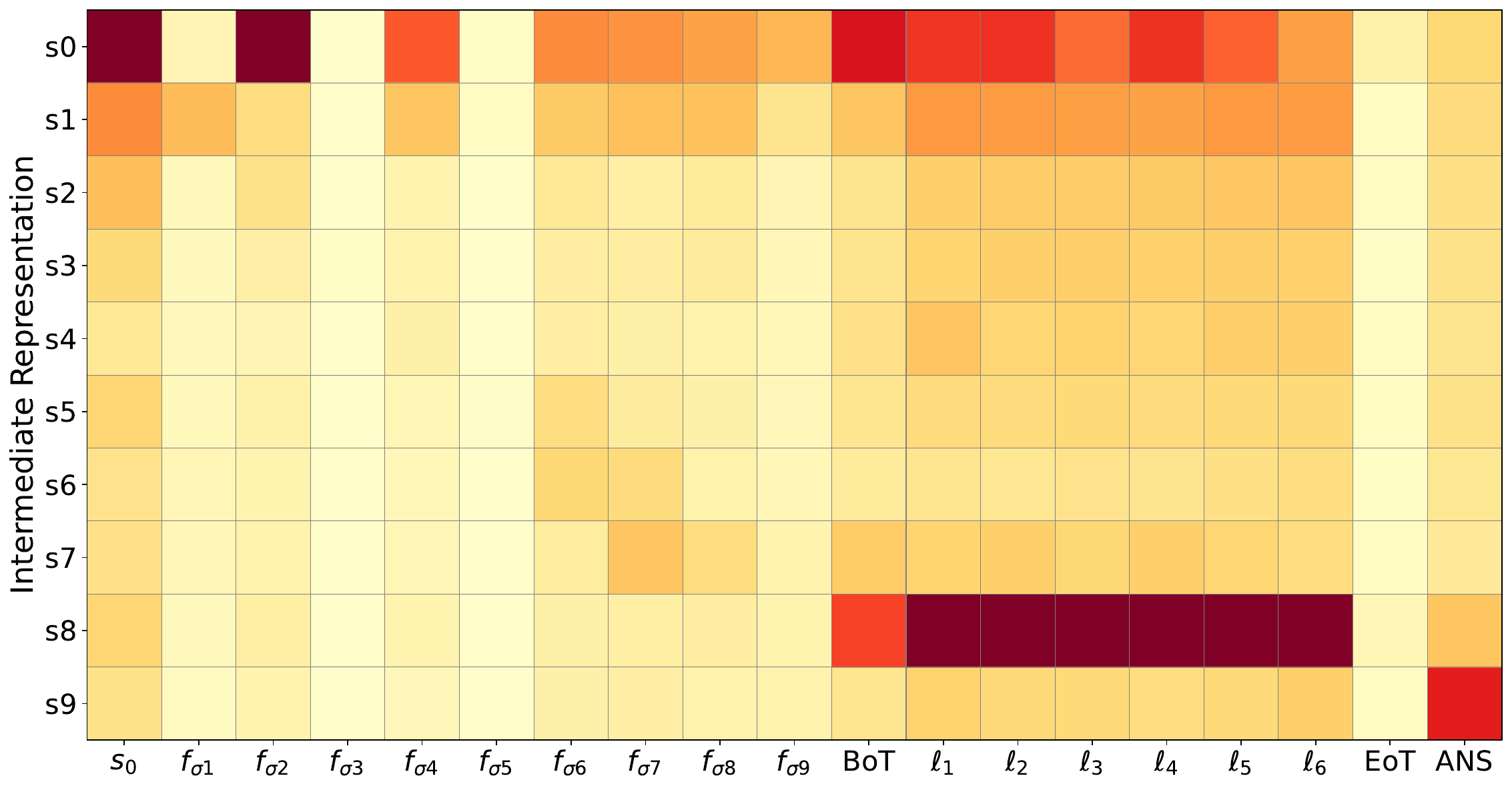}
    \caption{9-hop, $\alpha = 0.6$}
    \label{fig:fsc_logit_10_0.6}
  \end{subfigure}

  \caption{\textbf{Logit-lens analysis on long-horizon compressive FSC tasks.}
For the $n$-hop FSC task, only the final bridge state $s_{n-1}$ becomes reliably decodable along the latent trajectory, while earlier intermediate states are largely collapsed. This late-intermediate pattern is qualitatively consistent across the tested compressive regimes with different values of $\alpha$.}

  \label{fig:fsc_logit_seq_long}
\end{figure}

\subsubsection{Incompressible Region where $\alpha = 1$} 

\paragraph{Long-horizon behavior in the incompressible FSC regime.}

In the incompressible regime of FSC tasks, CODI exhibits step-by-step latent reasoning pattern for short-horizon instances such as 2-hop, 3-hop, and 4-hop FSC tasks. The logit-lens visualizations reveal a sequential emergence of the relevant bridge states along the latent trajectory, closely matching the behavior described above for compressive FSC tasks. This suggests that, when the reasoning horizon is short, CODI can faithfully construct the required intermediate states even when the task does not admit information compression. We again defer discussion of the 5-hop task, which lies near the transition point between short- and long-horizon behavior, to \Cref{app:transition_point}.

However, this faithful step-by-step reasoning pattern breaks down for longer-horizon FSC tasks. For 7-hop and 9-hop instances, we observe that the latent reasoning trajectory collapses: late-stage intermediate states are no longer reliably formed or decodable. For example, in the 7-hop task, the final bridge state $s_6$ is no longer observable, and in the 9-hop task, the final bridge state $s_8$ similarly fails to emerge in the latent trajectory as seen in \cref{fig:fsc_logit_seq_long_1.0}. This contrasts with the compressive FSC regime, where late intermediates remain visible, indicating that compression helps preserve a tractable partial latent reasoning path over longer horizons.

This qualitative behavior is consistent with our theoretical analysis.  In the incompressible FSC regime, corresponding to $\alpha=1$, every transition is bijective, so the composed dynamics contain no many-to-one contraction that can erase state distinctions. Consequently, successful long-horizon prediction cannot exploit a compressed prefix summary induced by the task itself; it requires faithfully propagating the running state through the sequence of updates. This requirement mirrors the polynomial task over a prime modulus, where the algebraic structure prevents lossy compression of the intermediate computation and therefore necessitates a full intermediate reasoning path. Consequently, both settings impose a stringent information-preservation constraint: successful reasoning requires the model to retain and propagate all intermediate states across the entire horizon. Our empirical observation that CODI succeeds on short horizons but exhibits collapsed reasoning on longer horizons is therefore aligned with the theory, which predicts that full-path reasoning becomes increasingly difficult as the horizon grows in non-compressive settings.

\begin{figure}[ht]

  \begin{minipage}{1\linewidth}
    \includegraphics[width=0.5\linewidth,height=0.5cm]{Other_Figures/Color_bars_horizontal_2.png}
    \vspace{-0.5em}
  \end{minipage}

  \centering
  \begin{subfigure}[t]{0.49\textwidth}
    \centering
    \includegraphics[width=\linewidth]{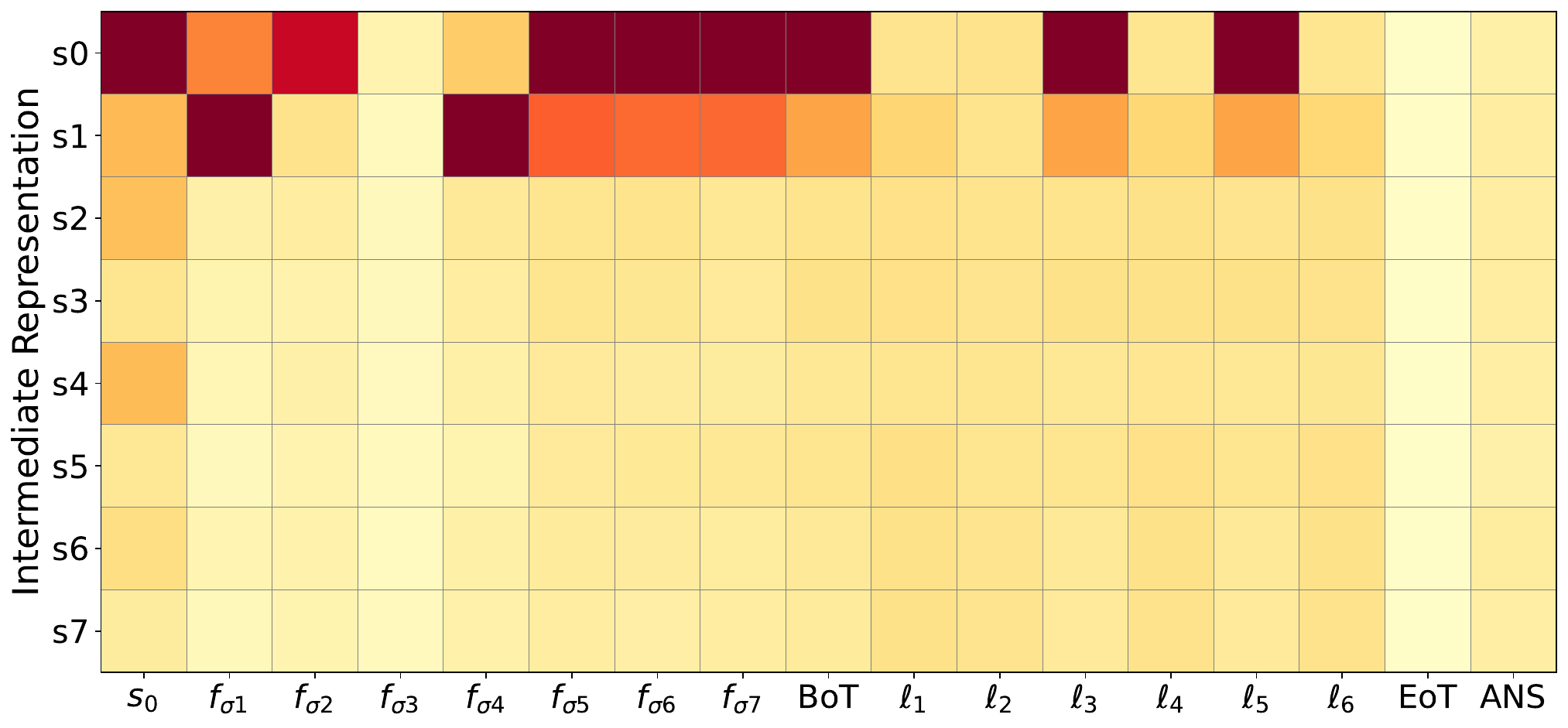}
    \caption{7-hop, $\alpha = 1.0$}
    \label{fig:fsc_logit_8_1.0}
  \end{subfigure}\hfill
  \begin{subfigure}[t]{0.49\textwidth}
    \centering
    \includegraphics[width=\linewidth]{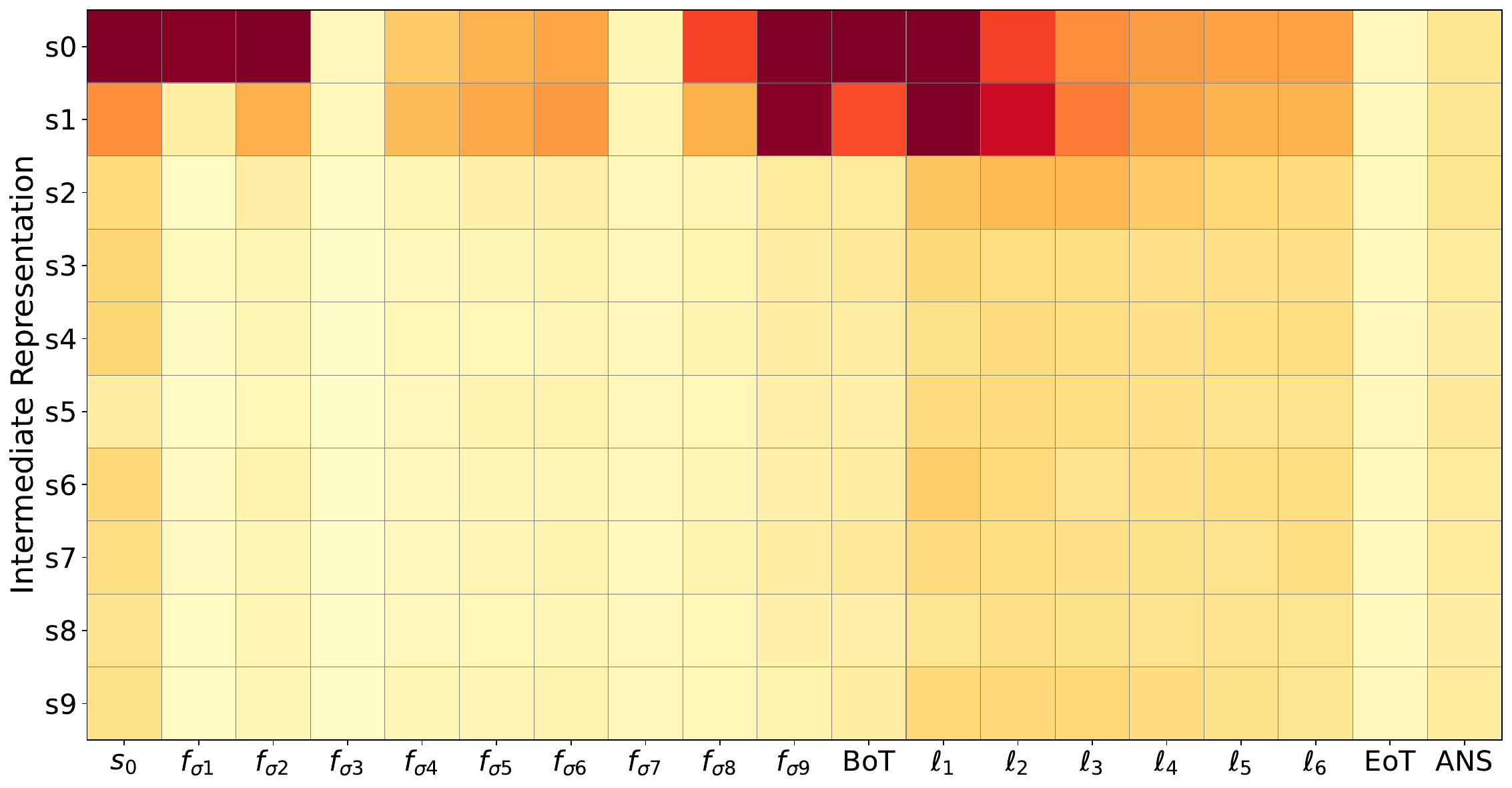}
    \caption{9-hop, $\alpha = 1.0$}
    \label{fig:fsc_logit_10_1.0}
  \end{subfigure}

  \caption{\textbf{Logit-lens analysis on long-horizon incompressible FSC tasks.}
The final bridge state $s_{n-1}$ is no longer decodable along the latent trajectory, indicating a breakdown of step-by-step reasoning}

  \label{fig:fsc_logit_seq_long_1.0}
\end{figure}




\subsection{Generality and Scope of FSC}
\label{app:fsc-generality}

FSC is not intended to capture all sequential reasoning. Its rigorous scope is deterministic finite-state, operation-indexed sequential computation.

\begin{proposition}[Representation of finite-state operation-indexed tasks]
\label{prop:fsc-representation}
Let $\mathcal{T}$ be any deterministic finite-state sequential task with finite hidden state space $\mathcal{Z}$, finite operation alphabet $\mathcal{A}=\{a_1,\ldots,a_K\}$, and Markovian update rule
\[
z_t=\delta(z_{t-1},a_t).
\]
Then $\mathcal{T}$ can be represented as an FSC instance.
\end{proposition}

\begin{proof}
Set $\mathcal{S}=\mathcal{Z}$. For each operation $a_i\in\mathcal{A}$, define
\[
f_i(z)=\delta(z,a_i).
\]
Given an operation sequence $a_{1:T}$, encode each operation by its index $\sigma_t$. Then the FSC recurrence
\[
s_t=f_{\sigma_t}(s_{t-1})
\]
exactly reproduces
\[
z_t=\delta(z_{t-1},a_t).
\]
If $f_i$ is bijective, it belongs to the bijective part of the library; otherwise, its local compression level is captured by its transition rank $r_i=|\operatorname{Im}(f_i)|$.
\end{proof}

\paragraph{Important special cases.}
The following task families are direct special cases of FSC.

\begin{itemize}

    \item \textbf{Finite automata.} A finite automaton has states $q_t$ and update rule $q_t=\delta(q_{t-1},a_t)$. For each symbol $a$, define $f_a(q)=\delta(q,a)$.

    \item \textbf{Lookup-table transition systems.} A fixed library of discrete transition tables is exactly an FSC function library.

    \item \textbf{Finite-register update systems.} If the state consists of a tuple of finite registers, define $\mathcal{S}$ as the product of all register values. Each operation is then a deterministic function on the joint state.

    \item \textbf{Bounded finite-memory algorithms.} Any algorithm with a finite internal configuration under a fixed memory or length bound can be represented by encoding the full configuration as the FSC state. This may cause state-space blowup, but it is still a finite-state representation.

    \item \textbf{Modular polynomial iteration.} The polynomial task is
    \[
    s_t=s_{t-1}x_t+b \pmod M.
    \]
    For each multiplier $x$, define
    \[
    f_x(s)=sx+b \pmod M.
    \]
    The operation index $\sigma_t$ is the multiplier $x_t$. If $x$ is a unit modulo $M$, then $f_x$ is bijective. If $\gcd(x,M)=d>1$, then $f_x$ is an exact $d$-to-one map with image size $M/d$. Thus the polynomial prime/composite distinction is a structured algebraic instance of the FSC bijective/compressive distinction.
\end{itemize}

\paragraph{Scope of the FSC abstraction.}
FSC is intended as a controlled abstraction for deterministic finite-state, operation-indexed sequential computation, rather than as a universal model of all sequential reasoning. This scope includes finite automata, lookup-table transition systems, bounded finite-register update systems,and the modular polynomial task studied in the main text. Tasks with unbounded memory, continuous or stochastic dynamics, or rich structured operations can be brought into this framework only after additional modeling choices, such as bounding the memory, discretizing the state space, encoding randomness as input, or expanding the state representation. Overall, FSC provides a clean setting for testing whether the compressibility mechanism observed in polynomial iteration persists beyond modular arithmetic.

\subsection{Implications for Latent-CoT Mechanisms}
\label{app:fsc-implications}

The FSC analysis suggests that the key variable in our polynomial experiments is not modular arithmetic itself, but transition-level compressibility. Fully bijective regimes preserve state distinctions and require full-history propagation, whereas compressive regimes contain low-rank updates that reduce the number of distinguishable prefix states and make late intermediates more predictive of the final answer.

Our FSC experiments support this interpretation. On short-horizon tasks, CODI forms faithful bridge-state representations. On long-horizon compressive FSC tasks, it instead exhibits late-bottleneck partial rollouts, with late intermediates decodable and earlier states largely collapsed. Fully bijective FSC regimes remove this compression mechanism and therefore provide an incompressible control analogous to prime-modulus polynomial tasks. Thus, FSC shows that the polynomial task is a structured algebraic instance of a broader finite-state rank-contraction phenomenon.

\section{Transition points depend on task difficulty}
\label{app:transition_point}

We emphasize that the transition from faithful step-by-step latent reasoning to partial or collapsed reasoning is not tied to a universal hop length. Rather, it depends on the difficulty of the task distribution relative to the model and training budget. In our default polynomial setting with modulus $m=50$, a 3-layer, 2-head CODI model trained with 2,500 examples per sequence length exhibits faithful bridge-state formation on 2-hop and 3-hop tasks (\cref{fig:logic_lens_seq_3,fig:logic_lens_3hop_L3H2}), but shifts to a late-bottleneck partial rollout on tasks with 4 or more hops (\cref{fig:n_hops_logic_lens}). Thus, under this setting, the transition appears around 4 hops. By contrast, under the same model and data budget but with prime modulus $m=41$, the latent reasoning path collapses even for 2-hop tasks: we no longer observe a reliable faithful bridge-state trajectory.

The FSC experiments further support this view. With $N=50$ states and $K=8$ transition functions, the transition point varies with the bijective fraction $\alpha$, which controls the degree of compressibility. For example, on 5-hop FSC tasks, the fully compressive regime $\alpha=0.0$ can still exhibit faithful step-by-step latent reasoning. In contrast, intermediate values such as $\alpha=0.2,0.6,0.8$ shift toward late-bottleneck partial reasoning, while the fully bijective regime $\alpha=1.0$ shows a collapse of the latent reasoning trajectory. Please see \cref{fig:logit_fsc_5hop_a_2x2} for more detail. This suggests that the transition point depends on the degree of compressibility: for $\alpha=0.0$, faithful latent rollouts persist beyond 5 hops, whereas for intermediate values of $\alpha$ the transition to partial late-bottleneck reasoning already occurs around the 5-hop regime. This pattern is consistent with our theory: increasing $\alpha$ reduces the frequency of rank-contractive updates, making the task less compressible and requiring more sustained state propagation.

Overall, these results suggest that CODI-style latent reasoning reflects a capacity--difficulty tradeoff shaped by model capacity, training budget, and task structure. In the FSC setting, difficulty depends on the state-space size $N$, the number of transition functions $K$, the horizon length, the bijective fraction $\alpha$, and the compressive image-size range $[m_{\min}, m_{\max}]$. When the model has sufficient effective capacity, CODI can sustain faithful step-by-step latent computation. As the task becomes harder, it shifts toward late-bottleneck partial rollouts when compressive structure is available, but collapses when the transition system is effectively incompressible. Formalizing this capacity--difficulty tradeoff remains an important direction for future work.

\begin{figure*}[ht]

  \begin{minipage}{1\linewidth}
    \includegraphics[width=0.5\linewidth,height=0.5cm]{Other_Figures/Color_bars_horizontal_2.png}
    \vspace{-0.5em}
  \end{minipage}

  \vspace{0.8em} 

  \centering
  \begin{subfigure}[t]{0.49\textwidth}
    \centering
    \includegraphics[width=\linewidth]{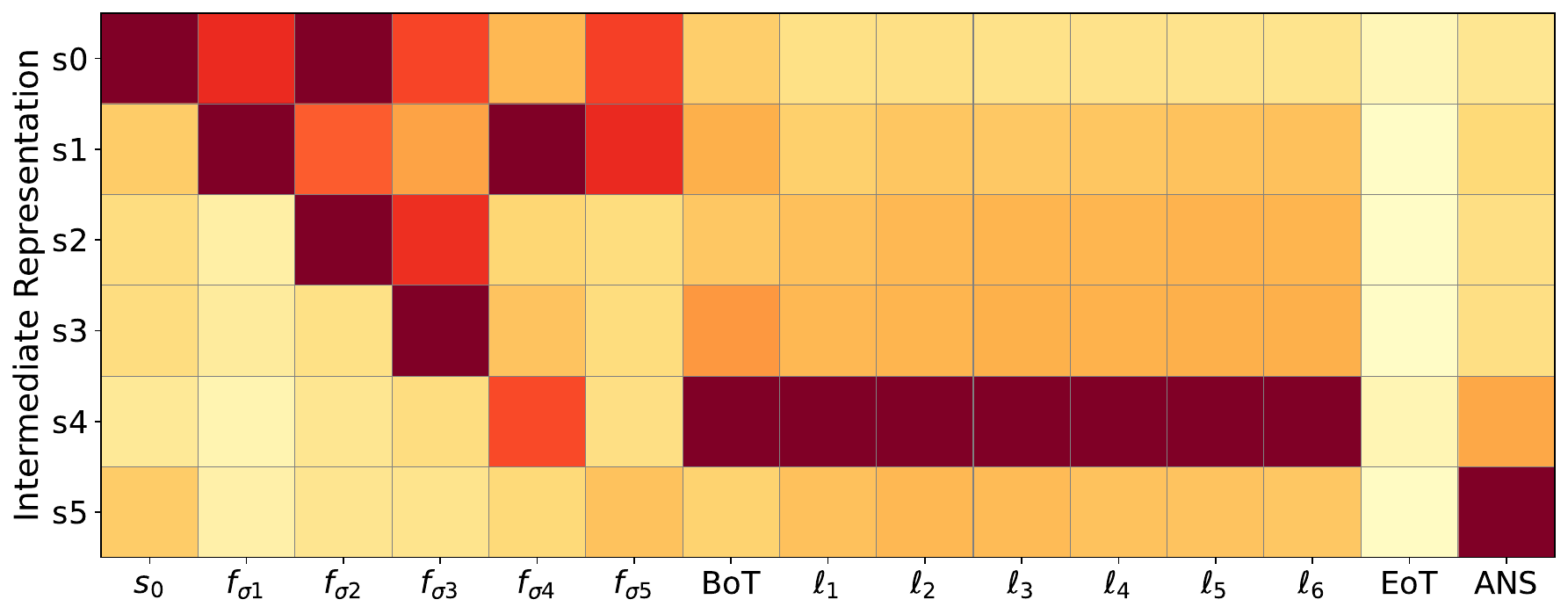}
    \caption{5-hop, $\alpha = 0$}
    \label{fig:5_hop_a_0_logit_fsc}
  \end{subfigure}\hfill
  \begin{subfigure}[t]{0.49\textwidth}
    \centering
    \includegraphics[width=\linewidth]{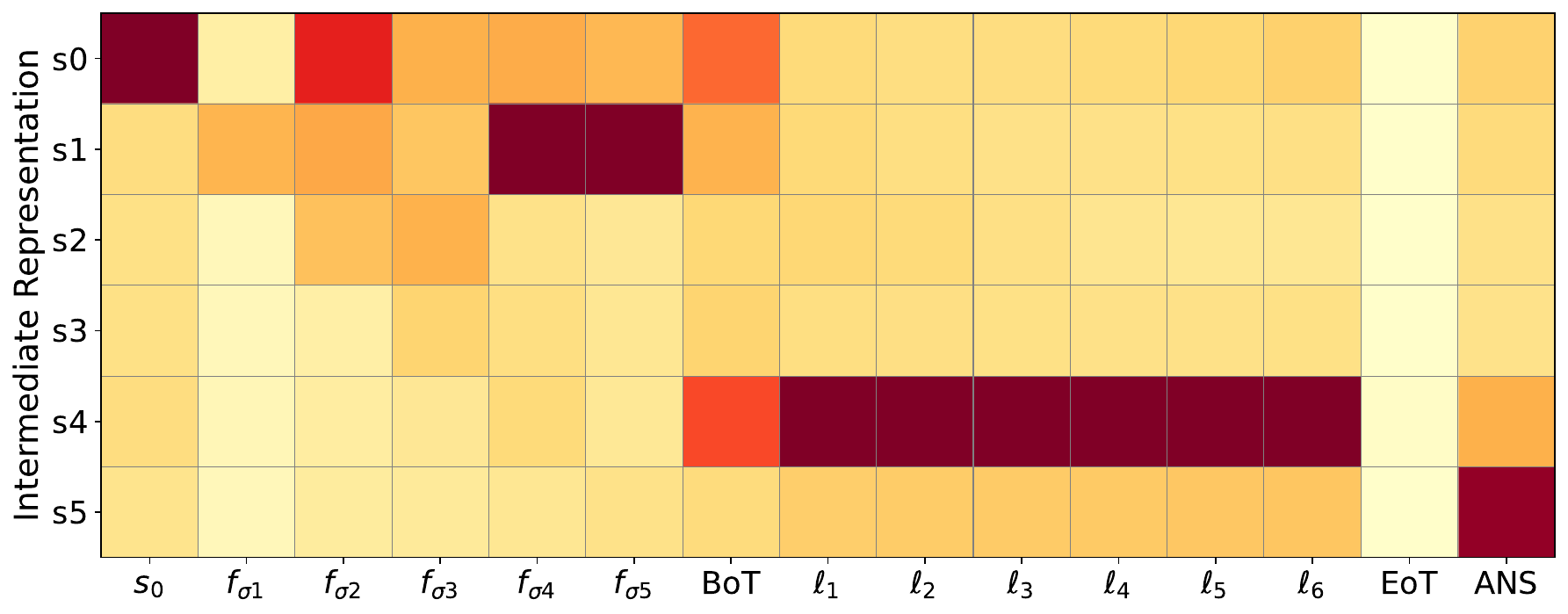}
    \caption{5-hop, $\alpha = 0.2$}
    \label{fig:5_hop_a_0.2_logit_fsc}
  \end{subfigure}

  \vspace{0.5em} 

  \begin{subfigure}[t]{0.49\textwidth}
    \centering
    \includegraphics[width=\linewidth]{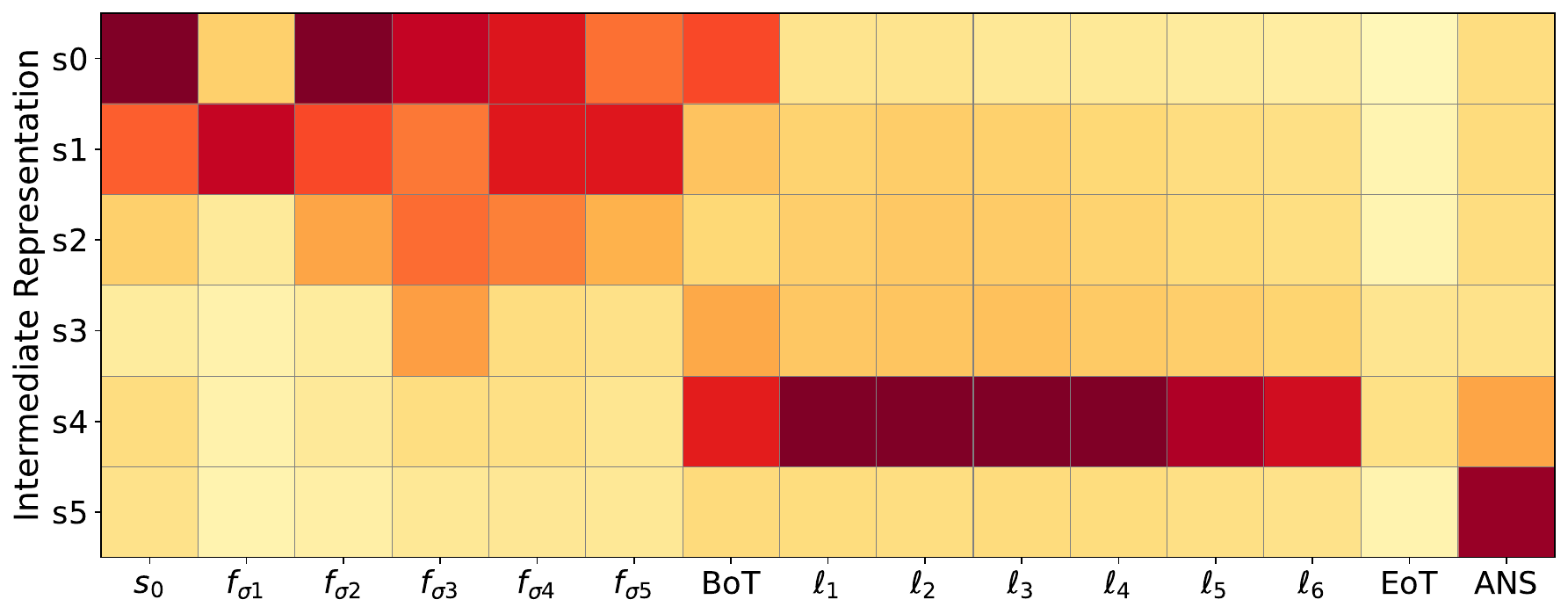}
    \caption{5-hop, $\alpha = 0.6$}
    \label{fig:5_hop_a_0.6_logit_fsc}
  \end{subfigure}\hfill
  \begin{subfigure}[t]{0.49\textwidth}
    \centering
    \includegraphics[width=\linewidth]{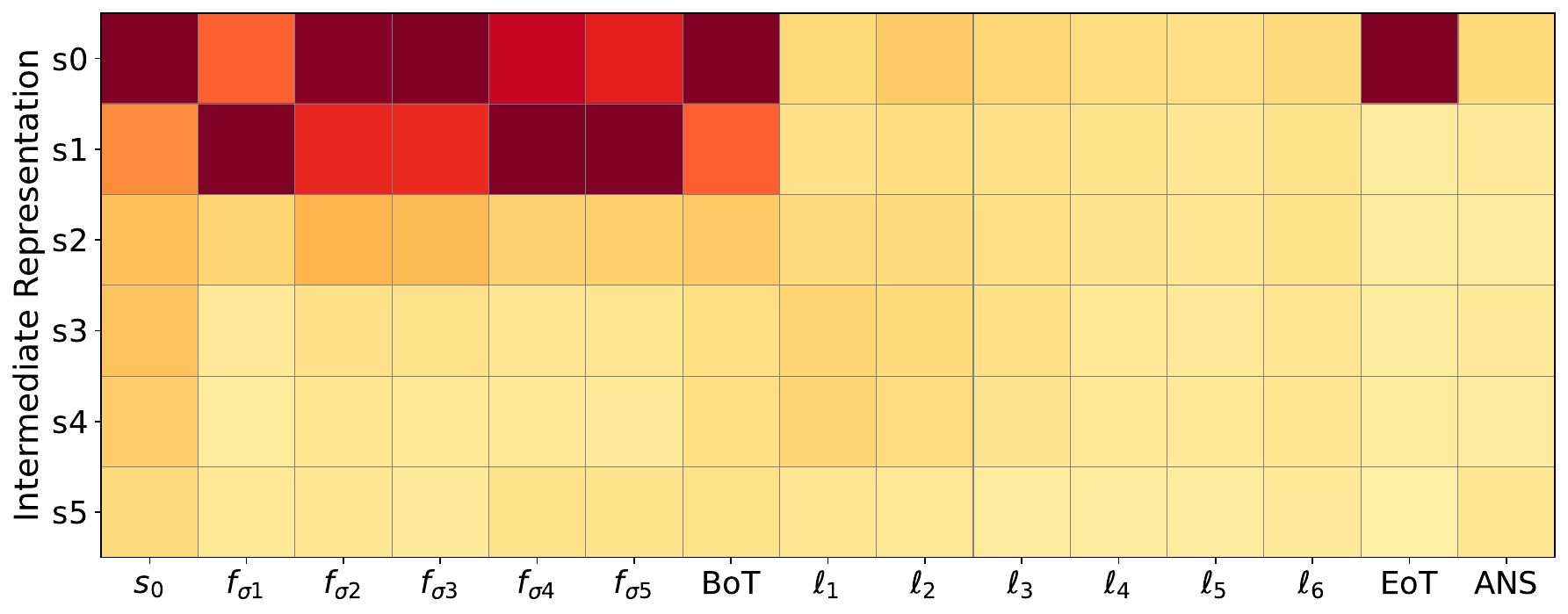}
    \caption{5-hop, $\alpha = 1.0$}
    \label{fig:5_hop_a_1.0_logit_fsc}
  \end{subfigure}

\caption{\textbf{Transition point depends on FSC compressibility.}
Logit-lens visualizations for 5-hop FSC tasks with $N=50$ states and $K=8$ transition functions, varying the bijective fraction $\alpha$. At the same horizon and model scale, CODI remains in a faithful step-by-step regime when $\alpha=0.0$, shifts to a late-bottleneck partial-rollout regime for intermediate $\alpha$, and collapses in the fully bijective regime $\alpha=1.0$. Thus, the hop length alone does not determine the transition from faithful reasoning to partial or collapsed reasoning; the transition point also depends on task compressibility, with larger $\alpha$ reducing rank-contractive updates and making sustained state propagation harder.}
  \label{fig:logit_fsc_5hop_a_2x2}

\end{figure*}

\section{Compute Resources}
\label{app:compute}

All experiments were run on single-GPU workers without distributed training. Most training runs used one NVIDIA A100-SXM4-80GB GPU. The A100-80GB was used primarily for availability rather than necessity: the models studied in this paper are small GPT-2-style transformers trained on synthetic sequential tasks, and each run used only a small fraction of the available GPU memory. A typical 3-layer, 2-head CODI model with six latent tokens completed training within a few hours on a single A100 GPU. Architecture, latent-step, and task-configuration sweeps were run as independent single-GPU jobs.

The mechanistic analyses were substantially cheaper than training. Logit-lens decoding, attention-map visualization, activation patching, state-swap interventions, and targeted attention knockouts were performed on trained checkpoints using forward passes over held-out examples. These analyses did not require distributed computation and used substantially less GPU memory than training; in practice, they can be reproduced on a single modern GPU. Linear probes were trained on frozen activations extracted from trained checkpoints and were also lightweight compared with model training.

Overall, reproducing the main results does not require large-scale compute. Although our experiments were mostly executed on A100-SXM4-80GB GPUs, substantially smaller GPUs should be sufficient for the main training and interpretability experiments because the models are small and the peak memory footprint is far below 80GB.

\section{Broad Impact}
\label{app:broad_impact}
Mechanistic interpretability aims to understand the computations implemented by deep learning systems by linking
behavior to internal representations and circuits. Because
latent chain-of-thought moves reasoning into hidden activations, mechanistic study can improve safety and reliability
by revealing when models carry out faithful multi-step computation versus shortcut or compression-driven strategies,
and by informing the design of more predictable objectives.
At the same time, such understanding could also enable
more capable models by making internal computation more
efficient, potentially amplifying broader societal issues associated with advanced AI systems; these issues are beyond
the scope of this statement. Our experiments are conducted
in a deliberately controlled synthetic setting (polynomial-iteration and FSC tasks over modular arithmetic) with randomly sampled inputs and ground-truth intermediate states; the data
contain no real-world entities or personal information, and
no human subjects are involved, so we do not anticipate
privacy, consent, or representational harms originating from
the dataset.



\end{document}